%% file: main.tex
\documentclass{article}

\def\neurips{0}
\def\supplemental{0} 
\def\tpdp{0} 

\ifnum\neurips=1
    
         \usepackage[final]{neurips_2021}
    
\else 
\fi

\input{macros}

\ifnum\neurips=1
    \usepackage[utf8]{inputenc} 
    \usepackage[T1]{fontenc}    
    \usepackage{hyperref}       
    \usepackage{url}            
    \usepackage{booktabs}       
    \usepackage{amsfonts}       
    \usepackage{nicefrac}       
    \usepackage{microtype}      
    \usepackage{xcolor}         
\else 
\fi

\ifnum\neurips=0
\title{Differentially Private Sampling from Distributions \footnote{A preliminary version of this work appeared at NeurIPS 2021 \cite{RaskhodnikovaSSS21}} \\}
\fi

\ifnum\neurips=1
\title{Differentially Private Sampling from Distributions \\}
\fi

\ifnum\neurips=1
    \author{%
    Sofya Raskhodnikova \\
    Department of Computer Science\\
    Boston University\\
    \texttt{sofya@bu.edu} \\
    \And
    Satchit Sivakumar \\
    Department of Computer Science\\
    Boston University\\
    \texttt{satchit@bu.edu} \\
    \And
    Adam Smith \\
    Department of Computer Science\\
    Boston University\\
    \texttt{ads22@bu.edu} \\
    \And
    Marika Swanberg \\
    Department of Computer Science\\
    Boston University\\
    \texttt{marikas@bu.edu} \\
    }
\else 
    \author{Sofya Raskhodnikova\thanks{Department of Computer Science, Boston University. \texttt{\{sofya,satchit,ads22,marikas\}@bu.edu}} \thanks{SR was supported in part by NSF award CCF-1909612.}  \and Satchit Sivakumar\footnotemark[2] \thanks{SS was supported in part by NSF award CNS-2046425, as well as Cooperative Agreement CB20ADR0160001 with the Census Bureau. The views expressed in this paper are those of the author and not those of the U.S. Census Bureau or any other sponsor.} \and Adam Smith\footnotemark[2] \thanks{AS and MS were supported in part by NSF awards CCF-1763786 and CNS-2120667, as well as faculty research awards from Google and Apple.} \and Marika Swanberg\footnotemark[2]~\footnotemark[5]}
\fi

\date{}

\begin{document}
\maketitle 

\begin{abstract}
    We initiate an investigation of  private sampling from distributions. Given a dataset with $n$ independent observations from an unknown distribution $P$, a sampling algorithm must output a single observation from a distribution that is close in total variation distance to $P$ while satisfying differential privacy. 
    Sampling abstracts the goal of generating small amounts of realistic-looking data.
    We provide tight upper and lower bounds for the dataset size needed for this task for three natural families of distributions: arbitrary distributions on $\{1,\ldots ,k\}$, arbitrary product distributions on $\{0,1\}^d$, and product distributions on $\{0,1\}^d$ with bias in each coordinate bounded away from 0 and 1. We demonstrate that, in some parameter regimes, private sampling requires asymptotically fewer observations 
    than learning a description of $P$ nonprivately; in other regimes, however, private sampling proves to be as difficult as private learning. Notably, for some classes of distributions, the overhead in the number of observations needed for private learning compared to non-private learning is completely captured by the number of observations needed for private sampling.
\end{abstract}

\ifnum\neurips=0
\ifnum\tpdp=0
\setcounter{tocdepth}{2}
\setlength{\columnsep}{1cm}
\begin{multicols*}{2}
    {\small \tableofcontents}
\end{multicols*}
\newpage 
\fi
\fi

\ifnum\tpdp=0
\section{Introduction}\label{sec:intro}
\fi

\ifnum\tpdp=0
Statistical machine learning models trained on sensitive data are now widely deployed in domains such as education, finance, criminal justice, medicine, and public health. The personal data used to train such models are more detailed than ever, and 
there is a growing awareness of the privacy risks that come with their use. Differential privacy is a standard for confidentiality that is now well studied and increasingly deployed. 

Differentially private algorithms ensure 
that whatever is learned about an individual from the algorithm's output would be roughly the same whether or not that individual's record was actually part of the input dataset. This requirement entails a strong guarantee, but limits the design of algorithms significantly. As a result, there 
\else
There
\fi 
is a substantial line of work on developing good differentially private methodology for learning and statistical estimation tasks, and on understanding how the sample size required for specific tasks increases relative to unconstrained algorithms.
A typical task investigated in this area is \emph{distribution learning}: informally, given records drawn i.i.d.\  from an unknown distribution $\distr$, the algorithm aims to produce a description of a distribution $Q$ that is close to~$\distr$.

However, often the task at hand
requires much less than full-fledged learning. We may simply need to generate a small amount of data that has the same distributional properties as the population, or perhaps simply ``looks plausible'' for the population. For example, one might need realistic data for debugging a software program or for getting a quick idea of the range of values in a dataset.

In this work, we study a basic problem that captures these seemingly less stringent aims. 
Informally, the goal is to design a sampling algorithm that, given a dataset with $n$ observations drawn i.i.d.\ from a distribution $\distr$, generates a single observation from a distribution $Q$ that is close to~$\distr$. 


To formulate the problem precisely, consider a randomized algorithm $\sampler:\universe^n\to \universe$ that takes as input a dataset of $n$ records from some universe $\universe$ and outputs a single element of $\universe$. 
Given a distribution~$\distr$ on $\universe$, let $\Datarv=(\datarv_1,\dots,\datarv_n)$ be a random dataset with entries drawn i.i.d.\ from $\distr$. Let $\sampler(\Datarv)$ denote the random variable corresponding to the output of the algorithm $\sampler$ on input $\Datarv$. 
This random variable depends on both the selection of entries of $\Datarv$ from $\distr$ and the coins of $\sampler$. 
Let $\distroutput{\sampler, \distr}$ denote the distribution of $\sampler(\Datarv)$, so that
$$\distroutput{\sampler, \distr}(z) = \Pr_{\substack{\Datarv~\sim_{i.i.d.} \distr \\ \text{coins of \sampler}}}(\sampler(\Datarv)=z) = \sum_{\Datafixed \in \universe^n} \Bparen {\Pr(\Datarv = \Datafixed) \cdot \Pr_{\text{coins of \sampler}}(\sampler(\Datafixed)=z) }\, .$$
The dataset size $n$ is a parameter of $\sampler$ and thus defined implicitly.%
\ifnum\neurips=0
\footnote{We sometimes consider datasets whose size is randomly distributed. In such cases, the random variable $\sampler(\Datarv)$ (and its distribution $\distroutput{\sampler, \distr}$) incorporates randomness from the choice of the input size in addition to the selection of the data and the coins of $\sampler$.}
\fi 
~We measure the closeness between the input and output distributions in total variation distance, denoted $d_{TV}$.

\begin{definition}[Accuracy of sampling \cite{Axelrod0SV20}]
\label{def:acc1} A sampler $\sampler$ is {\em $\alpha$-accurate on a distribution} $\distr$ 
if 
\begin{equation*}
    d_{TV}(\distroutput{\sampler, \distr}, \distr) \leq \alpha.
\end{equation*}
A sampler is {\em $\alpha$-accurate on a class $\class$ of distributions} if it is $\alpha$-accurate on every $\distr$ in $\class$.
\end{definition}


The class $\class$ in the accuracy definition above effectively encodes what the sampler is allowed to assume about $\distr$. For example, $\class$ might include all distributions on $k$ elements or all product distributions on $\bit{d}$ (that is, distributions on $d$-bit strings for which the individual binary entries are independent). 

Our aim in formulating Definition~\ref{def:acc1} was to capture the weakest reasonable 
task that abstracts the goal of generating data with the same distributional properties as the input data.
Without any privacy constraints, this task is trivial to achieve: an algorithm that simply outputs its first input record will sample from exactly the distribution $\distr$, so the interesting problem  is to generate a sample of size $m>n$ when given only $n$ observations (as in the work of Axelrod
et al.~\cite{Axelrod0SV20}, which inspired our investigation). However, a differentially private algorithm cannot, in general, simply output one of its records in the clear. Even producing a single correctly distributed sample is a non-trivial task. 

To understand the task and compare our results to existing work, it is helpful to contrast our definition of sampling with the more stringent goal of learning a
distribution. A learning algorithm gets input in the same format as a sampling algorithm. We state a definition of distribution learning formulated 
with a single parameter $\alpha$ that captures both the distance between distributions and the failure probability.

\begin{definition}[Distribution learning]
\label{def:learn-intro} An algorithm $\cB$
{\em learns 
a distribution class $\class$} to within error $\alpha$ if, given a dataset consisting of independent observations from a distribution $P\in\class$, algorithm $\cB$ outputs a description of a distribution  that satisfies
\begin{equation*} 
    \Pr_{\substack{\Datarv~\sim_{i.i.d.} \distr \\ \text{coins of \cB}}}\bparen{d_{TV}(\cB(\Datarv), \distr) \leq \alpha} \geq 1-\alpha.
\end{equation*}
\end{definition}

An algorithm $\cB$ that learns class $\class$ to within error $\alpha$ immediately yields a $2\alpha$-accurate sampler for class $\class$: 
the sampler first runs $\cB$ to get a distribution $\hat P$ and then generates a single sample from~$\hat P$. 
Thus, it is instructive to compare bounds on the dataset size required for sampling to known results on the sample complexity of distribution learning.
Lower bounds for sampling imply lower bounds for all more stringent tasks, including learning, whereas separations between the complexity of sampling and that of learning suggest settings where weakening the learning objective might be productive.

\paragraph{Differential privacy} Differential privacy is a constraint on the algorithm $\sampler$ that processes a dataset. It does not rely on any distributional property of the data. We generally use uppercase letters (e.g., $\Datarv =(\datarv_1,...,\datarv_n)$) when viewing the data as a random variable, and lowercase symbols (e.g., $\Datafixed = (\datafixed_1,...,\datafixed_n)$) when treating the data as fixed. 
Two datasets $\Datafixed, \Datafixed'\in \universe^n$ are {\em neighbors} if they differ in at most one entry. If each entry  corresponds to the data of one person, then neighboring datasets differ by replacing one person's data with an alternate record. Informally, differential privacy requires that an algorithm's output distributions are similar on all pairs of neighboring datasets.

\begin{definition}[Differential Privacy~\cite{DworkMNS06j,DworkKMMN06}]\label{def:DP} A randomized algorithm $\sampler: \universe^n \rightarrow \mathcal{Y}$ is {\em $(\eps, \delta)$-differentially private} (in short, $(\eps,\delta)$-DP) if for every pair of neighboring datasets $\Datafixed, \Datafixed'\in \universe^n$ and for all subsets $Y\subseteq \mathcal{Y}$,
\begin{equation*}
    \Pr[\sampler(\Datafixed) \in Y] \leq e^\eps \cdot \Pr[\sampler(\Datafixed') \in Y] + \delta.
\end{equation*}
\end{definition}

\ifnum\tpdp=0 
For $\eps \leq 1$ and $\delta < 1/n$, this guarantee implies, roughly, that one can learn the same things about any given individual from the output of the algorithm as one could had their data not been used in the computation \cite{KasiviswanathanS14}. 
When $\delta=0$, the guarantee is referred to as pure differential privacy.
\ifnum\neurips=0
We generally consider the approximate case where $\delta>0$. 
In the analysis of some of the algorithms,
we also use a variant of the definition, \emph{zCDP}~\cite{bun2016concentrated} (see Definition~\ref{def:CDP}).
\fi 
\fi

\ifnum\neurips=1
\subsection{Our results}
\else 
    \ifnum\tpdp=1
    \paragraph{Our Results}
    \else 
    \subsection{Our Results}
    \fi
\fi 
\label{sec:results}

We initiate a systematic investigation of the sample complexity of differentially private sampling. We provide upper and lower bounds for the  sample size needed for this task for three natural families of distributions: arbitrary distributions on $[k]=\{1,\ldots ,k\}$, arbitrary product distributions on $\{0,1\}^d$, and product distributions on $\{0,1\}^d$ with bounded bias. We demonstrate that, in some parameter regimes, private sampling requires asymptotically fewer samples than learning a description of $P$ nonprivately; in other regimes, however, sampling proves to be as difficult as private learning. 


\begin{table}[t] 
    \begin{center} \small
     \caption{Sample complexity of sampling and estimation tasks. Our negative results hold for $(\eps,\delta)$-differential privacy when $\delta<1/n$. In this table, $\eps \leq 1$ and $\delta = 1/n^c$ for constant $c>1$.}
    \label{tab:results}
\begin{tabular}{ |c||c||c|c| } 
      \hline
 & Distributions on $[k]$ 
 & Product distributions 
 & Product distributions 
 \\
 & 
 & on $\bit{d}$
 & with $p_j \in [\frac 1 3 , \frac 2 3]$ 
 \\ 
 \hline
 {{\color{white} \rule{1pt}{11pt}} Nonprivate learning } & 
 $\Theta\paren{\dfrac{k}{\alpha^2}}$
 & \multicolumn{2}{|c|}{$\Theta\paren{ \dfrac{d}{\alpha^2}}$}
 \\ 
 \hline
 {{\color{white} \rule{1pt}{11pt}} $(\eps, \delta)$-DP sampling} 
 & $\Theta\paren{\dfrac{k}{\alpha \eps}}$ 
 & $\tilde \Theta\paren{\dfrac{d }{\alpha \eps}}$
 & $\tilde O \paren{\dfrac{\sqrt{d}}{\eps} + \log\dfrac{d}{\alpha}}$,
 \quad $\Omega\paren{\dfrac{\sqrt{d}}\eps}$ 
 
 \\
 (this work) & Theorems~\ref{thm:kupperb}, \ref{thm:intro-k-ary-lb} & Theorems~\ref{thm:bernoulli-product-alg-intro}, \ref{thm:bernoulli-product-lb-intro}  & Theorems~\ref{thm:intro-bernoulli-product-bb}, \ref{thm:intro-product-bb-lb}
 \\
 \hline
 {{\color{white} \rule{1pt}{11pt}} $(\eps, \delta)$-DP learning} & $\Theta\paren{\dfrac{k}{\alpha^2} + \dfrac{k}{\alpha \eps}}$ \quad \cite{diakonikolas2015differentially} 
 & \multicolumn{2}{|c|}{$\tilde \Theta\paren{\dfrac{d }{\alpha \eps} + \dfrac{d}{\alpha^2}}$ \cite{KLSU19,BunKSW21J}}
 \\ 
 \hline
\end{tabular}
\end{center}
\end{table}

Our results are summarized in Table~\ref{tab:results}.
\ifnum\neurips=1
Proofs of all results are included in the supplementary material.
\fi
\ifnum\tpdp=1
Proofs of all results are included in the full version, which is forthcoming on arxiv.
\fi 
For simplicity, the table and our informal discussions focus  on $(\eps,\delta)$-differential privacy (Definition~\ref{def:DP})  in the setting\footnote{This setting precludes trivial solutions (which are  possible when $\delta =\Omega(1/n)$), but allows us to treat factors of $\log(1/\delta)$ as logarithmic in $n$ and absorb them in $\tilde O$ expressions.} where $\delta = 1/n^c$ for some constant $c>1$.

Let $\class_k$ be the class of $k$-ary distributions (that is, distributions on $[k]$). We show that  $n=\Theta\bparen{\frac{k}{\alpha \eps}}$ observations are necessary and sufficient for differentially private sampling from $\class_k$ in the worst case.

\begin{theorem}\label{thm:kupperb}
For all $k\geq 2,\eps>0$, and  $\alpha\in (0,1)$, there exists an $(\eps,0)$-DP sampler 
that is  $\alpha$-accurate on the distribution class $\class_k$ for datasets of size $n=O(\frac{k}{\alpha \eps})$.
\end{theorem}

\begin{theorem}\label{thm:intro-k-ary-lb}
For all  $k\geq 2$, $n \in \N$, $\alpha \in (0,\frac 1 {50}]$, $\eps  \in (0,1]$, and $\delta \in (0, \frac{1}{5000n}]$, 
if there is an 
$(\eps,\delta)$-DP sampler  that is $\alpha$-accurate on the distribution class $\class_k$ 
on datasets of size $n$, then 
$n = \Omega\paren{ \frac{k}{\alpha\eps}}$.
\end{theorem}

The second major class we consider consists of products of $d$ Bernoulli distributions, for $d\in \N$. We denote this class by $\cB^{\otimes d}$.  Each distribution in $\cB^{\otimes d}$ is described by a vector $(p_1,...,p_d) \in [0,1]^d$ of $d$ probabilities, called \emph{biases}; a single observation in $\bit{d}$ is generated by flipping $d$ independent coins with respective probabilities $p_1,...,p_d$ of heads. 
We show that  $n=\tilde\Theta\bparen{\frac{d}{\alpha \eps}}$ observations are necessary and sufficient for differentially private sampling from $\cB^{\otimes d}$ in the worst case.

\begin{theorem}\label{thm:bernoulli-product-alg-intro}
For all $d,n\in\N$ and $\eps,\delta,\alpha  \in (0,1)$, 
there exists an $(\eps, \delta)$-DP sampler 
that is $\alpha$-accurate on the distribution class $\cB^{\otimes d}$ for datasets of size $n=\tilde{O}\big(\frac{d}{\alpha\eps}\big)$, assuming $\log(1/\delta) = poly(\log n)$.
\end{theorem}

\begin{theorem}\label{thm:bernoulli-product-lb-intro}
For all sufficiently small $\alpha>0$, and
for all $d,n\in \N$, 
$\eps \in (0, 1]$, and $\delta \in \big[0, \frac 1{5000 n}\big]$,  
if there is an 
$(\eps,\delta)$-DP sampler that is $\alpha$-accurate on the distribution class $\cB^{\otimes d}$ on datasets of size $n$, then 
$n=\Omega(\frac{d}{\alpha \eps})$.
\end{theorem}

Finally, we give better samplers and matching lower bounds 
for  Bernoulli distributions and,  more generally, products of Bernoulli distributions, with bias bounded away from 0 and 1. For simplicity, we consider distributions with bias $p_j\in [\frac 1 3, \frac 2 3]$ in each coordinate $j\in[d]$.
For this class, we show that differentially private sampling can be performed with datasets of size roughly $\sqrt{d}/\eps$, significantly smaller than in the general case. Curiously, the accuracy parameter $\alpha$ has almost no effect on the sample complexity.
For Bernoulli distributions with bounded bias, we achieve this with pure differential privacy, that is, with $\delta =0$.  For products of Bernoulli distributions, we need $\delta>0$.

\begin{theorem}
\label{thm:intro-bernoulli-product-bb}
For all $d\in\N$ and $\eps,\delta,\alpha  \in (0,1)$,
there exists an $(\eps,\delta)$-DP sampler that is  $\alpha$-accurate on the class 
of products of $d$ Bernoulli distributions with biases in $\big[\frac 13,\frac 23 \big]$  for datasets of size $n=O\Big(\frac {\sqrt{d\log(1/\delta)}}{\eps}   + \log \frac d {\alpha}
\Big)$.
When $d=1$, the sampler has $\delta=0$ and $n=O(\frac 1{\eps}+\log \frac 1 {\alpha})$. 
\end{theorem}

\begin{theorem}
\label{thm:intro-product-bb-lb} 
For all sufficiently small $\alpha>0$, 
and for all $d,n\in \N$, 
$\eps \in (0,1]$,  and $\delta \in [0, \frac 1{100n}]$, 
if there exists an 
$(\eps,\delta)$-DP sampler that is $\alpha$-accurate on
the class 
of products of $d$ Bernoulli distributions with biases in $\big[\frac 13,\frac 23 \big]$ 
on datasets of size $n$, then  
$n={\Omega}(\sqrt{d}/\eps)$.
%
\end{theorem}

\paragraph{Implications}

Our results show that the sample complexity of private sampling can differ substantially from that of private learning (for which known bounds are stated in Table~\ref{tab:results}). In some settings, sampling is much easier than learning: for example, for products of Bernoulli distributions with bounded biases, private sampling has a lower dependence on the dimension (specifically, $\sqrt{d}$ instead of $d$) and essentially no dependence on $\alpha$. Even for arbitrary biases or arbitrary $k$-ary distributions, private sampling is easier when $\alpha\ll\eps$. In other settings, however, private sampling can be as hard as private learning: e.g., for $\eps \leq \alpha$, the worst-case complexity of sampling and learning $k$-ary distributions and product distributions is the same. 

A more subtle point is that, in settings where private sampling is as hard as private learning, sampling accounts for the entire cost of privacy in learning. Specifically,  the optimal sample complexity of differentially private learning for arbitrary $k$-ary distributions is $n=\Theta\bparen{\frac{k}{\alpha^2} + \frac{k}{\alpha \eps}}$ (e.g., see \cite[Theorem 13]{AcharyaSZ21}). 
This bound is the sum of two terms: the sample complexity of nonprivate learning, $\Theta\bparen{\frac k {\alpha^2}}$, plus a term to account for the privacy constraint, $\Theta\bparen{\frac k {\alpha \eps}}$. One interpretation of our result that private sampling requires $n=\Theta\bparen{\frac k {\alpha \eps}}$ observations is that \emph{the extra privacy term in the complexity of learning can be explained by the complexity of privately generating a single sample with approximately correct distribution}. 
\ifnum\neurips=0
Analogously, for product distributions with arbitrary biases, the sample complexity of private learning is the sum of the sample complexity of nonprivate learning, $\Theta\paren{\frac{d}{\alpha^2}}$, and the extra privacy term, $\tilde \Theta\paren{\frac{d }{\alpha \eps}}$. 
Our results demonstrate that the latter term is the number of observations needed to produce one private sample, giving a new interpretation of the overhead in the number of samples stemming from the privacy constraint.
\fi

Another implication of our results is that, \emph{in some settings, the distributions that are hardest for learning---nonprivate or private---are the easiest for sampling, and vice versa}. Consider the simple case of Bernoulli distributions (i.e., product distributions with $d=1$). The ``hard'' instances for nonprivate learning to within error $\alpha$ are distributions with bias $p = \frac{1\pm\alpha}{2}$, but private sampling is easiest in that parameter regime. In contrast, the ``hard'' instances in our $\Omega(\frac{1}{\alpha\eps})$ lower bound for Bernoulli distributions have bias  $10\alpha$, that is, close to 0 as opposed to close to 1/2.
A simple variance argument shows that nonprivate learning is easy in that parameter regime, requiring only $O(\frac1 \alpha)$ observations. Similarly, for product distributions, we show that the complexity of private sampling is only $\tilde \Theta(\sqrt d)$ when biases are bounded away from 0 and 1. For the same class, however, the complexity of private and nonprivate learning is $\Theta(d)$.

Our final point is
that our lower bounds for $k$-ary distributions and general product distributions only require that the sampler generate a value \textit{in the support} of the distribution with high probability. They thus apply to a weaker problem, similar in spirit to the interior point problem that forms the basis of lower bounds for private learning of thresholds \cite{BunNSV15}. 

Taken together, our results show that  studying the complexity of generating a single sample helps us understand the true source of difficulty for certain tasks and sheds light on when we might be able to engage in nontrivial statistical tasks with very little data.

 
\ifnum\tpdp=0

\paragraph{Open questions}
Our work raises many questions about the complexity of private sampling. First, our upper bounds achieve only the minimal goal of generating a single observation. In most settings, one would presumably want to generate several such samples. One can do so by repeating our algorithms several times on disjoint subsamples, but  in general this is not the best way to proceed. 
Second, we study only three classes of distributions. It is likely that the picture of what is possible for many classes is more complex and nuanced. It would be interesting, for example, to study private sampling for 
Gaussian distributions, since they demonstrate intriguing data/accuracy tradeoffs for nonprivate sampling~\cite{Axelrod0SV20}.


\ifnum\neurips=1
\subsection{An overview of our proofs and techniques}\label{intro:overview}
\else 
\subsection{An Overview of Our Proofs and Techniques}\label{intro:overview}
\fi 
\label{sec:techniques}

For both algorithms and lower bounds, our results require the development of new techniques. On the algorithmic side, we take advantage of the fact that sampling algorithms need only be correct on average over samples drawn from a given distribution. One useful observation that underlies  our positive results is that sampling based on an \emph{unbiased} estimate $\hat P$ of a probability distribution $P$ (in the sense that 
$\E[{\hat{P}(u)]=P(u)}$ for all elements $u$ in the universe $\universe$, where the expectation is taken over the randomness in the dataset and the coins of the algorithm)
has zero error, even though the learning error, e.g., $d_{TV}(\hat P,P)$ might be  large. For product distributions with bounded biases, we  also exploit the randomness of the sampling process itself to gain privacy without explicitly adding noise.

For negative results, we cannot generally use existing lower bounds for learning or estimation, because of a fundamental  obstacle. 
The basic framework used in proving most lower bounds on sample complexity of learning problems is based on \textit{identifiability}: to show that a large sample is required to learn class $\class$, one first finds a set of distributions $P_1,....,P_t$ in the class $\class$ that are far apart from each other and then shows that the output of any sufficiently accurate learning algorithm allows an outside observer to determine exactly which distribution in the  collection generated the data. The final step is to show that algorithms in a given family (say, differentially private algorithms with a certain sample size) cannot reliably identify the correct distribution.
This general approach is embodied in recipes such as Fano's, Assouad's, and Le Cam's methods from classical statistics (see, e.g., \cite{AcharyaSZ21} for a summary of these methods and their differentially private analogues).
For many sampling problems, the identifiability approach breaks down: a single observation is almost never enough to identify the originating distribution. 

One of our approaches to proving lower bounds is to leverage ways in which the algorithm's output directly signals a failure to sample correctly. For instance, our lower bound for $k$-ary distribution relies on the fact that an $\alpha$-accurate sampler must produce a value in the support of the true distribution $P$ with high probability. Another approach is to reduce from other distribution (sampling or estimation) problems. For example, our lower bound for product distributions with bounded biases is obtained via a reduction from an estimation problem, by observing that a small number of samples from a nearby distribution suffices for a very weak estimate of the underlying distribution's attribute biases.

\medskip

We break down our discussion of techniques according to the specific distribution classes we consider. 

\paragraph{Distributions on ${[k]}$} 
For the class of distributions on $[k]$, Theorem~\ref{thm:kupperb} shows that $\alpha$-accurate $(\eps,0)$-differently private sampling can be performed with a dataset of size  $O(\frac k {\alpha\eps})$. 
Our private sampler  computes, for each $j\in[k]$, the proportion $\hat P_j$ of its dataset that is equal to $j$, adds Laplace noise to each count, uses $L_1$ projection to obtain a valid vector of probabilities $\tilde P =(\tilde P_1,\dots,\tilde P_k)$, and finally outputs an element of $[k]$ sampled according to $\tilde P$.

Theorem~\ref{thm:intro-k-ary-lb} provides a matching lower bound on $n$ that holds for all $(\eps,\delta)$-differentially private algorithms with $\delta=o(1/n).$
We prove our lower bound separately for Bernoulli distributions and for discrete distributions with support size $k\geq 3$, using different analyses.
For Bernoulli distributions, we first exploit the group privacy property of  differentially private algorithms and the fact that the sampler must be accurate for the Bernoulli distribution $\Ber(0)$ to show that, on input with $t$ ones, a differentially private sampler outputs 1 with probability at most
$2\alpha e^{\eps t}.$ Then we consider $P=\Ber(10\alpha)$. We use $\alpha$-accuracy in conjunction with group privacy to give a lower and an upper bound on the probability of the output being~1 when the input is drawn i.i.d.\ from $P$. This allows us to relate the parameters involved in order to obtain the desired lower bound on~$n.$

The lower bound for distributions on $[k]$ with $k\geq 3$ is more involved. We start by identifying general properties of samplers that allow us to restrict our attention to relatively simple algorithms. First, we observe that every sampler can be converted to a {\em Poisson algorithm}, that is, an algorithm that, instead of receiving an input of a fixed size, gets an input with the number of records that follows a Poisson distribution.
This observation allows us to use a standard technique, Poissonization, that  
makes the empirical frequencies of different elements independent.
Next, we observe that privacy for samplers can be easily amplified, so that we can assume w.l.o.g.\ that $\eps$ is small. Finally, we observe that every sampler for 
\ifnum\neurips=1
the class of $k$-ary distributions
\else
a {\em label-invariant} class
\fi
can be converted to a {\em frequency-count-based algorithm}. 
\ifnum\neurips=0
A class $\class$ of distributions is {\em label-invariant} if, for every distribution $P\in\class$, every distribution obtained from $P$ by permuting the names of the elements in the support is also in $\class.$ 
\fi
A sampler is {\em frequency-count-based} if the probability it outputs a specific element depends only on the number of occurrences of this element in its input and the frequency counts\footnote{The vector of frequency counts is called a {\em fingerprint} or a {\em histogram} in previous work.
\ifnum\neurips=0
We avoid the first term to ensure there is no confusion with fingerprinting codes that are used in other lower bounds for differential privacy. In the literature on differential privacy, a histogram refers to the vector of frequencies of each element in the dataset, as opposed to the number of elements with the specified frequencies. To avoid confusion, we use the term ``frequency counts''.
\fi
} of the input (that is, the number of elements that occur zero times, once, twice, thrice, and so on).  Frequency-count-based algorithms have been studied for a long time in the context of understanding properties of distributions (see, e.g.,~\cite{batu2001testing, batu2000testing, raskhodnikova2009strong}).

Equipped with the three observations, we restrict our attention to Poisson, frequency-count-based algorithms, with small $\eps$ in the privacy guarantee. In contrast to our lower bound for Bernoulli samplers, we show that when the support size is at least 3 and the dataset size, $n$, is too small, the sampler is likely to output an element outside of the support of the input distribution $\distr.$ Here, we exploit group privacy, which implies that the probability that a sampler outputs a specific element which appears $j$ times in its input differs by at most a factor of $e^{\eps j}$ from the probability that it outputs a specific element that does not appear in the input. Then we consider a distribution $\distr$ that has most of its mass (specifically, $1-O(\alpha)$) on a special element, and the remaining mass spread uniformly among half of the remaining domain elements. That is, $\distr$ is a mixture of a unit distribution on the special element and a uniform distribution on half of the remaining elements. We show that, when the dataset size is too small, the sampler is nearly equally likely to output any non-special element. But it has to output non-special elements with probability $\Omega(\alpha)$ to be $\alpha$-accurate. This means that when the database size is too small, the sampler outputs a non-special element outside the support of $\distr$ with probability $\Omega(\alpha)$.
The details of the proof are quite technical and appear in 
\ifnum\neurips=1 
    Section~\ref{sec:kary-short}.
\else
    Section~\ref{sec:k-ary-lb-final}.
\fi 


\paragraph{Product distributions}
Our private sampler for product distributions over $\{0,1\}^d$, used to prove Theorem~\ref{thm:bernoulli-product-alg-intro},
builds on the recursive private preconditioning procedure designed by Kamath et al.~\cite{KLSU19} for learning this distribution class. In our case, the sampler gets a dataset of size which is asymptotically smaller than necessary for learning this distribution class in some important parameter regimes. Let $(\biasesfixed_1, \dots, \biasesfixed_d)$ be the attribute biases for the product distribution $\distr $ from which the data is drawn. For simplicity,  assume w.l.o.g.\ that all the marginal biases $\biasesfixed_j$ are less than 1/2. The main idea in~\cite{KLSU19} is that smaller biases have lower sensitivity in the following sense: if we know that a set of attributes has biases $p_j$ that are all at most some bound $u$, then, since the data is generated from a product distribution, the number of ones in those attributes should at most $2ud$ with high probability. We can enforce this bound on the number of ones in those coordinates by truncating each row appropriately,  and thus learn the biases of those coordinates to higher accuracy than we knew before. 
Building  on that idea, we partition the input into smaller datatsets and run our algorithm in rounds, each using fresh data, a different truncation threshold, and noise of different magnitude.

Our algorithm consists of two phases. In the bucketing phase, we use half of the dataset and the technique of \cite{KLSU19} to get rough \textit{multiplicative} estimates of all biases $p_j$ except the very small ones (where $p_j<1/d$). This allows us to partition the coordinates into $\log(d)$ buckets, where biases within each bucket are similar. We show this crude estimation only requires $n=\tilde O(d/\alpha \eps)$. 
%
%
In the sampling phase, we use the buckets to generate our output sample. For each bucket, we can get a fresh estimate of the biases using the other half of the dataset and, again, the technique from  \cite{KLSU19} to scale the noise proportionally to the upper bound on the biases in that bucket. These estimates are essentially unbiased. Flipping $d$ coins independently according to the estimated biases produces an observation with essentially the correct distribution. 

The proof of our lower bound for general product distributions proceeds via reduction from sampling of $k$-ary distributions for $k=d$. {In contrast,} 
lower bounds for private \textit{learning} of product distributions rely on fingerprinting codes (building on the framework of Bun et al. \cite{BunUV14j}). Although fingerprinting codes are indeed useful when reasoning about samplers for product distributions with bounded biases (discussed below), our approach relies, instead, on the fact that samplers must distinguish coordinates with bias 0 from coordinates with small bias. Specifically, given a distribution $\distr$ on $[2k]$ that is uniform on a subset $S$ of $[2k]$ of size $k$, we define a product distribution $\corr$ on $\bit{2k}$ with biases $p_j = \distr(j) = \frac{1}{k}$ for $j\in S$ and $p_j=0$ otherwise. We use Poissonization and coupling between Poisson and binomial distributions
%
to design a reduction that, given observations from $\distr$, first creates an appropriately distributed sample of almost the same size drawn according to $\corr$, then runs a hypothetical sampler for product distributions to get a vector in $\bit{2k}$ (drawn roughly according to $\corr$), and finally converts that vector back to a single element of $[2k]$ distributed roughly as $\distr$. The details are subtle since most draws from $\corr$ will not have exactly one nonzero element---this complicates conversion in both directions.  \ifnum\neurips=0 See Section~\ref{sec:prod-lb}. \fi

\paragraph{Distributions with bounded bias}

Interestingly, our sampler for product distributions with bounded bias does not directly add noise to data. It performs the following step independently for each attribute: compute the empirical mean of the sample in that coordinate, obtain the clipped mean by rounding the empirical mean to the interval $[1/4,3/4]$, and sample a bit according to the clipped mean. 
The key idea in the analysis of accuracy is that, conditioned on no rounding (that is, the empirical mean already being in the required interval), the new bit is sampled from the correct distribution, and rounding occurs with small probability.
We argue that this sampler is $(4/n,0)$-differentially private for the case when $d=1$.
For larger $d$, the sampler is a composition of $d$ differentially private algorithms, and the main bound follows from compositions theorems (and conversions between different variants of differential privacy).
\ifnum\neurips=0
\asnote{Wishlist: Explain somewhere that these algorithms show you can do better than just finding a DP approximation in TV to the empirical distribution?}
\fi

The lower bound for this class proceeds by a reduction from the following weak estimation problem: Given samples from a product distribution with biases $p_1,p_2,...,p_d$, output estimates $\tilde p_1,...,\tilde p_d$ such that each $\tilde p_i$ is within additive error $\frac1{20}$ of $p_i$, that is, $|\tilde p_i - p_i|\leq \frac 1 {20}$, with probability at least $1-\frac 1 {20}$ (where $\frac1{20}$ is just a sufficiently small constant). This problem is known to require datasets of size $\tilde\Omega(\sqrt{d})$ for differentially private algorithms~\cite{BunUV14j}. However, nonprivate algorithms require only $O(1)$ data records to solve this same problem! We can thus reduce from estimation to sampling with very little overhead: Given a private sampler, we run it a constant number of times on disjoint datasets to obtain enough samples to (nonprivately) estimate each of the $p_i$'s. The nonprivate estimation is a postprocessing of an output of a differentially private algorithm, so the overall algorithm is differentially private. Some care is required in the reduction, since we must ensure that the nonprivate estimation algorithm is robust (i.e., works even when the samples are only close in TV distance to the correct distribution) and that the  lower bound of \cite{BunUV14j} applies even when the biases $p_i$ lie in $[\frac 1 3, \frac 2 3]$.

\ifnum\neurips=0
The approach used to prove the lower bound is quite general: informally, we can reduce from a differentially private estimation problem for a class of distributions $\class$ to a combination of differentially private sampling and a non-private estimation problem for $\class$ as follows: Given a dataset sampled independently from a distribution $\distr$ in a class $\class$, first split the dataset into many parts, and run a private sampler on each part to get a set of samples that look like they were sampled independently from $\distr$. Then use these samples to solve the estimation problem non-privately. This simple approach does not always yield tight bounds---for example, for general products of $d$ Bernoulli distributions and discrete $k$-ary distributions it gives lower bounds of $\tilde{\Omega}(\frac{\sqrt{d}}{\eps})$ and $\Omega(1/\eps)$ respectively on the number of samples required for $\alpha$-accurate sampling. However, it does give a near-tight lower bound for bounded-bias product distributions. 

\fi

\fi

\ifnum\tpdp =0 
\subsection{Related work} 
\else
\paragraph{Related Work}
\fi

To the best of our knowledge, the private sampling problem we formulate has not been studied previously. There is work on nonprivate sample-size amplification, in which an initial dataset of $n$ records is used to generate an output sample of size $n'>n$ from a nearby distribution \cite{Axelrod0SV20, AxelrodGHS22}. Our formulation corresponds to the  case where $n'=1 \ll n$. 
There is also work on sampling (also called \textit{simulation}) in a distributed setting~\cite{AcharyaCT20,AcharyaCT20a}
where each of $n$ participants receives a single observation and is limited to $\ell\geq 1$ bits of communication. We are not aware of a direct technical connection with our work, even though the distributed setting is closely tied to that of \textit{local} differential privacy.

In the literature on differential privacy, there is work on private algorithms for sampling parameters from Bayesian posterior distributions  \cite{DimitrakakisNMR14,DimitrakakisNZM17,ZhangRD16,WangFS15}, often driven by the intuition that the randomness inherent in sampling provides some level of privacy ``for free''. Our algorithms for product distributions with bounded biases leverage a similar idea. We are not aware of any analysis of the Bayesian approach which provides guarantees along the lines of our formulation (Def.~\ref{def:acc1}). 

The most substantially related work in the privacy literature is on upper and lower bounds for learning and estimation tasks, dating back to Warner~\cite{Warner65} (see \cite{DworkS09,KamathUprimerpaper20} for partial surveys). The investigation of differentially private learning was initiated in \cite{KasiviswanathanLNRS11}. For learning $k$-ary distributions, the directly most relevant works are those of \cite{diakonikolas2015differentially,AcharyaSZ21}, though lower bounds for estimating histograms were known as folklore since the mid-2000s. For product distributions, tight upper and lower bounds for  $(\eps,\delta)$-differentially private estimation in TV distance appear in \cite{KLSU19}; the upper bound was later shown to be achievable by an $(\eps,0)$-differentially private algorithm \cite{BunKSW21J}. Our upper bound for general product distributions adapts the technique of~\cite{KLSU19}; it remains open whether our bound is achievable with $\delta=0$. 

Less directly related is the line of work on the generation of synthetic datasets that match the data distribution in a set of specified characteristics, such as linear queries~(e.g., \cite{blum2013learning}; see \cite{Vadhan17} for a tutorial). The goal in those works is to generate enough synthetic data to allow an analyst given only the output to estimate the specified characteristics---a much more difficult task than sampling, and one for which a different set of lower bound techniques apply~\cite{Vadhan17}. 
%

\ifnum\neurips=1

\input{lb-proof-short-neurips}

\begin{ack} We are grateful for helpful conversations with Clément Canonne, Thomas Steinke, and Jonathan Ullman. 
Sofya Raskhodnikova was partially supported by NSF award CCF-1909612.
Satchit Sivakumar was supported in part by NSF award CNS-2046425, as well as Cooperative Agreement CB20ADR0160001 with the Census Bureau. 
Adam Smith and Marika Swanberg were supported in part by NSF award CCF-1763786 as well as a Sloan Foundation research award.
The views expressed in this paper are those of the authors and not those of the U.S. Census Bureau or any other sponsor.
\end{ack}

\bibliographystyle{plain}
\bibliography{refs}{}

\fi

\ifnum\supplemental=1 
\newpage
\pagenumbering{roman}
\appendix

\begin{center}
    \Huge Supplementary Materials 
    
    \rule{3in}{1pt}
    
    \Large Differentially Private Sampling from Distributions
\end{center}

These supplementary materials are organized as follows. Section~\ref{sec:defs} collects standard definitions and mathematical tools.  Next, in Section~\ref{sec:properties}, we describe general transformations of samplers that we use in our lower bounds. In Section~\ref{sec:kary}, we prove upper and lower bounds for the task of private sampling from $k$-ary distributions, corresponding to Theorems~\ref{thm:kupperb} and~\ref{thm:intro-k-ary-lb} in the introduction. In Section~\ref{sec:prod}, we prove upper bounds for private sampling from product distributions over $\{0,1\}^d$, corresponding to Theorem~\ref{thm:bernoulli-product-alg-intro}. We defer the proof of Theorem~\ref{thm:bernoulli-product-lb-intro} to the full version of the paper. In Section~\ref{sec:bounded-bias}, we present our upper and lower bounds for private sampling from product distributions with bounded attribute biases, corresponding to Theorems~\ref{thm:intro-bernoulli-product-bb} and~\ref{thm:intro-product-bb-lb} in the introduction. Finally, in Section~\ref{sec:inequalities}, we discuss some standard results that we use in our proofs and, in Section~\ref{sec:resow}, we state some results from other papers that we use in our proofs.

\fi

\ifnum\tpdp=0
\ifnum\neurips=0
\section{Definitions}
\label{sec:defs}

\subsection{Differential Privacy}

A dataset $\Datafixed = (\datafixed_1, \ldots, \datafixed_n) \in \universe^n$ is a vector of elements from universe \universe. Two datasets are {\em neighbors} if they differ in at most one coordinate. Informally, differential privacy requires that an algorithm's output distributions are similar on all pairs of neighboring datasets. We use two different variants of differential privacy. The first one (and the main one used in this paper) is the standard definition of differential privacy.

\begin{definition}[Differential Privacy~\cite{DworkMNS06j,DworkKMMN06}]\label{def:differentially private} A randomized algorithm $\sampler: \universe^n \rightarrow \mathcal{Y}$ is {\em $(\eps, \delta)$-differentially private} if for every pair of neighboring datasets $\Datafixed, \Datafixed'\in \universe^n$ and for all subsets $Y\subseteq \mathcal{Y}$,
 \begin{equation*}
    \Pr[\sampler(\Datafixed) \in Y] \leq e^\eps \cdot \Pr[\sampler(\Datafixed') \in Y] + \delta.
 \end{equation*}
 \end{definition}
 In addition to standard differential privacy (Definition~\ref{def:DP}), we use a variant called {\em zero-mean concentrated differential privacy} \cite{bun2016concentrated} that is defined in terms of R\'enyi divergence.
\begin{definition}[R\'enyi divergence] Consider two probability distributions $P$ and $Q$ over a discrete domain~$S$. Given a positive $\alpha\neq 1$, R\'enyi divergence of order $\alpha$ of distributions $P$ and $Q$ is 
\begin{equation*}
    D_\alpha (P || Q) = \frac{1}{1-\alpha} \log\left(\sum_{\datafixed \in S} P(\datafixed)^\alpha Q(\datafixed)^{1-\alpha} \right).
\end{equation*}

\end{definition}

\begin{definition}[Zero-Mean Concentrated Differential Privacy (zCDP)~\cite{bun2016concentrated}] \label{def:CDP}
A randomized algorithm $\sampler : \universe^n \rightarrow \mathcal{Y}$ is $\rho$-zCDP if for every pair of neighboring datasets $\Datafixed, \Datafixed' \in \universe^n$,
\begin{equation*}
    \forall \alpha \in (1, \infty) \quad D_\alpha\left(\sampler(\Datafixed) ||\sampler(\Datafixed')\right) \leq \rho \alpha,
\end{equation*}
where $D_\alpha(\sampler(\Datafixed) ||\sampler(\Datafixed'))$ is the $\alpha$-R\'enyi divergence between $\sampler(\Datafixed)$ and $\sampler(\Datafixed')$.
\end{definition}

\begin{lemma}[Relationships Between $(\eps, \delta)$-Differential Privacy and $\rho$-CDP~\cite{bun2016concentrated}]\label{prelim:relate_dp_cdp} For every $\eps \geq 0$,
\begin{enumerate}
    \item If \sampler is $(\eps, 0)$-differentially private, then \sampler is $\frac{\eps^2}{2}$-zCDP.
    \item If \sampler is $\frac{\eps^2}{2}$-zCDP, then \sampler is $\left(\frac{\eps^2}{2} + \eps\sqrt{2 \log(1/\delta)}, \delta\right)$-differentially private for every $\delta > 0$.
\end{enumerate}
\end{lemma}

Both definitions of differential privacy are closed under post-processing.
\begin{lemma}[Post-Processing~\cite{DworkMNS06j,bun2016concentrated}]\label{prelim:postprocess} If $\sampler: \universe^n \rightarrow \mathcal{Y}$ is $(\eps, \delta)$-differentially private, and $\mathcal{B} : \mathcal{Y} \rightarrow \mathcal{Z}$ is any randomized function, then the algorithm $\mathcal{B} \circ \sampler$ is $(\eps, \delta)$-differentially private. Similarly, if $\sampler$ is $\rho$-zCDP then the algorithm $\mathcal{B} \circ \sampler$ is $\rho$-zCDP.
\end{lemma}

Importantly, both notions of differential privacy are closed under adaptive composition. For a fixed dataset \Datafixed, {\em adaptive composition} states that the results of a sequence of computations satisfies differential privacy even when the chosen computation $\sampler_t(\cdot)$ at time $t$ depends on the outcomes of previous computations $\sampler_1(\Datafixed), \ldots, \sampler_{t-1}(\Datafixed)$. Under adaptive composition, the privacy parameters add up.

\begin{definition}[Composition of $(\eps, \delta)$-differential privacy and $\rho$-zCDP~\cite{DworkMNS06j,bun2016concentrated}]\label{prelim:composition} Suppose $\sampler$ is an adaptive composition of differentially private algorithms $\sampler_1, \ldots, \sampler_T$.
\begin{enumerate}
    \item If for each $t \in [T]$, algorithm $\sampler_t$ is $(\eps_t, \delta_t)$-differentially private, then \sampler is $\left(\sum_t \eps_t, \sum_t \delta_t\right)$-differentially private.
    \item If for each $t \in [T]$, algorithm $\sampler_t$ is $\rho_t$-zCDP, then \sampler is $\left(\sum_t \rho_t\right)$-zCDP.
\end{enumerate}
\end{definition}

Standard $(\eps, \delta)$-differential privacy protects the privacy of groups of individuals.

\begin{lemma}[Group Privacy]\label{prelim:group_privacy}
Each $(\eps, \delta)$-differentially private algorithm \sampler is $\left(k\eps, \delta\frac{e^{k\eps} -1}{e^\eps-1}\right)$-differentially private for groups of size $k$. That is, for all datasets $\Datafixed, \Datafixed'$ such that $\|\Datafixed - \Datafixed' \|_0 \leq k$ and all subsets $Y \subseteq \mathcal{Y}$,
\begin{equation*}
    \Pr[\sampler(\Datafixed) \in Y] \leq e^{k\eps} \cdot \Pr[\sampler(\Datafixed') \in Y] + \delta \cdot \frac{e^{k\eps} -1}{e^\eps-1}.
\end{equation*}
\end{lemma}

\paragraph{Laplace Mechanism} 

 Our algorithms use the standard Laplace Mechanism to ensure differential privacy. 

\begin{definition}[Laplace Distribution] The Laplace distribution with parameter $b$ and mean $0$, denoted by $\Lap(b)$, is defined for all $x \in \mathbb{R}$ and has probability density
\begin{equation*}
    h(\ell) = \frac{1}{2b}e^{-\frac{|\ell|}{b}}.
\end{equation*}
\end{definition}

\begin{definition}[$\ell_1$-Sensitivity] Let $f: \universe^n \rightarrow \mathbb{R}^d$ be a function. Its $\ell_1$-sensitivity is
\begin{equation*}
    \Delta_f = \max_{\substack{\Datafixed, \Datafixed' \in \universe^n \\ \Datafixed, \Datafixed' \text{neighbors}}} \|f(\Datafixed) - f(\Datafixed')\|_1.
\end{equation*}
\end{definition}

\begin{lemma}[Laplace Mechanism]\label{prelim:laplace_dp} Let $f : \universe^n \rightarrow \mathbb{R}^d$ be a function with $\ell_1$-sensitivity $\Delta_f$. Then the Laplace mechanism is algorithm
\begin{equation*}
    \sampler_f(\Datafixed) = f(\Datafixed) + (Z_1, \ldots, Z_d),
\end{equation*}
where $Z_i \sim \Lap\left(\frac{\Delta_f}{\eps}\right)$. Algorithm $\sampler_f$ is $(\eps, 0)$-differentially private.
\end{lemma}

\paragraph{Gaussian Mechanism} Our algorithms also use the common Gaussian Mechanism to ensure differential privacy.
\begin{definition}[Gaussian Distribution] The Gaussian distribution with parameter $\sigma$ and mean 0, denoted $\Gauss(\sigma)$, is defined for all $\ell \in \mathbb{R}$ and has probability density
\begin{equation*}
    h(\ell) = \frac{1}{\sigma \sqrt{2\pi}} e^{-\frac{\ell^2}{2\sigma^2}}.
\end{equation*}
\end{definition}

\begin{definition}[$\ell_2$-Sensitivity] Let  $f: \universe^n \rightarrow \mathbb{R}^d$ be a function. Its $\ell_2$-sensitivity is
\begin{equation*}
    \Delta_f = \max_{\substack{\Datafixed, \Datafixed' \in \universe \\ \Datafixed, \Datafixed' \text{neighbors}}} \|f(\Datafixed) - f(\Datafixed')\|_2.
\end{equation*}
\end{definition}

\begin{lemma}[Gaussian Mechanism]\label{prelim:gauss_cdp} Let $f : \universe^n \rightarrow \mathbb{R}^d$ be a function with $\ell_2$-sensitivity $\Delta_f$. Then the Gaussian mechanism is algorithm
\begin{equation*}
    \sampler_f(\Datafixed) = f(\Datafixed) + (Z_1, \ldots, Z_d),
\end{equation*}
where $Z_i \sim \Gauss\left(\left(\frac{\Delta_f}{\sqrt{2\rho}}\right)^2 \cdot \mathbb{I}\right)$. Algorithm $\sampler_f$ is $\rho$-zCDP.
\end{lemma}

\begin{lemma}[Exponential Mechanism \cite{McTalwar}]\label{lem:expmech}
Let $L$ be a set of outputs and $g: L \times \mathcal{X}^n \to \mathbb{R}$ be a function that measures the quality of each output on a dataset. Assume that for every $m \in L$, the function $g(m,.)$ has $\ell_1$-sensitivity at most $\Delta$. Then, for all $\eps>0$, $n \in \mathbb{N}$ and for all datasets $\Datafixed \in \mathcal{X}^n$, there exists an $(\eps, 0)$-DP mechanism that, on input $\Datafixed$, outputs an element $m\in L$ such that, for all $a>0$, we have
\begin{equation*}
    \Pr\left[\max_{i \in [L]} g(i,\Datafixed) -  g(m,\Datafixed) \geq 2\Delta \frac{\ln |L| + a}{\eps}\right] \leq e^{-a}. 
\end{equation*}
\end{lemma} 

\subsection{Distributions}
Additionally, we use the Bernoulli, binomial, multinomial, and Poisson distributions as well as total variation distance.

\begin{definition}[Bernoulli Distribution] \label{prelim:bern_def}
The Bernoulli distribution with bias $\biasesfixed \in [0,1]$, denoted $\Ber(\biasesfixed)$, is defined for $\ell \in \{0,1\}$. It has probability mass
\begin{equation*}
    h(\ell) =  
    \begin{cases} 
      \biasesfixed & \text{ if } \ell = 1;\\
      1-\biasesfixed & \text{ if } \ell = 0.
   \end{cases}
\end{equation*}

\end{definition}

\begin{definition}[Binomial Distribution] The binomial distribution with parameters $n$ and $\biasesfixed$, denoted $\Bin(n,
\biasesfixed)$, is defined for all nonnegative integers $\ell$ such that  $\ell \leq n$. It has probability mass
\begin{equation*}
    h(\ell) = \binom{n}{\ell} p^{\ell} (1-p)^{n-\ell}.
\end{equation*}
\end{definition}

\begin{definition}[Multinomial Distribution] The multinomial distribution with parameters $n$ and  $\Biasesfixed \in \Delta^k$, denoted $\Mult(n, \Biasesfixed),$ is defined for all nonnegative integer vectors $\mathbf{\ell} = (\ell_1, \ldots, \ell_k)$ such that $\sum_{i\in[k]} \ell_i \leq n$. It has probability mass
\begin{equation*}
    h(\ell) = 
    \begin{cases}
     \frac{n!}{\ell_1! \cdots \ell_k!} \cdot \biasesfixed_1^{\ell_1} \cdot \ldots \cdot \biasesfixed_k^{\ell_k} & \text{if } \sum_{i\in[k]} \ell_i = n;\\
     0, & \text{otherwise. }
    \end{cases}
\end{equation*}

\end{definition}

\begin{definition}[Poisson Distribution] The Poisson distribution with parameter $\lambda$, denoted $\Po(\lambda)$, is defined for all nonnegative integers $\ell$ with probability mass
\begin{equation*}
    h(\ell) = \frac{\lambda^{\ell} e^{-\ell}}{\ell!}.
\end{equation*}
\end{definition}
We use the following relationship between Poisson and Multinomial distributions.
\begin{lemma}[Poissonization]\label{lem:multtopois}
Fix $h,\lambda > 0$. Let $B \sim \Po(\lambda)$. Let $B_1,\dots,B_{\ell}$ be random variables such that the random variables $B_j \mid \{B = h \}$ are jointly distributed as $\Mult(h,\Biasesfixed)$ where $\Biasesfixed = (\biasesfixed_1,\dots,\biasesfixed_{\ell})$. Then the random variables $B_j$ are mutually independent and distributed as $\Po(\lambda \biasesfixed_j)$.
\end{lemma}

\begin{definition}[Total Variation Distance] \label{def:TV} Let $P$ and $Q$ be discrete probability distributions over some domain $S$. Then
\begin{equation*}
    d_{TV}(P, Q) := \frac{1}{2}\|P - Q\|_1 = \sup_{E\subseteq S} |\Pr_{P}(E) - \Pr_{Q}(E)|.
\end{equation*}
\end{definition}

We use the fact that the total variation distance between two product distributions is subadditive. 
\begin{lemma}[Subadditivity of TV Distance for Product Distributions]\label{lem:subaddTV} Let $P$ and $Q$ be product distributions over some domain $S$. Let $P^1, \dots, P^d$ be the marginal distributions of $P$ and $Q^1, \dots, Q^d$ be the marginal distributions over $Q$. Then
\begin{equation*}
    d_{TV}(P, Q) \leq \sum_{i=1}^d d_{TV}(P^i, Q^i).
\end{equation*}
\end{lemma}
We also use the following lemma regarding total variation distance between a distribution $\distr$ and the distribution obtained by conditioning $\distr$ on a high probability event $E$.
\begin{lemma}[Claim 4, \cite{RaskhodnikovaS06}]\label{lem:TVcond}
Fix $\delta \in (0,1)$. Let $D$ be a distribution and let $E$ be an event that happens with probability $1-\beta$ under the distribution $D$. Let $D|_{E}$ be the distribution of $D$ conditional on event $E$. Then 
$$d_{TV}(D|_{E},D) \leq \frac{\beta}{1-\beta}.$$
\end{lemma}
Finally, we will need the following lemma.
\begin{lemma}[Information Processing Inequality]\label{lem:postTV} Let $A$ and $B$ be random variables over some domain $S$. Let $f$ be a randomized function mapping from $S$ to any codomain $T$. Then
\begin{equation*}
    d_{TV}(f(A), f(B)) \leq d_{TV}(A,B).
\end{equation*}
\end{lemma}

\begin{definition}[KL Divergence] Let $P$ and $Q$ be discrete probability distributions over some domain $S$. Then
\begin{equation*}
    d_{KL}(P, Q) := \frac{1}{2}\sum_{x \in S} P(x)\log\left( \frac{P(x)}{Q(x)}\right)
\end{equation*}
\end{definition}

\begin{claim}\label{clm:bern_acc_eq}
For a Bernoulli distribution $\Ber(p)$, we can simplify the definition of $\alpha$-accuracy (Definition~\ref{def:acc1}) of a sampler \sampler with inputs of size $n$ to require that
\begin{equation*}
d_{TV}(\distroutput{\sampler, \Ber(p)}, \Ber(p)) =
    \Big\lvert \Pr_{\Datarv\sim (\Ber(p))^{\otimes n}}[\sampler(\Datarv) = 1] - p \Big\rvert 
=   \Big|\E_{\Datarv\sim (\Ber(p))^{\otimes n}}
[\indicator_{\sampler(\Datarv) = 1}]-p\Big| 
    \leq \alpha.
\end{equation*}
\end{claim}




\fi


\ifnum\neurips=0

\ifnum\neurips=0
\section{Properties of Samplers}
\else 
\section{Properties of samplers}
\fi
\label{sec:properties}

In this section, we describe three general transformations of samplers that allow us to assume without loss of generality that samplers can take a certain specific form. All three transformations are used in our lower bound proofs.

\ifnum\neurips=1
\subsection{General samplers to Poisson samplers}\label{sec:gentopoiss}
\else 
\subsection{General Samplers to Poisson Samplers}\label{sec:gentopoiss}
\fi

In the first transformation, we show that any private sampling task can be performed by an algorithm that gets a dataset with size distributed as a Poisson random variable instead of getting a dataset of a fixed size. This enables the use of the technique called {\em Poissonization} to break dependencies between quantities that arise in trying to reason about samplers. Recall that $\Po(\lambda)$ denotes a Poisson distribution with mean $\lambda.$
\begin{lemma}\label{lem:poisson} 
If there exists an $(\eps, \delta)$-differentially private sampler $\sampler$ that is $\alpha$-accurate on distribution class~$\class$ for datasets of size $n$, then there exists an $(\eps, \delta)$-differentially private sampler  $\posampler$ that is $(\alpha + e^{-n/6})$-accurate on class $\class$ for datasets of size  distributed as $\Po(2n)$.
\end{lemma}

\begin{proof}
Algorithm~\ref{alg:poisson} is the desired sampler $\posampler$. It is $(\eps, \delta)$-differentially private since $\sampler$ is $(\eps, \delta)$-differentially private. Let $\Datarv$ represent the random variable corresponding to the dataset fed to \sampler.
\begin{algorithm}
    \caption{Sampler \posampler with dataset size $N \sim \Po(2n)$}
    \label{alg:poisson}
    \hspace*{\algorithmicindent} \textbf{Input:} dataset $\Datafixed = (\datafixed_1, \ldots, \datafixed_{N})$,  universe $\universe$, oracle access to $(\eps, \delta)$-DP sampler $\sampler$, parameter~$n$\\
    \hspace*{\algorithmicindent} \textbf{Output:} $i\in \universe$
    \begin{algorithmic}[1] 
            \State Fix an element $\el \in \universe$. 
            \If {$N < n$}  $i=\el$
            \Else  $\text{ } i \gets \sampler(\Datafixed)$\\
            \Return $i$
            \EndIf
    \end{algorithmic}
\end{algorithm}

We use the following tail bound for Poisson random variables \cite{clementpoiss}.
\begin{claim}[\cite{clementpoiss}]\label{lem:poiss_tail}
If $Y\sim \Po(\lambda)$, then $\Pr(Y \leq \lambda - y) \leq e^{-\frac{y^2}{2(\lambda + y)}}$ and $\Pr(Y \geq \lambda + y) \leq e^{-\frac{y^2}{2(\lambda + y)}}$ for all $y>0$.
\end{claim}
Let event $E$ correspond to $N < n$. Let $\overline{E}$ represent the complement of $E$. Then 
\begin{equation}\label{eq:event_tail}
\Pr(E) \leq e^{-\frac{n}{6}}
\end{equation}
by an application of Claim~\ref{lem:poiss_tail}. 
Let $\distr \in \class$ and $\Datarv \sim \distr^{\otimes N}$. Then
\begin{align*}
 d_{TV}(\distroutput{\posampler, \distr}, \distr) 
& =\frac 1 2 \sum_{i \in \universe}|\Pr_{N, \posampler, \Datarv}(\posampler(\Datarv)=i) - \distr(i)| \\
& =\frac 1 2   \sum_{i\in \universe}\left|\Pr_{N, \posampler, \Datarv}(\posampler(\Datarv)=i \land \overline{E}) + \Pr_{N, \posampler, \Datarv}(\posampler(\Datarv)=i \land  E) - \distr(i) \left(\Pr_{N}(\overline{E})+\Pr_{N}(E)\right) \right| \\
& \leq \frac 12 \sum_{i \in \universe}\left(\left|\Pr_{N, \posampler, \Datarv}(\posampler(\Datarv)=i \land  \overline{E}) - \distr(i)\Pr_{N}(\overline{E})\right| +  \Pr_{N, \posampler, \Datarv}(\posampler(\Datarv)=i \land E)+ \Pr_{N}(E)\distr(i)\right)  \\
& = \frac 12 \sum_{i \in \universe}\left|\Pr_{N, \posampler, \Datarv}(\posampler(\Datarv)=i \mid \overline{E})\Pr(\overline{E}) - \distr(i)\Pr_{N}(\overline{E})\right| + 
\frac 12 \cdot(\Pr_{N}(E)+\Pr_{N}(E))\\
& = \frac 12 \sum_{i \in \universe}\Pr_{N}(\overline{E})\cdot\left|\Pr_{\sampler, \Datarv, N \mid \overline{E}}(\sampler(\Datarv)=i \mid \overline{E}) - \distr(i)\right| + \Pr_{N}(E)\\
& \leq \frac 12 \sum_{i \in \universe}\left|\Pr_{\sampler, \Datarv}(\sampler(\Datarv)=i \mid \overline{E}) - \distr(i)\right| + \Pr_{N}(E)\\
& \leq   \alpha +{e^{-n/6}},
\end{align*} 
where the first equality is by the definition of total variation distance, the second equality is because $\Pr(\overline{E}) + \Pr(E) = 1$ and because $\Pr(a) = \Pr(a,E) + \Pr(a,\overline{E})$ for every event $a$, the first inequality is because of the triangle inequality, the third equality is by the product rule
and by marginalizing over the outputs, the fourth equality is because when $N > n$, $\posampler$ sets the output to be $\sampler(\Datarv)$, the second inequality is by the fact that $\Pr_{N}(\overline{E}) \leq 1$, and the final inequality is by (\ref{eq:event_tail}) and the fact that the sampler $\sampler$ is $\alpha$-accurate when it gets any fixed number of samples larger than $n$. 
Hence, $\posampler$ is $(\alpha + e^{-n/6})$-accurate on $\class$. 
\end{proof}

\ifnum\neurips=1
\subsection{Privacy amplification for samplers}\label{sec:privacy-amplification}
\else 
\subsection{Privacy Amplification for Samplers}\label{sec:privacy-amplification}
\fi
Our second general transformation shows how to amplify privacy (that is, decrease privacy parameters) of a sampler by subsampling its input. The transformation does not affect the accuracy. The following lemma quantifies how the privacy parameters and the dataset size are affected by privacy amplification. This result is needed in the proof of the lower bound for $k$-ary distributions, because the main technical lemma in that proof (Lemma~\ref{lem:main-k-ary-lb}) only applies to samplers with small $\eps.$ It is well known that subsampling amplifies differential privacy (see, e.g., \cite{NissimRS07,li2012sampling}).
\begin{lemma}\label{lem:amplify}
Fix $\eps\in(0,1], \delta \in (0,1),$ and $\beta \in (0, 1)$. If there exists an $(\eps,\delta)$-differentially private sampler~$\sampler$ that is $\alpha$-accurate on distribution class \class  for datasets of size distributed as $\Po(n)$ then there exists an $(\eps\beta, \delta \frac{\beta}2)$-differentially private sampler $\sampler_{\beta/2}$ that is $\alpha$-accurate on class \class for datasets of size distributed as $\Po\Big(n\cdot \frac{2}{\beta}\Big)$.
\end{lemma}
 
\begin{proof}
We construct $\sampler_{\beta/2}$ from sampler $\sampler$ as follows: Given a dataset $\Datafixed$, sampler $\sampler_{\beta/2}$ subsamples each record in $\Datafixed$ independently with probability $\beta/2$ to get a new dataset $\Datafixed^*$ and then returns $\sampler(\Datafixed^*)$.

First, we argue that if sampler $\sampler$ is $(\eps,\delta)$-differentially private, then sampler $\sampler_{\beta/2}$  is $(\eps\beta, \delta\frac{\beta}{2})$-differentially private. 
This follows from \cite[Theorem 1]{li2012sampling} which we state as Theorem~\ref{thm:LQS12} in the appendix. By Theorem~\ref{thm:LQS12}, algorithm 
$\sampler_{\beta/2}$ is $(\eps',\delta\cdot\frac \beta 2)$-differentially private with

\begin{align*}
    \eps'=\ln\Big(1 + \frac{\beta}{2}\cdot(e^{\eps} - 1 )\Big) 
    \leq\ln\Big(1 + \frac{\beta}{2}\cdot(2\eps )\Big)
    =\ln(1+\beta\eps)
    \leq \ln(e^{\beta\eps}) = \eps\beta,
\end{align*}
where the inequalities hold because $e^{\eps} -1\leq 2\eps$ for all $\eps \leq 1$ and  $1+\eps\beta \leq e^{\eps\beta}$ for all $\eps\beta$. 

Next, we argue that if sampler $\sampler$ is  $\alpha$-accurate on class \class for datasets of size $\Po(n)$, then sampler $\sampler_{\beta/2}$ is $\alpha$-accurate on class \class for datasets of size $\Po(n\cdot \frac2{\beta})$. 
Suppose $\sampler_{\beta/2}$ is given a sample $\Datarv$ of size $\Po\big(n\cdot \frac{2}{\beta}\big)$ drawn i.i.d.\ from some distribution $P$.
Then the size of $\Datarv^*$, obtained by subsampling each entry of $\Datarv$ with probability $\beta/2$, has distribution $\Po(n)$, and entries of $\Datarv^*$ are i.i.d.\ from $P$. Since the output distributions of $\sampler_{\beta/2}(\Datarv)$ and $\sampler(\Datarv^*)$ are the same, the accuracy guarantee is the same for both algorithms.
\end{proof}

\ifnum\neurips=1
\subsection{General samplers to frequency-count-based samplers}\label{sec:frequency-counts}
\else 
\subsection{General Samplers to Frequency-Count-Based Samplers}\label{sec:frequency-counts}
\fi

Our final transformation shows that algorithms that sample from distribution classes with certain symmetries can be assumed without loss of generality to use only frequency counts of their input dataset in their decisions. Before stating this result (Lemma~\ref{lem:frequency-counts}), we define {\em frequency counts}, {\em frequency-count-based algorithms}, and the type of symmetries relevant for the transformation.

\begin{definition}[Frequency Counts]
Given a dataset $\Datafixed$ and an integer $j\geq 0$, let $F_j(\Datafixed)$ denote the number of elements that occur $j$ times in $\Datafixed$.
The vector $F(\Datafixed)$ of {\em frequency counts} of a dataset $\Datafixed$ of size $n$ is $(F_{0}(\Datafixed), \dots ,F_{n}(
\Datafixed) ).$
\end{definition}

\begin{definition}[Frequency-count-based algorithms]\label{def-freqcountbased}
A sampler is {\em frequency-count-based} if, for every element $i$ in the universe, the probability that the algorithm outputs $i$ when given a dataset $\Datafixed$ only depends on $j$, the number of occurrences of $i$ in $\Datafixed$, and on $F(\Datafixed)$. If $\Datafixed$ contains an element $i \in \universe$ that occurs $j$ times in $\Datafixed$, then let $p_{j,F(\Datafixed)}$ denote the probability that the sampler outputs $i$; otherwise, let $p_{j,F(\datafixed)} = 0$.
\end{definition}
Next, we define the type of distribution classes for which our transformation works.
\begin{definition}\label{def:label-invariant}
A class \class of distributions over a universe $\universe$ is \emph{label-invariant} if for all distributions $\distr \in \class$ and permutations $\pi:\universe \to \universe$, we have $\pi(\distr) \in \class$, where $\pi(\distr)$ is the distribution obtained by applying permutation $\pi$ to the support of $\distr$ (that is, $\Pr_{\pi(\distr)}(\el) = \Pr_{\distr}(\pi^{-1}(\el))$ for all $\el \in \universe$).
\end{definition}

Examples of label-invariant classes include the class of all Bernoulli distibutions and, more generally, the class of all $k$-ary distributions, for any $k.$

\begin{lemma}\label{lem:frequency-counts} Fix  a label-invariant distribution class $\class$. If there exists an  $(\eps, \delta)$-differentially private sampler~$\sampler$ that is  $\alpha$-accurate on \class with a particular distribution on the dataset size, then there exists  an $(\eps, \delta)$-differentially private frequency-count-based sampler $\fpsampler$ that is $\alpha$-accurate on \class with the same distribution on the dataset size.
\end{lemma}

\begin{proof}
Consider an $\alpha$-accurate sampler $\sampler$ for the class $\class$. 
Construct the sampler $\fpsampler$ given in Algorithm~\ref{alg:frequency-counts}.

\begin{algorithm}
    \caption{Sampler \fpsampler}
    \label{alg:frequency-counts}
    \hspace*{\algorithmicindent} \textbf{Input:} dataset $\Datafixed$, universe $\universe$\\
    \hspace*{\algorithmicindent} \textbf{Output:} $i\in \universe$
    \begin{algorithmic}[1] 
            \State Choose a permutation $\pi : \universe \rightarrow \universe$ uniformly at random. \label{step:randperm}
            \State \Return $\pi^{-1}(\sampler(\pi(\Datafixed)))$
    \end{algorithmic}
\end{algorithm}
First, we show that sampler \fpsampler is $\alpha$-accurate for \class. For all $\distr \in \class$ and $\Datafixed \sim \distr$, denote by $\distroutput{\fpsampler(\Datafixed)}$ the distribution of outputs of $\fpsampler(\Datafixed)$. Define $\distroutput{\sampler(\Datafixed)}$ similarly. Then, for a fixed permutation $\pi,$ 
\begin{align}
    d_{TV} (\distroutput{\fpsampler(\Datafixed)}, \distr) 
    & =  d_{TV}\left(\pi(\distroutput{\fpsampler(\Datafixed)}), \pi(\distr)\right)\nonumber\\
    & = d_{TV}\left(\distroutput{\sampler(\pi(\Datafixed))}, \pi(\distr)\right)
     \leq \alpha, \label{calc:alpha_acc}
\end{align}
where the equalities hold by the definition of $\pi$ and $\fpsampler$, and the inequality holds because $\pi(\Datafixed)\sim \pi(\distr)$ and since $\class$ is label-invariant and $\sampler$ is $\alpha$-accurate for $\class$.
 For a fixed $\pi'$, let $\distroutput{\fpsampler(\Datafixed)|\pi'}$ represents the output distribution of \fpsampler conditioned on $\pi'$ being chosen in Step~\ref{step:randperm} of Algorithm~\ref{alg:frequency-counts}.
 
 For a uniformly chosen $\pi$, 
\begin{align*}
    d_{TV}(\distroutput{\fpsampler(\Datafixed)}, \distr) 
    & = d_{TV} \left(\E_\pi[\distroutput{\fpsampler(\Datafixed)|\pi}], \distr \right)\\
    & \leq \E_\pi [d_{TV}(\distroutput{\fpsampler(\Datafixed)|\pi}, \distr)] \quad\quad \text{By the triangle inequality} \\
    & \leq \max_{\pi} \{ d_{TV}(\distroutput{\fpsampler(\Datafixed)|\pi}{}, \distr)\}\\
    & = \max_{\pi} \{ d_{TV}(\distroutput{\sampler(\pi(\Datafixed))}, \pi(\distr)) \}
    \leq \alpha. \quad\quad \text{By (\ref{calc:alpha_acc})}
\end{align*}
Thus, algorithm $\fpsampler$ is $\alpha$-accurate for \class. 

Next, we show that $\fpsampler$ is frequency-count-based by proving that for all permutations $\pi^*$ on the universe and for all $i$ in the universe, $\Pr[\fpsampler(\pi^*(\Datafixed)) = \pi^*(i)] =
\Pr[\fpsampler(\Datafixed) = i] $. Let $\pi_0 = \pi \circ \pi^*$. We can characterise the output distribution of \fpsampler for a fixed $\Datafixed$ as follows
\begin{align*}
    \Pr[\fpsampler(\pi^*(\Datafixed)) = \pi^*(i)]
    & =  \frac{1}{|\universe|!} \sum_{\pi_0 \in [\universe!]} \Pr[\sampler(\pi_0\circ \pi^*(\Datafixed)) = \pi_0\circ \pi^*(i)]\\
    & =  \frac{1}{|\universe|!} \sum_{\pi_0 \circ \pi^*\in [\universe!]} \Pr[\sampler(\pi_0\circ \pi^*(\Datafixed)) = \pi_0\circ \pi^*(i)]\\
     & =  \frac{1}{|\universe|!} \sum_{\pi\in [\universe!]} \Pr[\sampler(\pi(\Datafixed)) = \pi(i)] \\
    & = \frac{1}{|\universe|!} \sum_{\pi\in [\universe!]} \Pr[\pi^{-1}(\sampler(\pi(\Datafixed))) = i]\\
    & = \Pr[\fpsampler(\Datafixed) = i]
\end{align*}
The third equality holds since the permutations $\pi, \pi^*$ are bijections. Thus \fpsampler is frequency-count-based.

Furthermore, the sizes of the input datasets to \fpsampler and \sampler are identical, so if \sampler takes a sample with size distributed according to some distribution $\distr$, then so does the frequency-count-based sampler \fpsampler.

Finally, \fpsampler inherits the privacy of \sampler since it simply permutes its input dataset before passing it to the $(\eps, \delta)$-differentially private sampler \sampler. Its output is a postprocessing of the output it receives from \sampler.
\end{proof}

\ifnum\neurips=1
\section{k-ary discrete distributions}
\else 
\section{\texorpdfstring{$k$}{k}-ary Discrete Distributions}
\fi 
We consider privately sampling from the class of discrete distribution over $[k] := \{1, 2, \ldots, k\}$. We call this class $\class_k$. We prove in this section that the sample complexity of this task is $\Theta(k/\alpha \eps)$, corresponding to Theorems~\ref{thm:kupperb} and~\ref{thm:intro-k-ary-lb}. The proof of Theorem~\ref{thm:intro-k-ary-lb} is split into two cases: Theorem~\ref{thm:bernoulli-lb} deals with the case where $k=2$ and Theorem~\ref{thm:k-ary-lb} deals with the case where $k \geq 3$. We combine these theorems appropriately at the end of Section~\ref{sec:k-ary-lb-final}.
 
\ifnum\neurips=1
\subsection{Optimal private sampler for \texorpdfstring{$k$}{k}-ary distributions}\label{sec:k-ary-ub}
\else 
\subsection{Optimal Private Sampler for \texorpdfstring{$k$}{k}-ary Distributions}\label{sec:k-ary-ub}
\fi

In this section, we prove Theorem~\ref{thm:kupperb}.

\begin{proof}[Proof of Theorem~\ref{thm:kupperb}]
Algorithm~\ref{alg:kary} is the desired $(\eps,0)$-differentially private sampler for $\class_k$. The algorithm computes the empirical distribution, adds Laplace noise to each count in $[k]$, and then projects the result onto $\class_k$ in order to sample from the resulting distribution. The $L_1$ projection onto $\class_k$ is defined as $L_1Proj(\distr) = \argmin_{\distr' \in \class_k} \|\distr - \distr' \|_1$.
\begin{algorithm}
    \caption{Sampler $\karysampler$ for $\class_k$}
    \label{alg:kary}
    \hspace*{\algorithmicindent} \textbf{Input:} dataset $\Datafixed \in [k]^n, \text{ parameter } \eps>0$\\
    \hspace*{\algorithmicindent} \textbf{Output:} $i\in [k]$
    \begin{algorithmic}[1] 
            \For{$j=1$ \text{to} $k$}
            \State $\hat{\distr_j} \gets \frac{1}{n}\sum_{i=1}^n \indicator_{[\datafixed_i=j]}$ \Comment{Compute the empirical distribution}
            \State $Z_j\sim \Lap(2/\eps n)$ \Comment{Sample Laplace noise}
            \State $\hat{\distr}^{noisy}_j \gets \hat{\distr}_j + Z_j$ \Comment{Compute noisy empirical estimate}
            \EndFor
            \State $\tilde{\distr} \gets L_1Proj(\hat{\distr}^{noisy})$ \Comment{Do $L_1$ projection of private empirical estimate to $\class_k$}
            \State $i\sim \tilde{\distr}$ \Comment{Sample from resulting distribution}
            \State \Return $i$
    \end{algorithmic}
\end{algorithm}

First, we argue that Algorithm~\ref{alg:kary} is $\alpha$-accurate. Let $\distr$ be the input distribution represented by a vector of length $k$. As defined in Algorithm~\ref{alg:kary}, let $\hat{\distr}$ be the empirical distribution, $\hat{\distr}^{noisy}$ be the empirical distribution with added Laplace noise, and $\tilde{\distr}$ be the distribution obtained after applying $L_1$ projection (all represented by vectors of length $k$). Then $\E_{\Datarv}[\hat{\distr}] = \distr$, since $\hat{\distr}$ is the empirical distribution of a dataset sampled from~$P$. Let $\distroutput{\karysampler, \distr}$ be the distribution of the sampler's output for dataset $\Datarv \sim \distr^n$. Then $\distroutput{\karysampler, \distr} = \E_{\Datarv, \sampler} [\tilde{\distr}],$ since the output of \karysampler is sampled from $\tilde{\distr}$. We get    
\begin{align}
   d_{TV}(\distroutput{\karysampler, \distr},\distr)
   & = \frac{1}{2}\Big\|\distr-\distroutput{\karysampler, \distr}\Big\|_1 
   = \frac{1}{2}\Big\|\distr - \E_{\Datarv, \karysampler} [\tilde{\distr}]\Big\|_1 
    = \frac{1}{2}\Big\| \E_{\Datarv, \karysampler}[ \hat{\distr} - \tilde{\distr}] \Big\|_1 \nonumber \\
   & \leq \frac{1}{2} \cdot \E_{\Datarv, \karysampler}\left[ \|\hat{\distr} - \tilde{\distr}\|_1 \right], \label{eq:upperbound_jensens}
\end{align}
where we applied Jensen's inequality in the last step of the derivation. Additionally,
\begin{align*}
\|\hat{\distr} - \tilde{\distr}\|_1  
& = \|\tilde{\distr} - \hat{\distr}^{noisy}+ \hat{\distr}^{noisy} - \hat{\distr}\|_1  \\
& \leq \|\hat{\distr} - \hat{\distr}^{noisy}\|_1 + \|\tilde{\distr} - \hat{\distr}^{noisy}\|_1 \quad & \text{By the triangle inequality} \\
& \leq 2\|\hat{\distr} - \hat{\distr}^{noisy}\|_1. \quad & \text{Definition of $L_1$ projection} 
\end{align*}
Substituting this into (\ref{eq:upperbound_jensens}), we get that
\begin{equation*}
   d_{TV}(\distr,\distroutput{\karysampler, \distr}) \leq \frac{1}{2} \cdot\E_{\Datarv, \karysampler}\Big[ 2\|\hat{\distr} - \hat{\distr}^{noisy}\|_1\Big] 
   = \E_{\karysampler}\Big[\sum_{j\in[k]} |Z_j|\Big] 
   =  \frac{2k}{n\eps}, 
\end{equation*}
since the expectation of the absolute value of a random variable distributed according to the Laplace distribution $\Lap(\frac{2}{n \eps})$ is $\frac{2}{n \eps}$. 
We conclude that with $n\geq\frac{2k}{\alpha \eps}$, Algorithm~\ref{alg:kary} is $\alpha$-accurate. 

Next, we show that Algorithm~\ref{alg:kary} is $(\eps, 0)$-differentially private. The sensitivity of a function $f: \mathcal{X}^n \to \mathbb{R}^d$ is defined as $\max_{\Datafixed,\Datafixed' \in \mathcal{X}^n, \|\Datafixed - \Datafixed'\|_0 = 1} \|f(\Datafixed) - f(\Datafixed')\|_1$. 
Recall that $\hat{P}$ is the empirical distribution of dataset $\Datafixed$ (represented by a vector of length $k$).
Changing one element of $\Datafixed$ can change only two components of $\hat{P}$ by $\frac{1}{n}$ each. Hence, the sensitivity of the empirical distribution is $\frac{2}{n}$. Algorithm~\ref{alg:kary} adds Laplace noise scaled to the sensitivity of the empirical distribution to each component of the empirical distribution and then post-processes the output. This is an instantiation of the Laplace Mechanism (proved in \cite{DworkMNS06j} to be $(\eps, 0)$-differentially private) followed by post-processing. Algorithm~\ref{alg:kary} is $(\eps, 0)$-differentially private since differential privacy is preserved under post-processing.
\end{proof}

%

\ifnum\neurips=1
\subsection{The lower bound for the class of Bernoulli distributions}
\else 
\subsection{The Lower Bound for the Class of Bernoulli Distributions}
\fi

We consider the class $\cB$ of Bernoulli distributions with an unknown bias $p.$ For all $p\in[0,1]$, distribution $\Ber(p)\in \cB$ outputs 1 with probability $p$ and 0 with probability $1-p$. Algorithm~\ref{alg:kary} for the special case of $k=2$ shows that $O(\frac 1 {\alpha\eps})$ samples are sufficient for $(\eps,0)$-differentially private $\alpha$-accurate sampling from $\cB$. In this section, we show that this bound is tight, even for $(\eps,\delta)$-differentially private samplers.

\begin{theorem}\label{thm:bernoulli-lb}
If $\eps \in (0,1],\alpha\in(0,1)$, and $\delta \leq \alpha\eps$, then every  $(\eps, \delta)$-differentially private sampler that is $\alpha$-accurate on the class $\cB$ of Bernoulli distributions requires $\Omega(\frac 1{\alpha\eps})$ samples.
\end{theorem}
\begin{proof}
The following lemma captures how differential privacy affects a sampler for Bernoulli distributions.
\begin{lemma} \label{lem:bern_sampler} 
Suppose $\delta \leq \alpha\eps$. If sampler \sampler is  $(\eps, \delta)$-differentially private and $\alpha$-accurate on the class $\cB$ of Bernoulli distributions then, for all $t\in[n],$
\begin{equation*}
    \Pr[\sampler(1^t0^{n-t}) = 1] \leq 2\alpha e^{\eps t}.
\end{equation*}
\end{lemma}

\begin{proof}
Fix $n$ and $t\in[n].$
Since $\sampler$ is $\alpha$-accurate on $\Ber(0),$ we have $\Pr[\sampler(0^n)=1]\leq\alpha.$
We start with the dataset $1^t0^{n-t}$ and replace 1s with 0s one character at a time until we reach $0^n.$
Since $\sampler$ is $(\eps, \delta)$-differentially private, its output distribution does not change dramatically with every replacement. Specifically,
\begin{align*}
    \Pr[\sampler(1^t 0^{n-t}) = 1] &\leq e^\eps \cdot \Pr[\sampler(1^{t-1} 0^{n-t+1}) = 1] + \delta\\
    &\leq  e^\eps(e^\eps \cdot \Pr[\sampler(1^{t-2} 0^{n-t+2}) = 1] + \delta)+\delta \leq\dots\\
     & \leq e^{\eps t}\cdot \Pr[\sampler(0^n)=1] + \delta\cdot \sum_{i = 0}^{t-1} e^{\eps t} 
     = e^{\eps t}\cdot \Pr[\sampler(0^n)=1] +\delta \cdot \frac{e^{\eps t} -1}{e^{\eps} - 1}\\
    & \leq  e^{\eps t}\cdot\alpha + \delta \cdot \frac{e^{\eps t} -1}{e^{\eps} - 1}
    \leq e^{\eps t}\Big(\alpha +\frac \delta \eps\Big)
    \leq 2\alpha e^{\eps t},
\end{align*}
where the last two inequalities hold because $e^\eps-1\leq\eps$ for all $\eps$ and since $\delta\leq\alpha\eps.$
\end{proof}

Consider a sampler \sampler, as described in Theorem~\ref{thm:bernoulli-lb}. By Lemma~\ref{lem:frequency-counts}, since the class $\cB$ is label-invariant, we may assume w.l.o.g.\ that \sampler is frequency-count-based. In particular, the output distribution of \sampler is the same on datasets with the same number of 0s and 1s.

Consider a Bernoulli distribution with $p= 10\alpha$. Let $T$ be a random variable that denotes the number of 1s in $n$ independent draws from $\Ber(p)$.
Then $T$ has binomial distribution $Bin(n,10\alpha).$
By Claim~\ref{clm:bern_acc_eq}  and  $\alpha$-accuracy of \sampler for $\Ber(p)$, we get
\begin{align}
    9\alpha \leq \E_{\Datarv\sim (\Ber(p))^{\otimes n}, \sampler}[\sampler(X)]
    &=\E_{T\sim Bin(n,10\alpha), \sampler}[\sampler(1^{T} 0^{n-T})]\nonumber\\
    &\leq \E_{T\sim Bin(n,10\alpha)}[2\alpha e^{\eps T}]
    =2\alpha (10\alpha(e^\eps-1)+1)^n \label{eq:ber2}\\
    &\leq 2\alpha (20\alpha\eps+1)^n
    \leq 2\alpha e^{20\alpha\eps n},\label{eq:ber3}
\end{align}
where, to get (\ref{eq:ber2}), we used Lemma~\ref{lem:bern_sampler} and then the moment generating function of the binomial distribution; in  (\ref{eq:ber3}), we used that $e^\eps-1\leq 2\eps$ for all $\eps\in (0,1]$ and, finally, that $x+1\leq e^x$ for all $x$ (applied with $x=20\alpha\eps$).
We obtained that $9\alpha \leq 2\alpha \cdot e^{200\alpha \eps n}$, so $n\geq \frac {20}{\ln 4.5} \frac 1 {\alpha\eps}$ samples are required. This completes the proof of Theorem~\ref{thm:bernoulli-lb}.
\end{proof}

\ifnum\neurips=1
\subsection{The lower bound for the class of $k$-ary distributions}\label{sec:kary}
\else 
\subsection{The Lower Bound for the Class of $k$-ary Distributions}\label{sec:kary}
\fi

In this section, we prove Theorem~\ref{thm:intro-k-ary-lb} by providing a lower bound for the universe size at least 3 (Theorem~\ref{thm:k-ary-lb}) and combining it with the previously proved lower bound for the binary case (Theorem~\ref{thm:bernoulli-lb}). The crux of the proof of Theorem~\ref{thm:k-ary-lb} is presented in Section~\ref{sec:k-ary-lb-frequency-count-based}, where we state and prove the lower bound for the special case of Poisson, frequency-count-based samplers with sufficiently small~$\eps$ (that is, a strong privacy guarantee). In Section~\ref{sec:k-ary-lb-final}, we complete the proof of the theorem by generalizing the lower bound from Section~\ref{sec:k-ary-lb-frequency-count-based}
with the help of the transformation lemmas (Lemmas~\ref{lem:poisson},~\ref{lem:amplify}, and~\ref{lem:frequency-counts}) that allow us to convert general samplers to Poisson, frequency-count-based algorithms with small privacy parameter~$\eps.$

\ifnum\neurips=1
\subsubsection{The lower bound for Poisson, frequency-count-based samplers with small $\eps$}\label{sec:k-ary-lb-frequency-count-based}
\else 
\subsubsection{The Lower Bound for Poisson, Frequency-Count-Based Samplers with Small $\eps$}\label{sec:k-ary-lb-frequency-count-based}
\fi


We start by defining the class of distributions used in our lower bound.
\begin{definition}\label{def:ksubclass}
Let \carb denote the the set of distributions with mass $1-60\alpha$ on one {\em special} element $s \in [2k+1]$, and the remaining mass uniform over a size-$k$ subset of $[2k+1] \setminus \{s\}$.
\end{definition} 
Observe that the class $\carb$ is label-invariant (Definition~\ref{def:label-invariant}). This allows us to focus on a simple class of sampling algorithms, called frequency-count-based algorithms (Definition~\ref{def-freqcountbased}), to prove our lower bound.

\begin{lemma}\label{lem:main-k-ary-lb} 
 Fix $k,n \in{\mathbb N}, \alpha \in (0,0.02], \eps \in (0,1/\ln (1/\alpha)],$ and $\delta \in [0, 0.1\alpha\eps/k]$. Let \carb denote the subclass of discrete distributions over the universe $[2k+1]$ specified in the previous paragraph. Let $\alpha^*=60 \alpha$.
 If sampler \sampler is $(\eps, \delta)$-differentially private, frequency-count-based, and $\alpha$-accurate on class \carb with dataset size distributed as $\Po(n)$, then $n > \frac 1 {60}\cdot \frac{k}{\alpha \eps}$. 
\end{lemma}

\begin{proof}
We consider the following distribution $\distr\in \carb.$ Let $\alpha^*=60 \alpha$. Fix a set $S^*\subset[2k]$ of size $k.$ Distribution $\distr$ has mass $\alpha^*/k$ on each element in $S^*$ and mass $1-\alpha^*$ on the {\em special} element $2k+1.$ 
Consider a sampler \sampler satisfying the conditions of Lemma~\ref{lem:main-k-ary-lb}. Let $\distroutput{\sampler, \distr}$ denote the output distribution of $\sampler$ when the dataset size $N \sim \Po(n)$ and the dataset $\Datarv \sim \distr^{\otimes N}$.
Observe that
\begin{eqnarray}\label{eq:main-dist-lb-delta}
d_{TV}(\distroutput{\sampler,\distr}, \distr)
  \geq \Pr_{\substack{N\sim\Po(n) \\ \Datarv\sim  \distr^{\otimes N}}}[\sampler(\Datarv) \notin Supp(\distr)].  
\end{eqnarray}
We will show that when $n\leq \frac k{60\alpha \eps}$  and $\eps$ and $\delta$ are in the specified range, the right-hand side of (\ref{eq:main-dist-lb-delta}) is large. 

We start by deriving a lower bound on $\Pr[\sampler(\Datafixed)\notin Supp(\distr)]$ for a fixed dataset $\Datafixed$ of a fixed size $N$. Since \sampler is frequency-count-based, the probability that it outputs a specific element in $[2k]$ that occurs $0$ times in $\Datafixed$ is $p_{0,F(\Datafixed)}$. Let $F^*_0(\Datafixed)$ denote the number of elements in $[2k]$ that occur $0$ times in $\Datafixed$ (note that the special element $2k+1$ is excluded from this count).
 By definition, $F^*_0(\Datafixed)\leq 2k.$ Consequently,
\begin{equation}\label{eq:notinsupport-delta}
\Pr[\sampler(\Datafixed) \notin Supp(\distr)] = k\cdot p_{0,F(\Datafixed)} \geq \frac{1}{2}\cdot  F^*_0(\Datafixed)\cdot p_{0,F(\Datafixed)}.
\end{equation}

The next claim uses the fact that sampler \sampler is $(\eps,\delta)$-differentially private to show that the probability $p_{j,F(\Datafixed)}$ (that \sampler outputs some specific element in the universe $\universe$  that appears $j$ times in the dataset $\Datafixed$) cannot be much larger than the probability that \sampler outputs a specific element in $\universe$ that does not appear in $\Datafixed$.

\begin{claim}\label{clm:epsdelfing}
For every $(\eps, \delta)$-differentially private sampler and every frequency count $f \in \mathbb{Z}^*$ and index $j \in \universe$,
\begin{equation}\label{eq:grouppriv}
 p_{j,f} 
 \leq e^{\eps j} \left( p_{0,f} + \frac{\delta}{\eps} \right).
\end{equation}

\end{claim}

\begin{proof}
Consider a frequency count $f$ and a dataset $\Datafixed$ with $F(\Datafixed) = f$. Note that (\ref{eq:grouppriv}) is true trivially for all $j$ such that $F_j(\Datafixed)=0$ because, in that case, $p_{j,F(\Datafixed)}$ is set to $0$.

Fix any $j \in \universe$ such that $F_j(\Datafixed) > 0$. Let $a$ be any element in $\universe$ that occurs $j$ times in the dataset $\Datafixed$. Let $b$ be any element in $\universe$ that is not in the support of the distribution $\distr$. Let $\Datafixed|_{a\rightarrow b}$ denote the dataset obtained by replacing every instance of $a$ in the dataset $\Datafixed$ with element $b$. By group privacy \cite{DworkMNS06j},
\begin{equation}\label{eq:group_privacy-delta}
\Pr[\sampler(\Datafixed) = a] \leq e^{\eps j} \Pr[\sampler( \Datafixed|_{a \rightarrow b}) = a ] + \delta \cdot \frac{e^{\eps j} -1}{e^{\eps} - 1}.
\end{equation}
Note that the dataset $\Datafixed|_{a\rightarrow b} $ does not contain element $a$, since we've replaced every instance of it with $b$. Importantly, $F(\Datafixed|_{a \rightarrow b}) = F(\Datafixed)$ because $b$ is outside of the support of the distribution $\distr$ and hence does not occur in $\Datafixed$. Since $\sampler$ is frequency-count-based and $F(\Datafixed) = F(\Datafixed|_{a \rightarrow b})$, we get that $p_{0,F(\Datafixed)} = p_{0,F(\Datafixed|_{a \rightarrow b})}$. 
Substituting this into (\ref{eq:group_privacy-delta}) and using the fact that $e^{\eps}-1\geq \eps$ for all $\eps$, we get that
\begin{equation*}
p_{j,F(\Datafixed)} \leq e^{\eps j}\cdot p_{0,F(\Datafixed)} + \delta \cdot \frac{e^{\eps j} - 1}{e^\eps -1}
\leq e^{\eps j} \left( p_{0,f} + \frac{\delta}{\eps} \right).
\end{equation*}
This completes the proof of Claim~\ref{clm:epsdelfing}.
\end{proof}

For a dataset $\Datafixed$ and $i\in[2k+1]$, let $N_i(\Datafixed)$ denote the number of occurrences of element $i$ in $\Datafixed$.
Next, we give a lower bound on $\Pr[\sampler(\Datafixed) \notin Supp(\distr)]$ in terms of the counts $N_i(\Datafixed)$.

\begin{claim}\label{clm:nonsupport-lb-fixed-delta}
Let $N\in\mathbb{N}$ and $\Datafixed \in [2k+1]^N$ be a fixed dataset. Set 
$Y=\sum_{i \in S^*}  \left[ e^{N_i(\Datafixed) \eps} \right]$.
Then
$$\Pr[\sampler(\Datafixed) \notin Supp(\distr)] \geq\frac{1}{2}
\cdot\frac{\Pr[\sampler(\Datafixed) \in [2k]]}{1+Y/k} -\frac{k\delta}\eps.$$
\end{claim}

\begin{proof}
In the following derivation, we use the fact that that an element $j\in[2k]$ that appears $j$ times in $\Datafixed$ is returned by \sampler with probability $p_{j,F(\Datafixed)}$, then split the elements into those that do not appear in $\Datafixed$ and those that do, next use the fact that all elements from $[2k]$ that appear in $\Datafixed$ must be in $S^*$, then apply Claim~\ref{clm:epsdelfing}, and finally substitute $Y$ for $\sum_{i \in S^*}  \left[ e^{N_i(\Datafixed) \eps} \right]$:
\begin{align*}
\Pr[\sampler(\Datafixed) \in[2k]] 
&=\sum_{i\in[2k]}  p_{N_i(\Datafixed),F(\Datafixed)} 
=F^*_0(\Datafixed)\cdot p_{0,F(\Datafixed)}+\sum_{i\in[2k]\cap\Datafixed}  p_{N_i(\Datafixed),F(\Datafixed)} \\
&\leq F^*_0(\Datafixed)\cdot p_{0,F(\Datafixed)}+\sum_{i\in S^*}  p_{N_i(\Datafixed),F(\Datafixed)}\\
&\leq F^*_0(\Datafixed)\cdot p_{0,F(\Datafixed)}+\sum_{i\in S^*}  p_{0,F(\Datafixed)}\cdot \left(e^{\eps N_i(\Datafixed)} + \frac{\delta}{\eps} \right)\\
&\leq \Big(F^*_0(\Datafixed)+Y\Big)\Big(p_{0,F(\Datafixed)}+  \frac{\delta}{\eps}\Big).
\end{align*} 
We rearrange the terms to get
$$
p_{0,F(\Datafixed)}
\geq \frac{\Pr[\sampler(\Datafixed) \in[2k]]}{F^*_0(\Datafixed)+Y}-\frac \delta \eps.
$$
Substituting this bound on $p_{0,F(\Datafixed)}$ into (\ref{eq:notinsupport-delta}), we obtain
\begin{align*}
   \Pr[\sampler(\Datafixed) \notin Supp(\distr)] 
   &\geq \frac{1}{2} \cdot \frac {F^*_0(\Datafixed)\Pr[\sampler(\Datafixed) \in [2k]]}{F^*_0(\Datafixed) + Y} -\frac 12 \cdot\frac{F^*_0(\Datafixed)\cdot\delta}{\eps} \\
   &= \frac{1}{2} \cdot \frac {\Pr[\sampler(\Datafixed) \in [2k]]}{1 + Y/F^*_0(\Datafixed)} -\frac 12 \cdot\frac{F^*_0(\Datafixed)\cdot\delta}{\eps}\\
   &\geq\frac{1}{2}
\cdot\frac{\Pr[\sampler(\Datafixed) \in [2k]]}{1+Y/k} -\frac{k\delta}\eps,
\end{align*}
where in the last inequality, we used that $k\leq F^*_0(\Datafixed) \leq 2k$.
This holds since the support of $\distr$ excludes $k$ elements from $[2k]$ and since $F^*_0(\Datafixed)$ counts only elements from $[2k]$ that do not appear in $\Datafixed.$
\end{proof}

Finally, we give a lower bound on the right-hand side of (\ref{eq:main-dist-lb-delta}). Assume for the sake of contradiction that $n\leq \frac{k}{\alpha^* \eps}$. 
By Claim~\ref{clm:nonsupport-lb-fixed-delta},
\begin{align}
 \Pr_{\substack{N\sim\Po(n) \\ \Datarv\sim  \distr^{\otimes N}}}[\sampler(\Datarv) \notin Supp(\distr)]
 &\geq \E_{\substack{N\sim\Po(n) \\ \Datarv\sim  \distr^{\otimes N}}}\left[\frac{1}{2}
\cdot \frac{\Pr[\sampler(\Datarv) \in [2k]]}{1+Y/k} -\frac{k\delta}\eps\right]\nonumber\\
& = \frac{1}{2} \cdot \E_{\substack{N\sim\Po(n) \\ \Datarv\sim  \distr^{\otimes N}}}\left[\frac{\Pr[\sampler(\Datarv) \in [2k]]}{1+Y/k}\right] -  \frac{k\delta}{\eps}. \label{eq:mainlbcond}
\end{align}
Next, we analyze the expectation in (\ref{eq:mainlbcond}). Let $E$ be the event that $\frac{Y}{k} \leq e^3$. By the law of total expectation, 
\begin{align}\label{eq:mainlbcondfirst}
     \E_{\substack{N\sim\Po(n) \\ \Datarv\sim  \distr^{\otimes N}}}\left[\frac{\Pr[\sampler(\Datarv) \in [2k]]}{1+Y/k}\right]
     & \geq \E_{\substack{N\sim\Po(n) \\ \Datarv\sim  \distr^{\otimes N}}}\left[\frac{\Pr[\sampler(\Datarv) \in [2k]]}{1+Y/k} \big | E\right] \Pr(E).
\end{align}
In Claims~\ref{claim:eventE} and~\ref{claim:exp-of-regular-output}, we argue that both $\Pr(E)$ and $\E_{\substack{N\sim\Po(n) \\ \Datarv\sim  \distr^{\otimes N}}}\left[\frac{\Pr[\sampler(\Datarv) \in [2k]]}{1+Y/k} \big | E\right]$ are sufficiently large.
\begin{claim}\label{claim:eventE}
Suppose $n\leq \frac k{60\alpha \eps}$. Let $E$ be the event that $\frac{Y}{k} \leq e^3$. Then
    $\Pr(E) \geq 1 - \alpha$.
    \end{claim}
\begin{proof}
Recall that $Y$ was defined as $\sum_{i \in S^*}  \left[ e^{N_i(\Datafixed) \eps} \right]$ for a fixed dataset $\Datafixed.$ Now we consider the case when dataset $\Datarv$ is a random variable. 
If $N\sim\Po(n)$ and $\Datarv \sim \distr^{\otimes N}$ then $N_i(\Datarv) \sim \Po(\frac{\alpha^* n}{k})$ for all $i \in S^*$ and, additionally, the random variables $N_i(\Datarv)$ are mutually independent. When $\Datarv$ is clear from the context, we write $N_i$ instead of $N_i(\Datarv)$. Now we calculate the moments of $\frac{Y}{k}$.  For all $\lambda > 0$, 
\begin{align}
\E_{\substack{N\sim\Po(n) \\ \Datarv \sim  \distr^{\otimes N}}} \left[\left(\frac{Y}{k}\right)^{\lambda} \right] 
 = \E_{\substack{N\sim\Po(n) \\ \Datarv \sim  \distr^{\otimes N}}} \left[\left(\frac{1}{k} \sum_{i \in S^*} e^{N_i(\Datarv) \eps} \right)^{\lambda} \right] 
 = \E_{N_1, \dots, N_k \sim\Po(\frac{\alpha^* n}{k})} \left[ \left(\frac{1}{k} \sum_{i \in S^*} e^{N_i \eps} \right)^{\lambda} \right]. \label{eq:mompoiss}
\end{align}


Finally, we bound the probability of event $E$. Set $c=e^3$ and $\lambda = \ln \frac{1}{\alpha}$. By definition of $E$, 
\begin{align}
    \Pr(\overline{E}) 
    & = \Pr\left(\frac{Y}{k} \geq c\right)\nonumber 
     = \Pr\left(\left(\frac{Y}{k}\right)^\lambda \geq c^{\lambda}\right) 
    \leq \frac 1{c^{\lambda}}\cdot {\E_{\substack{N\sim\Po(n) \\ \Datarv \sim  \distr^{\otimes N}}}
    \left[ \left(\frac{Y}{k}\right)^\lambda \right]} \nonumber \\
    & \leq  \frac 1{c^{\lambda}}\cdot {\E_{N_1, \dots, N_k \sim\Po(\frac{\alpha^* n}{k})} \left[ \left(\frac{1}{k} \sum_{i \in S^*} e^{N_i \eps} \right)^{\lambda} \right]} 
     \leq  \frac 1{c^{\lambda}}\cdot{\E_{N_1 \sim\Po(\frac{\alpha^* n}{k})} \left[ \left( e^{N_1 \eps} \right) ^{\lambda} \right]} \label{eq:moments2} \\
    & = c^{-\lambda}\cdot {e^{\frac{\alpha^* n}{k}\left(e^{\lambda \eps} - 1 \right)}}
    \leq e^{-3\lambda}\cdot {e^{\frac{\left(e^{\lambda \eps} - 1 \right)}{\eps}}}
    \leq e^{-3\lambda}\cdot e^{2\lambda}= e^{-\lambda}
    =e^{-\ln (1/\alpha)}=\alpha, \label{eq:moments3}
\end{align}
where we use $\lambda > 0$ in the second equality, then apply Markov's inequality. To get the inequalities in (\ref{eq:moments2}), we apply (\ref{eq:mompoiss}) and then Claim~\ref{claim:momavg} on the moments of the average of random variables. To get (\ref{eq:moments3}), we use the moment generating function of a Poisson random variable, and then we substitute $c=e^3$ and use the assumption that $n\leq \frac k{60\alpha \eps} =\frac k{\alpha^*\eps}$. The second inequality in (\ref{eq:moments3}) holds because $\lambda = \ln \frac{1}{\alpha}$ and $\eps \in (0,1/\ln \frac{1}{\alpha}]$, so $\lambda \eps \leq 1$ and hence $e^{\lambda \eps} \leq 1 + 2\lambda \eps$.
The final expression is obtained by substituting the value of $\lambda.$
We get that $\Pr(E) \geq 1-\alpha$, completing the proof of Claim~\ref{claim:eventE}.
\end{proof}
\begin{claim}\label{claim:exp-of-regular-output}
$\displaystyle\E_{\substack{N\sim\Po(n) \\ \Datarv\sim  \distr^{\otimes N}}}\left[\frac{\Pr[\sampler(\Datarv) \in [2k]]}{1+Y/k} \big| E \right] \geq 2.3\alpha.$
\end{claim}
\begin{proof}
When event $E$ occurs, $1+\frac{Y}{k} \leq 1+e^3<22$. Then  
\begin{align}\label{eq:conditioning}
    \E_{\substack{N\sim\Po(n) \\ \Datarv\sim  \distr^{\otimes N}}}\left[\frac{\Pr[\sampler(\Datarv) \in [2k]]}{1+Y/k} \big | E \right]
    & > \E_{\substack{N\sim\Po(n) \\ \Datarv\sim  \distr^{\otimes N}}}\left[\frac{\Pr[\sampler(\Datarv) \in [2k]]}{22} \big | E\right] 
     = \frac 1{22}\cdot\E_{\substack{N\sim\Po(n) \\ \Datarv\sim  \distr^{\otimes N}}}\left[\Pr[\sampler(\Datarv) \in [2k]] \mid E \right].
\end{align}
By the product rule,
$$\Pr[\sampler(\Datarv) \in [2k]] \mid E ]
=\frac{\Pr[\sampler(\Datarv) \in [2k]] \wedge E]}{\Pr[E]}
\geq \Pr[\sampler(\Datarv) \in [2k]] \wedge E]
\geq  {\Pr[\sampler(\Datarv) \in [2k]] - \Pr[\overline{E}]}.$$
Substituting this into (\ref{eq:conditioning}) and recalling that $\alpha^*=60\alpha$, we get
\begin{align*}
    \E_{\substack{N\sim\Po(n) \\ \Datarv\sim  \distr^{\otimes N}}}\left[\frac{\Pr[\sampler(\Datarv) \in [2k]]}{1+Y/k} \big | E \right]
    &\geq  \frac 1{22}\cdot\E_{\substack{N\sim\Po(n) \\ \Datarv\sim  \distr^{\otimes N}}}\left[\Pr[\sampler(\Datarv) \in [2k]] - \Pr[\overline{E}] \right]
     \geq \frac 1{22}\cdot \left( \alpha^* - \alpha - \alpha \right) 
     \geq 2.3\alpha,
\end{align*}
since sampler $\sampler$ is $\alpha$-accurate on $\distr$, and  $\distr$ has mass $\alpha^*$ on $[2k]$, and by Claim~\ref{claim:eventE}. 
\end{proof}

Combining (\ref{eq:main-dist-lb-delta}), (\ref{eq:mainlbcond}), and (\ref{eq:mainlbcondfirst}), applying Claims~\ref{claim:eventE} and~\ref{claim:exp-of-regular-output}, and recalling that $\delta\leq 0.1\cdot\alpha\eps/k$, we get
\begin{align*}
    d_{TV}(\distr,\distroutput{\sampler, \distr})
  &\geq \Pr_{\substack{N\sim\Po(n) \\ \Datarv\sim  \distr^{\otimes N}}}[\sampler(\Datarv) \notin Supp(\distr)]
  \geq \frac{1}{2} \cdot \E_{\substack{N\sim\Po(n) \\ \Datarv\sim  \distr^{\otimes N}}}\left[\frac{\Pr[\sampler(\Datarv) \in [2k]]}{1+Y/k}\right]  - \frac{k\delta}{\eps} \\
  &\geq \frac 12\cdot   \E_{\substack{N\sim\Po(n) \\ \Datarv\sim  \distr^{\otimes N}}}\left[\frac{\Pr[\sampler(\Datarv) \in [2k]]}{1+Y/k} \big | E\right] \Pr(E) - 0.1\alpha 
  \geq  \frac{1}{2} \cdot 2.3\alpha\cdot \left(1 - \alpha\right) - 0.1\alpha 
  > \alpha,
\end{align*}
where the last inequality holds since $\alpha\leq 0.02$. This contradicts $\alpha$-accuracy of $\sampler$ on datasets of size $\Po(n)$, where $n\leq \frac{k}{\alpha^* \eps}$, and completes the proof of Lemma~\ref{lem:main-k-ary-lb}.
\end{proof}

Next, we prove our lower bound by removing the assumptions that $\eps$ is small and that the samplers are frequency-count-based. This uses the properties of samplers that we proved in Section~\ref{sec:properties}.

\begin{lemma}\label{lem:k-ary-lb-pois}
For all sufficiently small $\alpha>0$, $k,n \in \mathbb{N}$, $\eps \in (0, 1]$, 
and $\delta \in \big[0, \frac{1}{5000n} \big]$, if there exists an $(\eps,\delta)$-differentially private sampler that is $\alpha$-accurate on the class $\carb$ with dataset size distributed as $\Po(n)$, then $n=\Omega(\frac{k}{ \alpha\eps})$.
\end{lemma}

\begin{proof}
First, we will prove the lemma assuming the range of $\delta$ is $\delta \in \big[0, 0.1 \cdot \frac{\alpha \eps}{k} \big]$. Then we will extend to the claimed range of $\delta$. We prove the lemma for $\delta \in \big[0, 0.1 \cdot \frac{\alpha \eps}{k} \big]$ by applying Lemmas~\ref{lem:amplify} and~\ref{lem:frequency-counts} to generalize the lower bound in Lemma~\ref{lem:main-k-ary-lb} to work for all differentially private samplers and all privacy parameters $\eps\in(0,1]$.

Suppose there exists an $(\eps, \delta)$-differentially private sampler $\sampler$ that is $\alpha$-accurate on the class $\carb$ with dataset size distributed as $\Po(n)$, for some $n \in \mathbb{N}, \eps \in (0,1], \delta \in \big[0, 0.1\cdot\frac{\alpha\eps}k \big]$, and $\alpha \in (0,0.01]$. 

By Lemma~\ref{lem:amplify}, we can amplify the privacy to construct an $(\eps',\delta')$-differentially private sampler~$\sampler'$  that is $\alpha$-accurate for datasets with size distributed as $\Po(4n\ln(1/\alpha'))$,
$\eps'=\frac{\eps}{\ln(1/\alpha)},$ and $\delta'=\frac{\delta}{2\ln(1/\alpha)}$. Then $\eps' \leq \frac{1}{\ln(1/\alpha)}, \delta'\leq 0.01 \frac{\alpha\eps'}k,$ and $\alpha\leq 0.02,$ as required to apply Lemma~\ref{lem:main-k-ary-lb} with privacy parameters $\eps',\delta'$ and  accuracy parameter $\alpha$. By Lemma~\ref{lem:frequency-counts}, we can assume the sampler is frequency-count-based with no changes in the privacy and accuracy parameters. Now, applying Lemma~\ref{lem:main-k-ary-lb} gives 
\begin{align*}
    4n\ln(1/\alpha) \geq \frac{1}{60} \cdot\frac{k}{2\alpha \cdot \frac{\eps}{\ln(1/\alpha)}}\;.
\end{align*}
Therefore, $n \geq \frac{k}{480\alpha\eps}=\Omega(\frac k {\alpha\eps})$.

Next, we extend this argument to all $\delta \in [0, \frac{1}{5000n}]$. Observe that when $n < \frac{k}{480 \alpha \eps} < \frac{1}{5000\delta}$, a direct application of our theorems proves the lower bound. When $n < \frac{1}{5000\delta} < \frac{k}{480 \alpha \eps}$, assume by way of contradiction that there exists an $(\eps, \delta)$-DP sampler that is $\alpha$-accurate on the class of $k$-ary distributions for input datasets with size distributed as $\Po(n)$. 
 
One can find a triple of values $(k',\alpha',\eps')$ such that $2 \leq k'\leq k$, $0.01\geq\alpha'\geq \alpha$ and $1 \geq \eps' \geq \eps$ such that $n\leq  \frac{k'}{480 \alpha'\eps'} <\frac 1 {5000 \delta}$. In particular, $\delta < 0.1\cdot \frac{\alpha'\eps'}{k'}$.  
Next, note that a sampler that achieves accuracy $\alpha < 0.01$ is also a sampler that achieves accuracy $\alpha'$ for all $0.01 \geq \alpha' > \alpha$. Additionally, since the class of discrete distributions over $[k']$ is a subclass of the class of discrete distributions over $[k]$, an $(\eps,\delta)$-differentially private sampler that is $\alpha$-accurate on the class of discrete distributions over $[k]$ when given an input dataset with size distributed as $\Po(n)$ is $(\eps',\delta)$-differentially private and $\alpha'$-accurate on the class of discrete distributions over $[k']$ with sample size distributed as $\Po(n)$.
We can then apply the lower bound obtained for the range $\delta \in \big[0,0.1 \cdot \frac{\alpha' \eps'}{k'}\big]$ to show that no such sampler exists. This proves the theorem for all $\delta \in [0, \frac{1}{5000n}]$.
\end{proof}

\ifnum\neurips=1
\subsubsection{Final lower bound for $k$-ary distributions}\label{sec:k-ary-lb-final}
\else 
\subsubsection{Final Lower Bound for $k$-ary Distributions}\label{sec:k-ary-lb-final}
\fi

In this section, we complete the proof of Theorem~\ref{thm:k-ary-lb}. 

\begin{theorem}\label{thm:k-ary-lb}
For all sufficiently small $\alpha>0$, $k,n \in \mathbb{N}$, $\eps \in (0, 1]$, , 
and $\delta \in \big[0, \frac{1}{5000n}\big]$, if there exists an $(\eps,\delta)$-differentially private sampler that is $\alpha$-accurate on the class $\class_{2k+1}$ of discrete distributions over universe $[2k+1]$ on datasets of size $n$, then $n=\Omega(\frac{k}{ \alpha\eps})$.
\end{theorem}

\begin{proof}
First, observe that a sampler that is $\alpha$-accurate on all distributions in $\class_{2k+1}$, is in particular $\alpha$-accurate on all distributions in the subclass $\carb$. Hence, a lower bound on the number of samples needed to achieve $\alpha$-accuracy for the class $\carb$ also applies to $\class_{2k+1}$. Hence, we can work with \carb for the rest of the proof. We apply Lemma~\ref{lem:poisson} to generalize the lower bound in Lemma~\ref{lem:k-ary-lb-pois} to work for samplers with fixed input dataset sizes. Suppose there exists an $(\eps, \delta)$-differentially private sampler $\sampler$ that is $\alpha$-accurate on the class $\carb$ for datasets of size $n$, for some $n \in \mathbb{N}, \eps \in (0,1], \delta \in \big[0, \frac{1}{5000n} \big]$, and $\alpha \in (0,0.01]$. Then, by Lemma~\ref{lem:poisson}, there exists an $(\eps, \delta)$-differentially private sampler $\posampler$ that is $(\alpha + e^{-n/6})$-accurate on $\class^*_{2k+1}$ when its dataset size is distributed as $\Po(2n)$. We can assume without loss of generality that $n \geq 6\ln(1/\alpha)$, since, if this is not the case, \sampler can ignore extra samples. This gives $e^{-n/6} \leq \alpha$, so sampler \posampler is $2\alpha$-accurate. By Lemma~\ref{lem:k-ary-lb-pois}, we have that $n = \Omega(\frac{k}{\alpha \eps})$, completing the proof. 
\end{proof}

 Now we can combine Theorem~\ref{thm:bernoulli-lb} (for $k=2$) and Theorem~\ref{thm:k-ary-lb} (for $k \geq 3$) to get the lower bound in Theorem~\ref{thm:intro-k-ary-lb} for all $k \geq 2$. 
 
%
%
\ifnum\neurips=1
\section{Product distributions over $\{0,1\}^d$}\label{sec:prod}
\else 
\section{Product Distributions Over $\{0,1\}^d$}\label{sec:prod}
\fi 
In this section, we consider the problem of privately sampling from the class $\cB^{\otimes d}$ of product distributions over $\{0,1\}^d$. 
We present and analyze a $\rho$-zCDP sampler for $\cB^{\otimes d}$ (Theorem~\ref{thm:bernoulli-product-alg}) and then 
apply a standard conversion from $\rho$-zCDP to $(\eps, \delta)$-differential privacy (Theorem~\ref{prelim:relate_dp_cdp}) to prove Theorem~\ref{thm:bernoulli-product-alg-intro}. Then, in Section~\ref{sec:prod-lb}, we prove the matching lower bound stated in Theorem~\ref{thm:bernoulli-product-lb-intro}.
%
%

\subsection{Upper Bound for Products of Bernoulli Distributions}
\begin{theorem}[Upper bound for product distributions]\label{thm:bernoulli-product-alg}
For all $\rho\in(0,1]$, $\alpha\in (0,1)$, and $d$ greater than some sufficiently large constant, there exists a $\rho$-zCDP sampler for the class $\cB^{\otimes d}$ of product Bernoulli distributions that is  $\alpha$-accurate on datasets of size $n=O\Big(\frac {d}{\alpha\sqrt{\rho}} \cdot \Big[\log^{9/4}d + \log^{5/4}\frac{1}{\alpha \sqrt{\rho}}\Big]\Big)$.
\end{theorem}
This theorem implies Theorem~\ref{thm:bernoulli-product-alg-intro}; we prove the implication at the end of this section. 

Our main technical tool is the recursive preconditioning technique of \cite{KLSU19}.
Let $\Biasesfixed = (\biasesfixed_1, \dots, \biasesfixed_d)$ be the unknown attribute biases for the product distribution $\distr \in \cB^{\otimes d}$ from which the data is drawn. 
For some intuition, consider the following natural differentially private algorithm for sampling from a product distribution: First,  privately estimate each of the attribute biases $\biasesfixed_j$ by adding noise to the sample mean; then sample each attribute independently from a Bernoulli with this estimated bias. This approach does not work directly because the $\ell_2$-sensitivity of the vector of sample means is $\sqrt{d}/n$. To accurately estimate tiny biases, we require a large sample size $n$. For instance, in the case where all the attribute biases are roughly $1/d$, the naive algorithm described above would require $n=\Omega(d^{3/2})$ records to be $\alpha$-accurate for a small constant $\alpha$.

%
%
To get around this (in the context of distribution learning),  Kamath et al. \cite{KLSU19} observe that 
when the biases are small and the input is drawn from a product distribution, the number of $1$s in each record is constant---say, at most  $10$---with high probability. Viewing the records as vectors in $\mathbb{R}^d$, we can therefore truncate every record so that its $\ell_2$-norm is at most $\sqrt{10}$ (that is, we leave short vectors alone and shrink longer records) and then average the truncated data entries to obtain a \textit{truncated mean}. We call $\sqrt{10}$ the \emph{truncation ceiling}. Truncation reduces sensitivity, which allows one add less noise---and thus give better attribute bias estimates---while preserving privacy. When the biases are at most $1/d$, the sample complexity for constant accuracy $\alpha$ is reduced to $O(d/\eps)$. (A similar idea works for biases very close to 1. For simplicity, we assume that all attribute biases $\biasesfixed_i$ are between $0$ and $1/2$. See Footnote~\ref{foot:flip}.)


The challenge with this approach is that we don't know biases ahead of time; when coordinates have large bias, setting the truncation ceiling too low leads to high error. Kamath et al. address this by estimating the attribute biases in rounds: in round $j$, attributes with biases close to $2^{-j}$ are estimated reasonably accurately, while smaller biases are passed to the next round where truncation can be applied more aggressively. This process is called \textit{recursive preconditioning}, and it is an important part of our algorithm.




Our algorithm proceeds in two phases. 

\begin{itemize}
    \item \textbf{Bucketing Phase:} This phase implements recursive private preconditoning from \cite{KLSU19} to estimate the attribute biases $\biasesfixed_i$. The main difference is that, for coordinates with large bias, we require less accurate estimates than \cite{KLSU19} and can thus use fewer samples. 
    
    In a bit more detail: The interval $[0,\frac 12]$ is divided into $\lceil \log_2 d\rceil +1$ overlapping sub-intervals that we call \textit{buckets}. The $r^{th}$ bucket corresponds to the interval $[\frac 1 4 \cdot 2^{-r},\   2^{-r}]$. The exception is the smallest bucket, which corresponds to $[0,\frac 1 d]$. 
    
    We proceed in rounds, one per bucket. The bucketing phase uses half of the overall dataset and, for simplicity, those records are split evenly among rounds. Each round thus uses  $m\approx \frac{n}{2\log d}$ records. At round $r$, some coordinates are classified as having attributes in bucket $r$, while others are passed to the next round. With high probability, we maintain the invariants that (a) \textit{only} coordinates with bias at most $2^{-r}$ are passed to round $r$, and (b) \textit{all} coordinates with bias at most $2^{-r-2}$ are passed to round $r+1$ (except for the last round, in which no records are passed on). As a result, coordinates classified in round $r$ have biases in the bucket $[2^{-r-2},2^{-r}]$; records left in the last round have bias at most $1/d$.
    
    For example, the first round corresponds to bucket $[\frac 1 8, \frac 12]$. All coordinates are passed to that round (they have bias at most $\frac 1 2$ by assumption). 
    Using its batch of $m$ records, this round of the algorithm computes the empirical means for all coordinates, adds Gaussian noise about $\frac{\sqrt{d}}{m}$ to each, and releases the list of noisy means. We select $n$ large enough for these noisy estimates to each be within $\frac 1 {16}$ of the true attribute bias with high probability (over the sampling of both the data and the noise).  Attributes with noisy estimates below $3/16$ are passed to round 2, while the rest are assigned to bucket 1. One can check that the invariants are maintained: attributes with bias below 1/8 are passed to round 2; those with bias at least 1/4 are assigned to bucket 1; and those in between may go either way.%
    \footnote{\label{foot:flip} One can also handle biases larger than $1/2$ at this phase. Specifically, the first round of noisy measurements allows us to divide the coordinates into three disjoint sets, each containing only coordinates with biases in $[0,1/4]$, $[1/8,7/8]$, and $[3/4,0]$, respectively. We can work with the coordinates in the first two sets as they are. For coordinates in the third set, we can flip all entries (from 0 to 1 and vice-versa), treat them as if their biases were in $[0,1/4]$, and flip the corresponding output bits.}
    
    At round $r$, we proceed similarly except that we can restrict records to those attributes that were passed to this round and we can truncate records so their $\ell_2$ norm is at most $T_r\approx \sqrt{d} / 2^r$. When the data are from a product distribution and prior rounds were correct, this truncation has essentially no effect on the records but allows us to add less noise. We get noisy means that are within $\pm \frac 1 8 \cdot 2^{-r}$ of the true biases.
    The invariants are maintained if we pass biases with noisy means below $\frac 3 8 \cdot 2^{-r}$ to the next round, and assign the rest to bucket $r$.

    \item \textbf{Sampling Phase:} 
    In the second phase, 
    we use fresh data for the sampling phase to construct new, \textit{unbiased} noisy empirical estimates of the attribute biases. In round $r$ of this phase, we restrict records to the attributes assigned to bucket $r$. We can  truncate the records to have norm $T_r$ (because the biases are at most $2^{-r}$) and add noise as before. This gives us a list of noisy means, which we clip to $[0,1]$ by rounding up negative values and rounding down values above 1. We sample one bit for each attribute independently, according to these clipped noisy means.
    
    For attributes in all buckets except the last, we get noisy means that lie in $[0,1]$ with high probability (because the biases are at least $\frac 1 4 \cdot 2^{-r}$). Since the estimates are unbiased and no clipping occurs, we sample from the correct distribution. For the attributes in the last bucket, we may get negative noisy means. However, the noise is small in these attributes, and we can bound the overall effect on the distribution. Interestingly, almost all the error of our algorithm comes from these low-bias attributes.
    
    

\end{itemize}

We present our sampler $\sampler_{prod}$ for $\cB^{\otimes d}$ in Algorithm~\ref{alg:prod}. 
Let $\Datafixed = (\datafixed_1, \dots, \datafixed_n)$ be a dataset with $n$ records. The truncated mean operation, used in the algorithm, 
is defined as follows:
\begin{align*}
    trunc_B(\datafixed_i)
    &=
   \begin{cases}
    \datafixed_i        & \text{if } \|\datafixed_i\|_2 \leq B; \\
    \frac{B}{\|\datafixed_i\|_2}\; \datafixed_i   & \text{otherwise;}
   \end{cases}\\
    tmean_{B}(\Datafixed) &= \frac{1}{n}\sum_{i=1}^n trunc_B(\datafixed_i).
\end{align*}
Recall that we assume that all of the attribute biases $\biasesfixed_j \in [0,1/2]$.

\begin{algorithm}[ht]
    \caption{Sampler $\sampler_{prod}$ for $\cB^{\otimes d}$}
    \label{alg:prod}
    \hspace*{\algorithmicindent} \textbf{Input:} dataset $\Datafixed \in \{0,1\}^{d\times n}$, privacy parameter  $\rho \in (0,1]$, failure parameter $\beta > 0$. \\
    \hspace*{\algorithmicindent} \textbf{Output:} $b \in \{0,1\}^d$ 
    \begin{algorithmic}[1] 
            \State Set $R \gets \log_2 \frac d{40}, m \gets \frac{n}{2R+1}$.  \Comment{For analysis,  assume 
            $m = \frac{1200d}{\alpha \sqrt{2 \rho}} \log^{5/4}\frac{d R}{\alpha \beta \sqrt{2\rho}}$}
            \State Split $\Datafixed$ into $2R+1$ datasets $\Datafixed^1,\dots,\Datafixed^{2R+1}$ 
            of size $m$ 
            \medskip
            
            \textbf{Bucketing Phase}
            \State Set $S_1 \gets [d]$, $u_1 \gets \frac{1}{2}, \tau_1 \gets \frac{3}{16}$
            \For{$r=1$ to $R$}
            \State $S_{r+1} \gets \emptyset$, $T_r \gets \sqrt{6u_r |S_r| \log\frac{mR}{\beta}}$.
            \State Set $\tilde{\biasesfixed}[S_r] \gets \text{tmean}_{T_r}(\Datafixed^r[S_r]) + \Gauss(0, \frac{T_r^2}{2 \rho m^2} \mathbb{I})$ \label{Step:gauss1} \Comment{Form noisy bias estimates}
            \For{$j \in S_r$} 
            \If{$\tilde{\biasesfixed}[j] < \tau_r$} \Comment{Compare noisy bias estimate to threshold} 
            \State $S_{r+1} \gets S_{r+1} \cup \{j\}$ \Comment{Send $j$ to next round}
            \EndIf
            \EndFor
            \State $S_{R+r} \gets S_r \setminus S_{r+1}, T_{R+r} \gets T_r$
            \State $\tau_{r+1} \gets \frac{\tau_r}{2}, u_{r+1} \gets \frac{u_r}{2}$
            \EndFor
            \medskip
            
            \textbf{Sampling Phase}
            \State $T_{2R+1} \gets  \sqrt{200 \log\frac{m}{\beta}}, S_{2R+1} \gets S_{R+1}, S_{R+1} \gets S_1 \setminus S_2$
            \For{$r=R+1$ to $2R+1$} 
            \State $\tilde{\biasesfixed}[S_r] \gets \text{tmean}_{T_{r}}({\Datafixed}^r[S_r]) + \Gauss(0, \frac{T_{r}^2}{2 \rho m^2}\mathbb{I} )$ \label{Step:gauss2} \Comment{Estimate biases using fresh data and noise}
            \For{$j \in S_r$,}
            \State Set $q[j] \gets [\tilde{\biasesfixed}[j]]_0^1$ \label{step:qdef} \Comment{Clip to lie in the interval [0,1]}
            \State Sample $b_j \sim Ber(q[j])$ \Comment{Sample from estimated marginal distribution}
            \EndFor
            \EndFor
            \State \Return $(b_1, \dots, b_d)$
    \end{algorithmic}
\end{algorithm}
First, we argue that this algorithm is private.
\begin{lemma}\label{lem:prod-privacy}
$\sampler_{prod}$ is $\rho$-zCDP.
\end{lemma}
\begin{proof}
Each input record $\datafixed_i$ is used only in one round in one phase. Assume without loss of generality that this round is in the bucketing phase.  The $\ell_2$-sensitivity of the truncated mean $tmean_{T_r}(\Datarv^r[S_r])$ is $T_r/m$. By the privacy of the Gaussian mechanism (Lemma~\ref{prelim:gauss_cdp}), the step that produces this estimate is $\rho$-zCDP. The remaining steps simply post-process this estimate. Hence, by Lemma~\ref{prelim:postprocess}, Algorithm~\ref{alg:prod} is $\rho$-zCDP.  
\end{proof}

\subsubsection{Overview of Accuracy Analysis}
We analyze the two phases of Algorithm~\ref{alg:prod} separately. Our analysis of the bucketing phase mirrors that of~\cite{KLSU19}. (Their results are not directly applicable to our setting because our algorithm use fewer samples. We therefore give new lemma statements and proofs.)

In Section \ref{sec:techlemmas}, we prove technical lemmas that are used multiple times in the analysis of both phases. In Section~\ref{sec:bucketing}, we show that with high probability, the bucketing phase is successful---that is, we classify all of the attribute biases into the right buckets. This is encapsulated by Lemma~\ref{thm:bucketsuccess}. This corresponds to obtaining good multiplicative approximations of all attribute biases except the smallest ones, for which we obtain good additive approximations.

Next, in Section~\ref{sec:sampling}, we prove a key lemma regarding the success of the sampling phase. 

The intuition behind the analysis of this phase is as follows. Algorithm $\sampler_{prod}$ samples its output from a product distribution. Since the input distribution is in $\cB^{\otimes d}$, each attribute of an input record is sampled from a Bernoulli distribution. By the subadditivity of total variation distance for product distributions, the overall accuracy of $\sampler_{prod}$ can be bounded by showing that $\sampler_{prod}$ samples each attribute independently from a Bernoulli distribution with bias close to the true attribute bias $\biasesfixed_j$. 

The main idea is that the empirical attribute bias has expectation equal to the true attribute bias. Additionally, to preserve privacy, we add zero-mean Gaussian noise. Hence, a noisy empirical estimate of the true attribute bias has mean equal to that attribute bias. If we knew for sure that the noisy empirical estimate for an attribute bias in the sampling phase was always between $0$ and $1$, then the sampler would sample this attribute from exactly the right Bernoulli distribution. 
    
Alas, the noisy empirical estimate of an attribute bias could be less than $0$ or larger than $1$, and we would have to clip it to the interval $[0,1]$ before sampling. This clipping introduces error since we no longer necessarily sample from the right distribution in expectation. We get around this by proving that for attribute biases $\biasesfixed_j$ larger than $\frac{\alpha}{d}$, clipping happens with low probability, and hence the loss in accuracy caused by clipping is small in expectation. However, for attribute biases $\biasesfixed_j$ that are smaller than $\frac{\alpha}{d}$, clipping could occur with high probability. For such attribute biases, we argue that the absolute difference between the clipped noisy empirical mean estimates and the true attribute biases is small enough (at most $\frac{\alpha}{d}$) with high probability. This argument is described in Lemma~\ref{lem:mainacc}.

Finally, we prove the main upper bound theorem in Section~\ref{sec:prod-ub-main} by putting everything together.

\subsubsection{Analysis of Good Events}\label{sec:techlemmas}
In the accuracy analysis, we assume that 
$m \geq    \frac{1200 \ d}{\alpha\sqrt{2 \rho}} \cdot \log^{5/4}\paren{\frac{d R}{\alpha \beta \sqrt{2\rho}}}$ and $d$ is sufficiently large (that is, greater than some positive constant).
In this section, we define three good events $G_1,G_2,$ and $G_3$ that, respectively, represent that empirical means are close to the attribute biases in all $2R+1$ datasets into which Algorithm~\ref{alg:prod} subdivides its input, that truncation does not occur in any round (assuming successful bucketing for that round), and that the added Gaussian noise is sufficiently small. We show that each of these events fails to occur only with small probability.

First, we prove that the empirical means are close to the attribute biases with high probability. Define the empirical mean $\hat{\biasesfixed}_r[j] := \frac{1}{m} \sum_{i=1}^m \datafixed_i^{r}[j]$.

\begin{lemma}\label{lem:empiricalest}
Let $G_1$ be the event that for all rounds $r \in [2R+1]$, the following conditions hold:
\begin{enumerate}
    \item For all $j \in [d]$, if $\frac{\alpha}{d} \leq \biasesfixed_j \leq \frac{1}{2}$ then $|\hat{\biasesfixed}_r[j] - \biasesfixed_j| \leq \frac{\biasesfixed_j}{16}.$
    \item For all $j \in [d]$, if $\biasesfixed_j < \frac{\alpha}{d}$ then  $|\hat{\biasesfixed}_r[j] - \biasesfixed_j| \leq \frac{\alpha}{4d}$.
\end{enumerate}
Then $\Pr\Big[\overline{G_1}\Big]\leq 2\beta,$ where the probability is over the randomness of the input data and the coins of $\sampler_{prod}$.
\end{lemma}

\begin{proof}
Fix $r\in [2R+1]$ and $j \in [d]$. Note that $\E[\hat{\biasesfixed}_r[j]] = \biasesfixed_j$ for all $r\in [2R+1]$. 

We prove Item 1 of the lemma by a case analysis on $\biasesfixed_j$. First, when  $\biasesfixed_j \geq \frac{\alpha}{4d}$, we use the multiplicative Chernoff bound from Claim~\ref{claim:cher_bounds} for $\gamma \in (0,1)$:
\begin{align*}
  \Pr\Big[\hat{\biasesfixed}_{r}[j] > \biasesfixed_j\Big(1 + \frac{\alpha}{4d\biasesfixed_j}\Big)\Big] 
  \leq \exp\Big({-\frac{\alpha^2 \biasesfixed_j m}{48 d^2(\biasesfixed_j)^2}}\Big) 
  = \exp\Big({-\frac{\alpha^2m}{48 d^2\biasesfixed_j}} \Big)
  \leq \exp\Big({-\frac{\alpha}{12 d}\frac{1000 d}{\alpha}} \log\frac{dR}\beta \Big)
  \leq \frac{\beta}{4d(R+1)},
\end{align*}
where in the third inequality we used that $\biasesfixed_j \leq \frac{\alpha}{d}$ and substituted in a lower bound for~$m$. 

Secondly, when $\biasesfixed_j < \frac{\alpha}{4d}$, we use the multiplicative Chernoff bound for all $\gamma > 0$ from Claim~\ref{claim:cher_bounds}:
\begin{align*}
  \Pr\Big[\hat{\biasesfixed}_r[j] > \biasesfixed_j\Big(1 + \frac{\alpha}{4d\biasesfixed_j}\Big)\Big] 
  &\leq \exp\Big(- \frac{\alpha^2\biasesfixed_j m}{16d^2\biasesfixed_j^2(2 + \frac{\alpha}{4d\biasesfixed_j})}\Big)
  \leq \exp\Big(- \frac{\alpha^2 m}{12d^2\biasesfixed_j(\frac{\alpha}{d})}\Big)\\
  &= \exp\Big(- \frac{\alpha m}{12d\biasesfixed_j}\Big)
  \leq  \exp\Big({-\frac{\alpha}{12 d} \frac{1000 d\log(dR/\beta)}{\alpha}} \Big)
  \leq \frac{\beta}{4d(R+1)},
\end{align*}
where in the first inequality we used that since  $\frac{\alpha}{4d\biasesfixed_j} > 1$,  $\frac{\alpha}{4d\biasesfixed_j} + 2 \leq 3\frac{\alpha}{4d\biasesfixed_j}$, and in the third inequality we substituted a lower bound for the value of $m$ and upper bounded $\biasesfixed_j$ by $\alpha$.  

Similar inequalities hold for the lower tails of $\hat{\biasesfixed}_r[j]$. Taking a union bound over all $j \in [d]$ such that $\biasesfixed_j \leq \frac{\alpha}{d}$ completes the the proof of Item~1 in Lemma~\ref{lem:empiricalest}.

Next, assume that $\frac{\alpha}{d} \leq \biasesfixed_j \leq \frac{1}{2}$. By the Chernoff bound from Claim~\ref{claim:cher_bounds} for $\gamma \in (0,1)$, 
\begin{align*}
    \Pr\left[\hat{\biasesfixed}_r[j] - \biasesfixed_j \geq \frac{\biasesfixed_j}{16}\right] 
    = \Pr\Big[\hat{\biasesfixed}_r[j] 
    \geq \biasesfixed_j \Big(1 + \frac{1}{16}\Big)\Big] 
    \leq \exp\left(-\frac{\biasesfixed_j m}{3\cdot 16}\right)
    \leq \frac{\beta}{2d(R+1)},
\end{align*}
where the final inequality holds since $\biasesfixed_j m \geq \frac{\alpha}{d}1000\frac{d}{\alpha}\log\frac{dR}{\beta}$. A similar bound holds for the lower tail of $\hat{\biasesfixed}_r[j]$. Taking a union bound over all $j \in [d]$, and all $r \in [2R+1]$ gives the result.
\end{proof}

Next, we argue that truncation is unlikely in any round (given successful bucketing). Recall that $u_r=1/2^r$ for all $r\in[R]$ (see Algorithm~\ref{alg:prod}). For all $r \in [R]$, let $u_{R+r}=u_r$. Let $u_{2R+1}=20/d$. A version of the following lemma is stated and proved in~\cite{KLSU19} (for us, the smallest upper bound of a bucket is $u_{2R+1}=20/d$ instead of $1/d,$ but the truncation ceiling $T_{2R+1}$ is also larger than in~\cite{KLSU19} to balance this out.)

\begin{lemma}[\cite{KLSU19}, Claims 5.10 and 5.18]\label{lem:trunc}
Let $G_2$ be the following event that, for every round $r\in [2R+1]$, the following holds: 
if $\biasesfixed_j \leq u_r$ for all $j \in S_r$, then 
for every $i \in [m]$, $$\|\datafixed_i^r [S_r] \|_2 \leq  T_r,$$ 
that is, no rows are truncated in the calculation of $\text{tmean}_{T_r}(\Datafixed^r[S_r])$ in Steps~\ref{Step:gauss1} or~\ref{Step:gauss2} of Algorithm~\ref{alg:prod}. 
Then $\Pr\Big[\overline{G_2}\Big]\leq 3\beta,$ where the probability is over the randomness of the input data and the coins of $\sampler_{prod}$.
\end{lemma}

Finally, we prove that the amount of noise added in any round is unlikely to be large.
For all $r \in [2R+1]$, let $Z_r$ be a $d$-dimensional random vector representing the noise added in round $r$ as in Steps~\ref{Step:gauss1} and~\ref{Step:gauss2} of Algorithm~\ref{alg:prod}. For attributes $j \in [d]$ to which no noise is added in round $r$, the coordinate $Z_r[j] = 0$. The remaining $Z_r[j]$ are drawn from independent zero-mean Gaussians with standard deviation specified in Steps~\ref{Step:gauss1} and~\ref{Step:gauss2} of Algorithm~\ref{alg:prod}. 

\begin{lemma}\label{lem:Gaussnoise}
Let $G_3$ be the event that for all rounds $r \in [2R+1]$, for all $j \in S_r$,
\begin{equation*}
    |Z_r[j]| \leq \frac{\alpha u_r}{100}.
\end{equation*}
Then $\Pr\Big[\overline{G_3}\Big]\leq \beta,$ where the probability is over the randomness of the input data and the coins of $\sampler_{prod}$. 
\end{lemma}
\begin{proof}
For rounds $r \in [2R]$, the standard deviation of univariate Gaussian noise $Z_r[j]$ added in round $r$ is $\sigma_r = \sqrt{\frac{3u_r |S_r|}{\rho m^2} \log \frac{mR}{\beta}}$. Set $t = \sqrt{2 \ln \frac{6dR}{\beta}}$. By Lemma~\ref{lem:gaussconc} on the concentration of a zero-mean Gaussian random variable along with a union bound,
$$ \Pr(\max_{j \in S_r} |Z_r[j]| \geq t \sigma_r) \leq \sum_{j \in S_r} \Pr(|Z_r[j]| \geq t \sigma_r) \leq  \sum_{j \in S_r} 2e^{-t^2/2} \leq \frac{\beta}{2R+1}.$$
Since $m\geq\frac{600 d}{\alpha \sqrt{\rho}} \log^{5/4}(\frac{dR}{\alpha \beta \sqrt{\rho}})$ and , and because $u_r \geq \frac{40}{d}$, for all $r \in [2R]$,
\begin{align*}
    t \sigma_r & =
    \sqrt{\frac{6 u_r |S_r| \log \frac{mR}{\beta} \ln \frac{6dR}{\beta} }{\rho m^2}} 
     \leq \alpha \sqrt{\frac{6 u_r |S_r| \log \frac{dR}{\alpha \beta \sqrt{\rho}} \ln \frac{6dR}{\beta} }{36000 d^2 \log^{10/4}(\frac{dR}{\alpha \beta \sqrt{\rho}})}} 
    \leq \alpha \sqrt{\frac{6 u_r d}{3600 d^2 \log^{1/4}(\frac{dR}{\beta})}} \leq \frac{\alpha u_r}{100}  ,
\end{align*}
where the first inequality is because $\log\frac{mR}{\beta} / m$ is a decreasing function for $m$ and hence we can upper bound the expression by using a lower bound of $m$. We also use the fact that the term $\log \frac{mR}{\beta} \leq 10 \log \frac{dR}{\alpha \beta \sqrt{\rho}})$. The second inequality follows by cancelling out some log terms and using the fact that $|S_r| \leq d$, and the last inequality follows because $\frac{1}{d} \leq \frac{u_r}{40}$, $\beta \leq 1$, and because $d$ is sufficiently large.
For $r = 2R+1$, the standard deviation $\sigma_r = \sqrt{\frac{100}{\rho m^2}\log\frac{m}{\beta}}$, so with the same value of $t$, we get the same result. Taking a union bound over all $r \in [2R+1]$ gives the result. 
\end{proof}
The following corollary summarizes our analysis of good events and follows from Lemmas~\ref{lem:empiricalest}%
--\ref{lem:Gaussnoise} by a union bound.
\begin{corollary}\label{cor:good-event-prod}
Let $G$ be the event $G_1\cap G_2\cap G_3.$ Then $\Pr[\overline{G}] \leq 6 \beta$, where the probability is over the randomness of the input data and the coins of $\sampler_{prod}$.
\end{corollary}
\subsubsection{Success of the Bucketing Phase}\label{sec:bucketing}
In this section, we argue that if the good event $G$ occurs, then the bucketing phase succeeds.

\begin{lemma}\label{thm:bucketsuccess} 
Let \bucket be the event that that the bucketing procedure is successful, namely, for all rounds $r \in [R]\cup\{2R+1\}$ and for all coordinates $j \in [d]$, the following statements hold:
\begin{enumerate}
    \item If $r \in [R]$ and $\biasesfixed_j \in S_{R+r}$, then $u_r / 4 \leq \biasesfixed_j \leq u_{r}$.
    \item If $p_j \in S_{2R+1}$ then
    $\biasesfixed_j \leq u_{2R+1}$.
\end{enumerate}
If the good event $G$ defined in Corollary~\ref{cor:good-event-prod} occurs then $\bucket$ occurs.
\end{lemma}
\begin{proof}
Assume that $B$ occurs. We prove this lemma by induction on $r$. Recall that $S_{R+r}=S_{r} \setminus S_{r+1}$ for all $r \in [R]$. To prove Item 1, we show that, for all rounds $r \in [R]$, if $j \in S_r$ then $\biasesfixed_j \leq u_r$, and if $j \not \in S_r$ then $\biasesfixed_j \geq u_r / 2$.
For the first round (the base case of the induction), since $u_1 = 1/2$, and since by assumption $\biasesfixed_j \leq 1/2$ for all $j \in [d]$, we have that $\biasesfixed_j \leq u_1$. Additionally, since $S_1 = [d]$, it vacuously holds that $p_j \geq u_{1}/2$ for all $j \not \in S_{1}$. Next, fix any $r \in [R-1]$. The inductive hypothesis is that for round $r$, if $j \in S_{r}$ then $\biasesfixed_j \leq u_{r}$ and if $j \not \in S_{r}$ then $\biasesfixed_j \geq u_r / 2 = u_{r+1}$.

We prove that this statement holds for round $r+1$. 
For all $j \in S_r$, let $\Tilde{\biasesfixed}_r[j]$ be the noisy empirical estimate obtained for coordinate $j$ in Step~\ref{Step:gauss1} of Algorithm~\ref{alg:prod} (in round $r$).
By Item 1 of the definition of event~$G_1$, for all $j \in S_r$ with $p_j > \frac{\alpha}{d}$,
$$|\hat{\biasesfixed}_r[j] - \biasesfixed_j| \leq \frac{\biasesfixed_j}{16} \leq \frac{u_r}{16},$$
where the second inequality is by the induction hypothesis.
Similarly, by Item 2 of the definition of event $G_1$, for all $j \in S_r$ with $p_j \leq \frac{\alpha}{d}$,
$$|\hat{p}_r[j] - p_j| \leq \frac{\alpha}{4d} \leq  \frac{u_r}{160},$$
where the second inequality holds since $u_r \geq \frac{40}{d}$ for all $r \in [R]$ and $\alpha \leq 1$. 

By the inductive hypothesis, $\biasesfixed_j \leq u_r$ for all $j \in S_r$. Hence, by the definition of event $G_2$, no truncation occurs in round $r$. Also, $|Z_r[j]| \leq \frac{u_r}{100}$, by the definition of event $G_3$, since $\alpha \leq 1$. Hence, for all $j \in S_r$,
\begin{align*}
    |\hat{\biasesfixed}_r[j] - \Tilde{\biasesfixed}_r[j]| \leq \frac{u_r}{100}.
\end{align*}

By the triangle inequality, we get that for all $j \in S_r$,
\begin{equation}\label{eq:triangle-bucketing}
|\Tilde{\biasesfixed}_r[j] - \biasesfixed_j| \leq  |\hat{\biasesfixed}_r[j] - \Tilde{\biasesfixed}_r[j]| + |\hat{\biasesfixed}_r[j] - \biasesfixed_j| \leq + \frac{u_r}{100} + \frac{u_r}{16}  \leq \frac{u_r}{8}.
\end{equation}
Fix any $j \in S_r$. Recall that $\tau_r = \frac{3u_{r+1}}{4}$. If $\Tilde{\biasesfixed}_r[j] \leq \tau_r$ then, by (\ref{eq:triangle-bucketing}), $\biasesfixed_j \leq \tau_r + \frac{u_r}{8} = \frac{3u_{r+1}}{4} + \frac{u_{r+1}}{4} = u_{r+1}$. 
Similarly, if $\Tilde{\biasesfixed}_r[j] \geq \tau_r$, then $\biasesfixed_j \geq \frac{3u_{r+1}}{4} - \frac{u_{r+1}}{4} = \frac{u_{r+1}}{2}$.

This completes the inductive step and proves that at the beginning of round $r+1$, we have that $p_j \leq u_{r+1}$ for all  $j \in S_{r+1}$, and $p_j \geq \frac{u_{r+1}}{2}$ for all $j \not \in S_{r+1}$. Item 2 follows from an extension of the same argument.
\end{proof}

\subsubsection{Success of the Sampling Phase}\label{sec:sampling}
\begin{lemma}[Success of sampling phase]\label{lem:mainacc}
For all $j \in [d]$, for $q[j]$ defined as in Step~\ref{step:qdef} of algorithm $\sampler_{prod}$, when $\sampler_{prod}$ is run with failure probability parameter $\beta \in(0, \frac{1}{12}]$ and target accuracy $\alpha \in(0, 1]$, 
\begin{enumerate}
    \item if  $\frac{\alpha}{d} < \biasesfixed_j \leq \frac 1 2$, then
    $\displaystyle| \mathbb{E}[ q[j] - \biasesfixed_j ] | \leq 12\beta;$
    \item if  $\biasesfixed_j \leq \frac{\alpha}{d}$, then
    $\displaystyle | \mathbb{E}[ q[j] - \biasesfixed_j ] | \leq \frac{\alpha}{2d} + 6\beta;$
\end{enumerate}
where the expectations are taken over the randomness of the data and the noise.
\end{lemma}
\begin{proof}
We start by proving Item 1. 

Fix any $j \in [d]$ with $\frac{\alpha}{d} \leq \biasesfixed_j \leq \frac{1}{2}$. First, we argue that if event $G$ occurs, then no noisy empirical means are clipped in the sampling phase. By construction, $(S_{R+1},\dots,S_{2R+1})$ is a partition of $[d]$. For all $j\in[d]$, let $r(j)$ denote the round $r\in\{R+1,\dots,2R+1\}$ such that $j\in S_r.$
%
%
Now suppose $G$ occurred. 
By the triangle inequality and since $G$ implies $G_1,G_2,G_3,$ and $\bucket,$
\begin{align}\label{eq:clip}
    |\biasesfixed_j- \tilde{\biasesfixed}_{r(j)}[j]| \leq |\biasesfixed_j- \hat{\biasesfixed}_{r(j)}[j]| + |\hat{\biasesfixed}_{r(j)}[j]- \tilde{\biasesfixed}_{r(j)}[j]| \leq \frac{\biasesfixed_j}{16} + |Z_{r(j)}[j]| \leq \frac{\biasesfixed_j}{16} + \frac{\alpha u_{r(j)}}{100} \leq \frac{\biasesfixed_j}{3},
\end{align}
where the second inequality is by the definition of event $G_1$, the fact that $G$ implies $\bucket$, and the definition of event $G_2$ and the third inequality is by the definition of event $G_3$. The final inequality uses the fact that event $B$ occurs; if $\biasesfixed_j > \frac{5}{d}$ then $u_{r(j)} \leq 4\biasesfixed_j$, and otherwise $u_{r(j)} = \frac{20}{d}$ and hence $\frac{\alpha u_{r(j)}}{100} = \frac{20 \alpha }{100 d} \leq \frac{\biasesfixed_j}{5}$.

If $G$ occurs, by (\ref{eq:clip}) and since $0 < \biasesfixed_j \leq \frac{1}{2}$, we have
\begin{align*}
    0 < \frac{2\biasesfixed_j}{3} \leq  \tilde{\biasesfixed}_{r(j)}[j]
    \leq \frac{4\biasesfixed_j}{3} < 1,
\end{align*}
and thus $\tilde{\biasesfixed}_{r(j)}[j]$ does not get clipped.

Next, by the law of total expectation,
\begin{align*}
   \mathbb{E}[q[j] -  \biasesfixed_j]
   & =  \mathbb{E}[q[j] -  \biasesfixed_j \mid G] \cdot \Pr[G]
    +  \mathbb{E}[q[j] -  \biasesfixed_j \mid \overline{G}]\cdot \Pr[\overline{G}] \\
   & \leq \mathbb{E}[q[j]  \mid G] - \biasesfixed_j  + 6\beta  \leq \frac{\mathbb{E}[\hat{\biasesfixed}_{r(j)}[j] + Z_{r(j)}[j]] }{\Pr[G]}- \biasesfixed_j  + 6\beta \nonumber\\
   & \leq \frac{\biasesfixed_j }{1-6\beta}- \biasesfixed_j   + 6\beta \leq \frac{1}{2}\left(\frac{1}{1-6\beta}- 1\right)  + 6\beta \leq 12 \beta, 
\end{align*}
where the first inequality holds by Corollary~\ref{cor:good-event-prod} and the fact that $\mathbb{E}[q[j] -  \biasesfixed_j] \leq 1$. The second inequality uses the fact that when $G$ occurs, there is no clipping and truncation, and $\E[A \mid E] \leq \E[A]/\Pr[E]$ for all random variables $A$ and events $E$. The third inequality is by the fact that $\mathbb{E}[\hat{\biasesfixed}_{r(j)}[j]] = \biasesfixed_j$ and  $\E[Z_{r(j)}[j]] = 0$, and by Corollary~\ref{cor:good-event-prod}. The last inequality holds because $\beta \leq \frac {1}{12}$ and $\biasesfixed_j \leq \frac{1}{2}$ by assumption. 

Analogously, $\mathbb{E}[\biasesfixed_j - q[j]] \leq 12 \beta$, which completes the proof of Item 1.

Next, we prove Item 2. Recall the event $G$ defined in Corollary~\ref{cor:good-event-prod} and the event $B$ defined in Lemma~\ref{thm:bucketsuccess}. 
Fix a coordinate $j \in [d]$ with $\biasesfixed_j \leq \frac{\alpha}{d}$.
By Lemma~\ref{thm:bucketsuccess}, the law of total expectation, and the fact that $|\E[ q[j] - \biasesfixed_j \mid \overline{\bucket}]| \leq 1$, we get
\begin{align}
 |\E[ q[j] - \biasesfixed_j ]|
  \leq |\E[ q[j] - \biasesfixed_j \mid G]|\cdot\Pr[G] + |\E[ q[j] - \biasesfixed_j \mid \overline{B}]|\cdot\Pr[\overline{G}] 
 \leq |\E[ q[j] - \biasesfixed_j \mid G]| + 6\beta. \label{eq:interim3}
\end{align}
Now, we show that $|\E[ q[j] - \biasesfixed_j \mid G| \leq  \frac{\alpha}{2d}$. Conditioned on event $G$, using Lemma~\ref{thm:bucketsuccess}, event $B$ occurs, and the output bits for all coordinates $j$ with $\biasesfixed_j \leq \frac{\alpha}{d}$ are sampled in round $2R+1$.  Conditioned on event $G$, using the fact that event $B$ occurs, and using the definition of event $G_2$ on truncation of empirical estimates, 
$$|[\tilde{\biasesfixed}_{2R+1}[j]]_0^1 - \biasesfixed_j | \leq |\tilde{\biasesfixed}_{2R+1}[j] - \biasesfixed_j| = |Z_{2R+1}[j]| \leq \frac{\alpha u_{2R+1}}{100} \leq \frac{\alpha}{2d},$$
where the second to last inequality is by the definition of event $G_3$, and the last inequality is since $u_{2R+1} = \frac{20}{d}$.
Thus, $|\E[ q[j] - \biasesfixed_j \mid G]|\leq \frac{\alpha}{2d}$. Combining this with (\ref{eq:interim3}) proves Item 2 of Lemma~\ref{lem:mainacc}. 
\end{proof}

\subsubsection{Proof of Main Theorem}\label{sec:prod-ub-main}
Finally, we use Lemma~\ref{lem:mainacc} to prove the theorem.
\begin{proof}[Proof of Theorem~\ref{thm:bernoulli-product-alg}.]
Fix $\rho \in(0, 1], \alpha \in (0,1), \beta = \frac{\alpha}{12d}$, and $R = \log_2 (d/40)$. Fix the sample size 
$$n = \frac{1200(2R+1)d}{\alpha \sqrt{2 \rho}}  \log^{5/4}  \frac{dR}{\alpha \beta \sqrt{2\rho}} = \tilde{O} \Big(\frac{d}{\alpha\sqrt{\rho}} \Big)$$ for this setting of $\beta$ and $R$. 

First, by Lemma~\ref{lem:prod-privacy}, we have that $\sampler_{prod}$ is $\rho$-zCDP.

Next, we reason about accuracy. Let $\distroutput{\sampler_{prod}, \distr^{\otimes d}}$ be the distribution of the output of the sampler $\sampler_{prod}$ with randomness coming from the data and coins of the algorithm. Observe that  $\distroutput{\sampler_{prod}, \distr^{\otimes d}}$ is a product distribution and that the marginal bias of each coordinate $j \in [d]$ is $\mathbb{E}[q[j]]$. Let the marginal distributions of  $\distroutput{\sampler_{prod}, \distr^{\otimes d}}$  be $Q_1, \dots, Q_d$. By the subadditivity of total variation distance between two product distributions (Lemma~\ref{lem:subaddTV}),
\begin{align*}
    d_{TV}(\distroutput{\sampler_{prod}, \distr}, \distr^{\otimes d}) 
    & \leq \sum_{i=1}^d d_{TV}(Q_i, \distr) 
     = \sum_{i=1}^d |\mathbb{E}[ q[j]-  \biasesfixed_j]| \\
    & = \sum_{i: \biasesfixed_i > \frac{\alpha}{d}}  |\mathbb{E}[ q[j] -  \biasesfixed_j]| + \sum_{i: \biasesfixed_i \leq \frac{\alpha}{d}}  |\mathbb{E}[ q[j] -  \biasesfixed_j]| \\
    & \leq  \sum_{i: \biasesfixed_i > \frac{\alpha}{d}} 12\beta  + \sum_{i: \biasesfixed_i \leq \frac{\alpha}{d}} \left(\frac{\alpha}{2d} + 6\beta\right)
    \leq \alpha,
\end{align*}
where we got the first equality by substituting the expression for the total variation distance between two Bernoulli distributions, the second inequality is by Lemma~\ref{lem:mainacc} (since $\beta = \frac{\alpha}{12d} \leq \frac{1}{12}$, this lemma is applicable), and the final inequality holds because $\beta = \frac{\alpha}{12d}$.
\end{proof}

Finally, we complete this section by proving Theorem~\ref{thm:bernoulli-product-alg-intro} from the introduction. 
\begin{proof}[Proof of Theorem~\ref{thm:bernoulli-product-alg-intro}]
Set $\rho = \frac{\eps^2}{16\log(1/\delta)}$. By Lemma~\ref{prelim:relate_dp_cdp}, for all $\delta\in(0,1/2]$, algorithm $\sampler_{prod}$ is $(\eps, \delta)$-differentially private. Substituting this value of $\rho$ into Theorem~\ref{thm:bernoulli-product-alg}, we get that the sampler $\sampler_{prod}$ is $\alpha$-accurate for input datasets of size
$$n = O\left( \frac{d}{\alpha \eps} \cdot \sqrt{\log\frac{1}{\delta}}  \Bigg( \log^{9/4} d + \log^{5/4} \frac{\log\frac{1}{\delta}}{\alpha \eps} \Bigg) \right).$$ For $\log(1/\delta) = polylog(n)$, we get that $\sampler_{prod}$ is $\alpha$-accurate for input datasets of size 
$n = \tilde{O}(\frac{d}{\alpha \eps})$. This proves the theorem, since $\delta \leq 1/2$.
\end{proof}



\newpage 
\subsection{Lower bound for Products of Bernoulli Distributions}\label{sec:prod-lb}
In this section, we prove the following theorem.
\begin{theorem}\label{thm:product-bern-lb}
For all $d,n\in\N$, $\eps \in (0,\frac{1}{2}]$, $\delta \in [0,\frac 1 {5000n}]$, and sufficiently small $\alpha>0$, every $(\eps,\delta)$-differentially private sampler that is $\alpha$-accurate on the class of products of $d$ Bernoulli distributions needs datasets of size $n=\Omega(\frac{d}{\alpha \eps})$.
\end{theorem}

The theorem is proved via a reduction from the problem of private sampling from discrete $k$-ary distributions.
Recall that in Section~\ref{sec:kary}, we proved a lower bound for this problem
by considering the class $\carb$ (see Definition~\ref{def:ksubclass}) and giving a lower bound of $\Omega(k/\alpha \eps)$ on the sample complexity of $\eps$-differentially private $\alpha$-accurate sampling from this class. For each distribution in $\carb$, we consider the corresponding product distribution, defined next.

\begin{definition}\label{def:corrprod}
For every distribution $\distr \in \carb$ with a special element $s$,  its {\em corresponding product distribution} is the distribution $\corr \in \cB^{\otimes 2k}$, where the bias of each coordinate $j\in[2k] $ of $\corr$ is $P(j)$ for $j<s$ and $P(j+1)$ for $j\in [s,2k]$ (note that the special element is eliminated). 
\end{definition}

The initial idea that inspired our final reduction is to create a private sampler for $\carb$ from a private sampler for product distributions as follows. Replace each non-special entry of the input dataset sampled from $\distr \in \carb$ with its one-hot encoding (write element $j \in [2k]$ as a binary vector of size $2k$, with a $1$ in position $j$ and $0$'s in all other positions) to create a dataset that looks like it was drawn from the corresponding product distribution $\corr$. Then apply the private product distribution sampler to this modified dataset. Finally, map the resulting sample in $\{0,1\}^{2k}$ back to $[2k+1]$ by uniformly sampling the index of a random nonzero element and returning the special element if the sample is the all-zeros vector. 

This initial idea does not work, because one-hot encoding produces vector entries in $\{0,1\}^{2k}$ that always have a single $1$, whereas samples from a product distribution could have multiple $1$'s, or no $1$'s. However, the reduction presented here has a similar structure. The starting point of our reduction is Lemma~\ref{lem:k-ary-lb-pois}, which applies to Poisson samplers. It states that for every $(\eps,\delta)$-differentially private sampler that is $\alpha$-accurate on the class $\carb$ with dataset size distributed as $\Po(n)$, we have $n=\Omega(\frac{k}{ \alpha\eps})$.
The Poisson sampler $\sampler_{red}$ for $\carb$ in our reduction is described in Algorithm~\ref{alg:redsampler}. It first privately identifies the ``special element'' (using a simple instantiation of the exponential mechanism described in Algorithm~\ref{alg:special_element}) and then proceeds in three steps: {\em dataset transformation, sampling,} and {\em universe transformation.}
{\em Dataset transformation} transforms the input dataset drawn from a distribution $\distr$ in $\carb$ to a dataset distributed as independent samples
from the corresponding product distribution $\corr$ in $\cB^{\otimes 2k}$. The {\em sampling step} runs a private sampler $\sampler_{prod}$ for class $\cB^ {\otimes 2k}$ with privacy parameters $\eps$ and $\delta$, accuracy parameter $\alpha$, and the dataset from the previous step to obtain one sample $\prodoutput$ in $\{0,1\}^{2k}$. 
The {\em universe transformation} transforms the output $\prodoutput$ from the previous step into a single sample in $[2k+1]$. This sample is the final output of Algorithm~\ref{alg:redsampler}. We will show that if the original input table was drawn i.i.d.\ from a distribution $\carb$, then the final output is a sample from a nearby distribution.

\begin{algorithm}
        \caption{Sampler $\sampler_{red}$ for $\carb$ with dataset size $N \sim \Po(n)$}
    \label{alg:redsampler}
    \hspace*{\algorithmicindent} \textbf{Input:} Dataset $\Datafixed = (\datafixed_1,\dots,\datafixed_N)\in [2k+1]^N$, privacy parameters $\eps, \delta>0$, parameters $k,n \in \N$, accuracy parameter $\alpha \in (0,1)$, black box access to private sampler $\sampler_{prod}$ for $\cB^ {\otimes 2k}$ \\
    \hspace*{\algorithmicindent} \textbf{Output:} $y' \in [2k+1]$
    \begin{algorithmic}[1] 
           \State Let $C$ be a sufficiently large constant. Draw $L \sim \Bin(N, \frac{2C \log k}{\alpha \eps n})$. Set $R = N-L$. \label{step:datasplitval}
           \State If $L < C\frac{\log k}{\alpha \eps}$, output $2k+1$ and break. \label{step:largespec}
           
           \State Partition $\Datafixed$ into datasets $\Datafixed^L$ and $\Datafixed^R$, where $\Datafixed^L$ has $L$ records.
           \label{step:partition}
           \State Run Algorithm~\ref{alg:special_element} on $\Datafixed^L$, value $k$, and privacy parameter $\frac{\eps}{2}$ to obtain a candidate special element $\hat{s}$. \label{step:specialel}
           \State Run Algorithm~\ref{alg:datatrans} to obtain a dataset $\Datafixedy \gets \reduction^{\rightarrow}(\Datafixed^R, k, n, 60\alpha, \hat{s})$. \label{step:datasettrans}
           \State $\prodoutput \gets \sampler_{prod}(\frac{\eps}{4},\frac{\delta}{2}, \frac{\alpha}{25}, \Datafixedy)$ \label{step:privateprodsamp}
           
           \State Run Algorithm~\ref{alg:univtrans} to obtain $z \gets \reduction^{\leftarrow}(\prodoutput, \hat{s})$. 
           \State Output $z$.
    \end{algorithmic}
\end{algorithm}

\begin{algorithm}
        \caption{Special element picker $SE$}
    \label{alg:special_element}
    \hspace*{\algorithmicindent} \textbf{Input:} Dataset $\Datafixed = (\datafixed_1,\dots,\datafixed_n)$, privacy parameter $\eps$, value $k$ \\
    \hspace*{\algorithmicindent} \textbf{Output:} $\hat{s} \in [2k+1]$
    \begin{algorithmic}[1] 
           \State For $j \in [2k+1]$, define the utility function $u(\Datafixed, j) = \frac{1}{n}\sum_{i=1}^n \indicator[\datafixed_i = j]$.
           \State Run the exponential mechanism \cite{McTalwar} with privacy parameter $\eps$ and utility function $u(\cdot, \cdot)$ on dataset $\Datafixed$ to choose an index $\hat{s} \in [2k+1]$. (Specifically, each index $j$ is selected with probability $\propto e^{\frac{\eps\cdot u(\Datafixed,j)}{2}}$).
           \State Output $\hat{s}$.
    \end{algorithmic}
\end{algorithm}

We now explain the dataset and universe transformations.


\paragraph{Dataset Transformation (Algorithm~\ref{alg:datatrans})} 
The goal in this step is to produce a dataset that is distributed as an i.i.d.\ sample from a product distribution. Such distributions are characterized by two properties: 
\begin{enumerate}
    \item The marginal distributions over the columns are independent.
    \item The bits of each column are  mutually independent Bernoulli random variables.
\end{enumerate}
%
The dataset transformation uses this characterization to convert a Poisson number of samples from a $k$-ary distribution to a dataset distributed as independent samples
from the corresponding product distribution. 
Algorithm~\ref{alg:datatrans} is inspired by an elegant coupling by Klenke and Mattner \cite[Theorem 1.f]{klenke2010stochastic} that shows that a Poisson random variable stochastically dominates a binomial random variable with sufficiently small probability parameter.  Algorithm~\ref{alg:datatrans} relies on the fact that if the size of the input dataset follows a Poisson distribution, then  the number of appearances $\hist[i]$ of each element $i$ are mutually independent Poisson random variables.  

The transformed dataset is built column by column. We first compute the histogram of the input dataset. Then, for each element $i \in [2k+1] \setminus \{s\}$, we sample random variables $A_1,\dots,A_{n/2}$ from the multinomial distribution with $\hist[i]$ trials and probability vector $(2/n,\dots,2/n)$. We truncate each of these random variables to have value at most $1$. Finally, we independently set each random variable to $0$ with a carefully chosen probability. Column $i$ of the transformed dataset is then set to $(A_1,\dots,A_{n/2})$. (There is no column $s$.)

We argue that after thresholding, $A_1,\dots,A_{n/2}$  are mutually independent Bernoulli random variables. Independently setting each of them to $0$ with some probability adjusts their biases. Additionally, as discussed earlier, the element counts in the input dataset are independent, and hence the marginal distributions of the columns are independent as well. This guarantees that the rows in the dataset look like independently sampled entries from the corresponding product distribution. We formalize these ideas in Lemma~\ref{lem:datatrans}.

\begin{algorithm}
    \caption{Dataset transformation algorithm $\reduction^{\rightarrow}$}
    \label{alg:datatrans}
    \hspace*{\algorithmicindent} \textbf{Input:} Dataset $\Datafixed = (\datafixed_1, \ldots, \datafixed_{N})\in [2k+1]^N$; parameters $k,n \in\N, \alpha^*\in(0,1)$, special element $s$ \\
    \hspace*{\algorithmicindent} \textbf{Output:} Dataset $\Datafixedy = (\datafixedy_1, \ldots, \datafixedy_{n/2})\in (\{0,1\}^{2k})^{n/2}$
    \begin{algorithmic}[1] 
            \State Initialize $\Datafixedy$ as an empty matrix \Comment{We will build the $\frac{n}{2} \times 2k$ matrix column by column}
            \State $\hist \gets histogram(\Datafixed)$ \label{step:hist}
            \For{$i \in [2k+1] \setminus \{s\}$}
                \State Sample $(A_1,\dots,A_{n/2}) \sim Mult(\hist[i], (2/n,\dots,2/n))$. \label{step:mult}
                \For{$j \in [n/2]$}
                \State Set $A_j\gets \min (A_j, 1)$ \label{step:threshold}
                \State Set $A_j\gets 0$ with probability $1-\frac p{1-e^{-2p}}$, where $p=\frac{\alpha^*}k.$ \Comment{Reduce the probability that $A_j$ is $1$}
                \label{step:flip}
                \EndFor
                \State Set column $i$ of $\Datafixedy$ to $(A_1,\dots,A_{n/2})$.
            \EndFor
            \State Output $\Datafixedy$
    \end{algorithmic}
\end{algorithm}

\paragraph{Universe Transformation (Algorithm~\ref{alg:univtrans})} Our universe transformation procedure takes an element in the universe $\{0,1\}^{2k}$ and converts it to an element in the universe $[2k+1]$. Once we apply our product distribution sampler on the dataset obtained in Algorithm~\ref{alg:datatrans}, we use this procedure to transform the resulting sample so that the final output looks like it was sampled from a distribution in $\carb$. 
For any distribution $\distr$ in $\carb$, we call the $k$ non-special elements with nonzero probability mass the \textbf{participating elements}. When the input coordinates are all $0$s, we output the special element. Otherwise, we output a random non-special element corresponding to one of the coordinates with value 1. In Lemma~\ref{lem:easyunivtrans}, we prove that when this procedure is run with a sample from the right product distribution, it outputs a sample from a distribution~$Q$ such that $d_{TV}(\distr,Q)\leq (\alpha^*)^2$.
%

In Lemma~\ref{lem:specialel}, we show that the special element is chosen with high probability. Finally, we combine Lemmas~\ref{lem:datatrans}--\ref{lem:specialel} with the lower bound in Lemma~\ref{lem:k-ary-lb-pois} to prove Theorem~\ref{thm:product-bern-lb}.
\begin{algorithm}
        \caption{Universe transformation algorithm $\reduction^{\leftarrow}$}
    \label{alg:univtrans}
    \hspace*{\algorithmicindent} \textbf{Input:} A sample $\prodoutput \in \{0,1\}^{2k}$, special element $s$ \\
    \hspace*{\algorithmicindent} \textbf{Output:} $z \in [2k+1]$
    \begin{algorithmic}[1] 
           \If{$\prodoutput = (0,\dots,0)$}
           \State Set $z \gets s$
           \Else 
           \State Choose $z$ uniformly at random from $\{j\in[2k+1]: [j<s \wedge \prodoutput_j=1]\vee [j>s \wedge \prodoutput_{j-1}=1]\}$ \hspace{5cm}
\Comment{Shift by 1 to skip the special element} 
           \EndIf 
           \State Output $z$
    \end{algorithmic}
\end{algorithm}

\begin{lemma}\label{lem:datatrans}
Fix a sufficiently small $\alpha > 0$, set $\alpha^* = 60 \alpha$, fix a distribution $\distr \in \carb$, and let $\corr$ be the corresponding product distribution. Let $\eps \in (0,1], \delta \geq 0$, and $n\geq 2$. Let $\reduction^{\rightarrow}$ be Algorithm~\ref{alg:datatrans}. Let dataset $\Datarv \sim \distr^{\otimes N}$ where $N\sim \Po(n)$. Then the random variable $\reduction^{\rightarrow}(\Datarv,k,n,\alpha^*,s)$ 
has distribution $\corr^{\otimes (n/2)}$. 
\end{lemma}
\begin{proof}
Let $\Datarv \sim \distr^{\otimes N}$ be the input dataset of size $N \sim \Po(n)$. Let $p=\frac{\alpha^*}{k}$. 

The output $\Datarvy$ of Algorithm~\ref{alg:datatrans} is a dataset of size $n/2$. We want to show that it is independently sampled from the corresponding product distribution $\corr$. This is equivalent to the following two conditions: 
\begin{enumerate}
    \item For all columns  $i \in[2k]$, the bits of column $i$ in $\Datarvy$ are  mutually independent Bernoullis with bias $p$ if $i$ is a participating element of $\distr$, and bias $0$ if $i$ is non-participating.
    \item The columns in $\Datarvy$ are  mutually independent.
\end{enumerate}

We start by proving the first condition. 
Fix a participating element $i$. Since the input dataset has size distributed as $\Po(n)$, by Lemma~\ref{lem:multtopois}, the count of element $i$ (defined as $\hist[i]$ in Step~\ref{step:hist} of Algorithm~\ref{alg:datatrans}) is distributed as $\Po(pn)$. In Step~\ref{step:mult} of Algorithm~\ref{alg:datatrans}, we draw $A_1,\dots,A_{n/2}$ from $\Mult(\hist[i],(2/n,\dots,2/n))$. By Lemma~\ref{lem:multtopois}, they are mutually independent Poisson random variables with mean $2p$. Thus, the random variables $A_j$ after thresholding at $1$ in Step~\ref{step:threshold} of Algorithm~\ref{alg:datatrans} are mutually independent Bernoulli random variables that are $0$ with probability $e^{-2p}$. After setting each random variable $A_j$ independently to $0$ with some probability in Step~\ref{step:flip} (note that the probability $1-p-e^{-2p} \geq 1-p- (1-2p + 4p^2) \geq  0$ and so this step is well defined), we get that the random variables $A_j$ are mutually independent Bernoullis that are $0$ with probability $$e^{-2p} + (1-e^{-2p}) \frac{1-p-e^{-2p}}{1-e^{-2p}} = 1-p.$$
On the other hand, for non-participating elements $i$, the $A_j$s are sampled from a multinomial distribution with $0$ trials and hence will be identically $0$. They are not changed in the remaining steps of Algorithm~\ref{alg:datatrans}. Hence, they can be thought of as mutually independent Bernoullis with bias $0$.

Next, we prove the second condition. Lemma~\ref{lem:multtopois} applied to the element counts $\hist[i]$ implies that they are mutually independent. Column $i$ in $\Datarvy$ only depends on the element count $\hist[i]$, and hence, the columns of $\Datarvy$ are also mutually independent. 

Thus, the output dataset $\Datarvy$ is correctly distributed.
\end{proof}

\begin{lemma}\label{lem:easyunivtrans}
Let $\alpha \in (0, 1/60]$ and $\alpha^* = 60\alpha$. Fix any distribution $\distr \in \carb$. Let $s$ be its special element and $B$ be a sample from the corresponding product distribution $\corr$. Let $\reduction^{\leftarrow}$ be Algorithm~\ref{alg:univtrans}. Then $d_{TV}(\reduction^{\leftarrow}(B,s),\distr) \leq (\alpha^*)^2$. 
\end{lemma}
\begin{proof}
Let $q$ be the probability that $B \neq (0,\dots,0)$. Since $B$ is sampled from the corresponding product distribution $\corr$, $1-q=(1-\frac{\alpha^*}{k})^k$, which is between $1-\alpha^*$ and $ 1-\alpha^* + (\alpha^*)^2$ for all $\alpha^* \in (0,1)$ and $k\geq 1$. The lower bound of $1-\alpha^*$ follows from Bernoulli's inequality (Lemma~\ref{lem:bernoulli}) with $a = -\alpha^*/k$ and $r=k$. The upper bound holds because $(1-\frac{\alpha^*}{k})^k \leq e^{-\alpha^*} \leq 1-\alpha^* + (\alpha^*)^2$, where the last inequality is obtained by truncating the Taylor expansion of $e^{-\alpha^*}$.

Algorithm $\reduction^{\leftarrow}(B,s)$ returns the special element $s$ with probability $1-q$. Thus, the probability that $\reduction^{\leftarrow}(B,s)$ returns the special element $s$ is within $(\alpha^*)^2$ of the probability mass of $s$ in $\distr$, which is $1-\alpha^*$.

Conditioned on not returning the special element $s$, the reduction $\reduction^{\leftarrow}(B,s)$ will output an element uniformly from $supp(\distr) \setminus \{s\}$ by symmetry. (This is because all coordinates in $[2k+1] \setminus \{s\}$ have bias either 0 or $\frac{\alpha^*}{k}$ under $\corr$.) Specifically, the reduction places mass $q/k$ on all elements in the support of $P$ other than $s$.  Therefore,
\begin{eqnarray*}
d_{TV}(\reduction^{\leftarrow}(B,s),\distr) 
&=& \frac{1}{2} 
\sum_{i \in [2k+1]} |\Pr(\reduction^{\leftarrow}(B,s) = i) - \distr(i)| \\
&=&\frac 1 2  \left[\left| (1-q) - (1-\alpha^*) \right| + k \cdot \left| \frac q  k - \frac{\alpha^*} k \right|\right] \\
&=& |(1-q) - (1-\alpha^*)|\leq (\alpha^*)^2 \, . \hspace*{2cm} \hfill  \qedhere
\end{eqnarray*}
\end{proof}

Next, we prove that Algorithm~\ref{alg:special_element} returns the special element with high probability.

\begin{lemma}\label{lem:specialel}
Let $\eps \in (0,1]$. Let $\alpha \in (0,1]$ be sufficiently small, and $k \in \mathbb{N}$ be sufficiently large. Let $\alpha^* = 60\alpha$. Fix a distribution $\distr \in \carb$, and let $s$ be its special element. Let $\Datarv$ be a dataset drawn from $\distr^{\otimes n}$, where $n=\frac{C\log k}{\alpha \eps}$ for a sufficiently large constant $C$. Then $\Pr[SE(\Datarv,\eps,k) = s]\geq 1-\frac{\alpha}{2}$, where $SE$ is described in Algorithm~\ref{alg:special_element}.
\end{lemma}

\begin{proof}
The probability $\distr(s)$ of the special element $s$ under the distribution $\distr$ is $1-\alpha^*$. Let $W$ be the number of times $s$ occurs in $\Datarv$. Since $W$ is a sum of $n$ Bernoulli random variables with bias $(1-\alpha^*)$, the variance of $\frac{W}{n}$ is $\frac{\alpha^*(1-\alpha^*)}{n}\leq \frac{\alpha^*}n$. Let $E$ be the event that $\left|\frac{W}{n} - \distr(s)\right| > \sqrt{\alpha}$. By Chebyshev's inequality,
$$\Pr[E] \leq\frac{\text{Var}\big[\frac W n\big]} \alpha \leq\frac{\alpha^*}{n\alpha} \leq \frac{\alpha}{4}$$
for $n=\frac{C\log k}{\alpha \eps}$. Note that for sufficiently small $\alpha$, event $\overline{E}$ implies that $\frac{W}{n} > 0.6$, which implies that in $\Datarv$, the number of occurrences of the special element is at least $0.2n$ more than the number of occurrences of any other element. 

The sensitivity of the score function used by the exponential mechanism in Algorithm~\ref{alg:special_element} is $\frac{1}{n}$. Hence, by Lemma~\ref{lem:expmech} on the accuracy of the exponential mechanism, for any fixed dataset, with probability greater than or equal to $1-e^{-t}$, the exponential mechanism outputs an element that occurs $an$ times, where $a> \max_{j \in [2k+1]} \left(\frac{1}{n}\sum_{i=1}^n \indicator[\datafixed_i = j]\right) - \frac{2}{n \eps}(\log(2k+1) + t)$. Using $n=\frac{C\log k}{\alpha \eps}$, and setting $t = \log \frac{4}{\alpha}$, we get that $\frac{2}{n \eps}(\log(2k+1) + t) \leq \alpha(2 + \log \frac{4}{\alpha}) < \frac{1}{5}$ and $e^{-t} = \frac{\alpha}{4}$ for sufficiently large $k$ and sufficiently small $\alpha$. Hence, conditioned on event $\overline{E}$, since the number of occurrences of the special element is at least $0.2n$ more than the number of occurrences of any other element, and the guarantee of the exponential mechanism gives that with high probability, the element output occurs $an$ times where $an > \max_{j \in [2k+1]} \left(\frac{1}{n}\sum_{i=1}^n \indicator[\datafixed_i = S]\right) - 0.2n$, we get that the special element is output with probability at least $1 - \frac{\alpha}{4}$. Using the fact that $\Pr(\overline{E}) > 1 - \frac{\alpha}{4}$, we get that with probability at least $1 - \frac{\alpha}{2}$, $SE(\Datarv,\eps,k)=s$.
\end{proof}

\begin{lemma} \label{lem:privred}
The sampler defined in Algorithm~\ref{alg:redsampler} is $(\eps, \delta)$-differentially private. 
\end{lemma}

\begin{proof}
Fix $k,n\in\mathbb{N}$ and  $\eps, \delta > 0$. Fix two neighboring datasets $\Datafixed, \Datafixed' \in [2k+1]^n$. Let $i^* \in [n]$ be the index on which they differ. Let $\datafixed_{i^*}=u$ and $\datafixed'_{i^*}=\ell$. By the privacy of the exponential mechanism, the subroutine call to the Special element picker (Algorithm~\ref{alg:special_element}) in Step~\ref{step:specialel} of Algorithm~\ref{alg:redsampler} is $\eps/2$-differentially private. The output datasets obtained in the data transformation step (Step~\ref{step:datasettrans} of Algorithm~\ref{alg:redsampler}) when run on $\Datafixed$ and $\Datafixed'$ can differ only in the $u^{th}$ and $\ell^{th}$ columns. Without loss of generality, consider the $u^{th}$ column, and let $\hist(u)$ and $\hist(u)+1$ be the number of times $u$ occurs in $\Datafixed$ and $\Datafixed'$. Then, in Step~\ref{step:mult} of Algorithm~\ref{alg:datatrans} run on datasets $\Datafixed$ and $\Datafixed'$, the random variables $(A_1,\dots,A_{n/2})$ are sampled from $\Mult(\hist(u),(1/n,\dots,1/n))$ and $\Mult(\hist(u)+1,(1/n,\dots,1/n)) = \Mult(\hist(u),(1/n,\dots,1/n)) + \Mult(1,(1/n,\dots,1/n))$, respectively, before being post-processed in the same way in Step~\ref{step:flip} of Algorithm~\ref{alg:datatrans}. Thus, for fixed coins of the algorithm, the $u^{th}$ columns of the datasets produced in Step~\ref{step:datasettrans} of Algorithm~\ref{alg:redsampler} when run on $\Datafixed$ and $\Datafixed'$ respectively, differ by at most $1$ entry. Similarly, the $\ell^{th}$ columns obtained in these two runs also differ in at most a single entry. Hence, for fixed coins of the algorithm, the datasets $\Datafixedy$ obtained in these two runs differ in at most $2$ elements. 
By group privacy and the law of total probability, the subroutine call to sampler $\sampler_{prod}$ in Step~\ref{step:privateprodsamp} of Algorithm~\ref{alg:redsampler} is then $(\eps/2, \delta)$-differentially private. The last step is just post-processing. Thus, using basic composition, the sampler described in Algorithm~\ref{alg:redsampler} is $(\eps, \delta)$-differentially private. 
\end{proof}

Armed with the above lemmas, we prove the main theorem. The high-level structure of the proof is as follows. We reduce the problem of privately sampling from the class $\carb$ given a dataset with size distributed as $\Po(n)$ to privately sampling from the class $\cB^{\otimes 2k}$ given a dataset of size $\frac{n}{2}$ (the reduction is formally described in Algorithm~\ref{alg:redsampler}). Let $\distr$ be the unknown distribution from $\carb$. 
We assume that we have a private product distribution sampler that is accurate given $n/2$ samples. We then use Lemmas~\ref{lem:datatrans} and~\ref{lem:easyunivtrans} to argue that conditioned on the special element being chosen correctly in Algorithm~\ref{alg:redsampler}, the output distribution of Algorithm~\ref{alg:redsampler} is close in total variation distance to $\distr$. We then use Lemma~\ref{lem:specialel} to argue that the special element is chosen in Algorithm~\ref{alg:redsampler} with high probability. This implies that Algorithm~\ref{alg:redsampler} is an accurate sampler for $\carb$. By Lemma~\ref{lem:privred}, it is $(\eps, \delta)$-differentially private. Hence, we can invoke the lower bound for privately sampling from $\carb$ to obtain a lower bound on the number of samples for accurately sampling from $\cB^{\otimes 2k}$.
\begin{proof}[Proof of Theorem~\ref{thm:product-bern-lb}] 

 For any distribution $\distr$ and randomized function $f$, we will use $f(\distr)$ to represent the distribution of the random variable obtained by applying $f$ to a random variable distributed according to $\distr$. Also, without loss of generality, we can assume that $d$ is greater than some fixed constant. This is because for $d$ smaller than that constant, the lower bound for privately sampling from Bernoulli distributions in Theorem~\ref{thm:bernoulli-lb} directly gives the lower bound.
 
 Fix $n > \frac{2C\log k}{\alpha \eps}$. Set $k=\frac{d}{2}$. Fix any distribution $\distr \in \carb$, and let its special element be $s$. Let $C$ be the constant in the statement of Lemma~\ref{lem:specialel}. Let $\Datarv$ be a dataset of size $\Po(n)$ with entries drawn independently from $\distr$. 
 Fix an arbitrary $(\eps/2, \delta)$-DP sampler $\sampler_{prod}$ for $\cB^{\otimes d}$ that is $\frac{\alpha}{25}$-accurate when given $\frac{1}{2}[n - \frac{2C \log k}{\alpha \eps}]$ samples. Let $\corr$ be the corresponding product distribution of $\distr$. Run $\sampler_{red}$ with the following inputs: dataset $\Datarv$, privacy parameters $\eps,\delta$, the values $k,n$, accuracy parameter $\alpha$, and black box access to $\sampler_{prod}$. Let event $E_{s}$ represent  successfully choosing the special element $s$ in Algorithm~\ref{alg:redsampler} (this event always occurs when $L$ satisfies the condition in Step~\ref{step:largespec} and the special element is chosen correctly in  Step~\ref{step:specialel}). Let the output distribution of Algorithm~\ref{alg:redsampler} conditional on $E_{s}$ be $Q_{E_s}$. 

By Poissonization (Lemma~\ref{lem:multtopois}), $L$ and $R$ sampled in Step~\ref{step:datasplitval} of Algorithm~\ref{alg:redsampler} are independent Poisson random variables with means $\frac{2C \log k}{\alpha \eps}$ and $n - \frac{2C \log k}{\alpha \eps}$, respectively.

For the next part of the proof, condition on event $E_s$ occurring. We show that conditioned on this event, the output distribution of Algorithm~\ref{alg:redsampler} is close in total variation distance to distribution $\distr$.

Observe that the conditioning on $E_s$ does not change the distribution of $R$, since only the first partition consisting of $L$ samples is used to determine the special element (and $L$ is independent of $R$). By Lemma~\ref{lem:datatrans}, Step~\ref{step:datasettrans} then produces a dataset of size $\frac{1}{2}[n-\frac{2C \log k}{\alpha \eps}]$ that is distributed as independent samples from $\corr$. Next, let $Q_{prod}$ be the output distribution of Step~\ref{step:privateprodsamp}. By the accuracy of the private product sampler, $d_{TV}(Q_{prod}, \corr) \leq \alpha/25$. By the information processing inequality for total variation distance (Lemma~\ref{lem:postTV}), $$d_{TV}(\reduction^{\leftarrow}(Q_{prod}, s), \reduction^{\leftarrow}(\corr,s)) \leq \alpha/25.$$
By Lemma~\ref{lem:easyunivtrans}, $d_{TV}(\reduction^{\leftarrow}(\corr,s), \distr) \leq (\alpha^*)^2$, where $\alpha^* = 60\alpha$. Additionally,  $\reduction^{\leftarrow}(Q_{prod},s) = Q_{E_s}$ is the output distribution of Algorithm~\ref{alg:redsampler} conditional on event $E_s$. Hence, for sufficiently small $\alpha$, by the triangle inequality, 
\begin{align} \label{eq:totalvarcond}
d_{TV}(Q_{E_s},\distr) \leq \alpha/25 + (\alpha^*)^2 \leq \alpha/13.
\end{align}
Next, we show that the probability of choosing the special element incorrectly (event $\overline{E_s}$) in Algorithm~\ref{alg:redsampler} is small. By a union bound, this probability is at most the sum of the probability that the random variable $L$ was not large enough in Step~\ref{step:largespec}, and the probability that Step~\ref{step:specialel} chooses the wrong special element. By a tail bound on Poisson random variables (Claim~\ref{lem:poiss_tail} \cite{clementpoiss}), the probability that $L<\frac{C\log k}{\alpha \eps}$ is at most $e^{-\frac{C \log k}{6\alpha \eps}} \leq e^{-\frac{C \log(4/\alpha)}{6}} \leq  \frac{\alpha}{4}$ for sufficiently small $\alpha$ (upper bounding $\eps$ by $1$, lower bounding $\log k$ by $1$, and lower bounding $\frac{1}{\alpha}$ by $\log(4/\alpha)$). By Lemma~\ref{lem:specialel}, the probability of choosing the wrong special element in Step~\ref{step:specialel} is at most $\frac{\alpha}{2}$. Hence, the probability of $\overline{E_s}$ is at most $\alpha/4 + \alpha/2 = 3\alpha/4$.

Let $\distroutput{\sampler_{red}, \distr}$ be the output distribution of Algorithm~\ref{alg:redsampler}. Next, we show that $\distroutput{\sampler_{red}, \distr}$
is close to distribution~$\distr$. 

We first invoke a standard lemma to argue that the output distribution of Algorithm~\ref{alg:redsampler} (i.e. distribution $\distroutput{\sampler_{red}, \distr}$)
is close to the output distribution conditioned on event $E_s$ (i.e. distribution $Q_{E_s}$).  
Accounting for choosing the special element incorrectly, we can use Lemma~\ref{lem:TVcond} with $\beta = 3\alpha/4$, event $E_s$ and distribution $D$ being the output distribution $\distroutput{\sampler_{red}, \distr}$ to argue that 
\begin{align}\label{eq:condtvspecial}
    d_{TV}(\distroutput{\sampler_{red}, \distr}, Q_{E_s}) \leq \frac{3\alpha/4}{1-3\alpha/4} \leq \frac{3\alpha/4}{1-3/16} \leq  \frac{12\alpha}{13},
\end{align}
where we have used $\alpha \leq 0.25$. Finally, by the triangle inequality, using (\ref{eq:totalvarcond}) and (\ref{eq:condtvspecial}), we get that $$d_{TV}(\distroutput{\sampler_{red}, \distr}, \distr) \leq d_{TV}(\distroutput{\sampler_{red}, \distr}, Q_{E_s}) + d_{TV}(Q_{E_s}, \distr) \leq \frac{12\alpha}{13} + \frac{\alpha}{13} = \alpha.$$
Hence, Algorithm~\ref{alg:redsampler} is $\alpha$-accurate for $\carb$. By Lemma~\ref{lem:privred}, Algorithm~\ref{alg:redsampler} is $(\eps, \delta)$-differentially private. Putting both together, Algorithm~\ref{alg:redsampler} is an $(\eps, \delta)$-differentially private, $\alpha$-accurate sampler for $\carb$ with sample size distributed as $\Po(n)$. Since $d$ (and hence $k$) is sufficiently large, the lower bound for Poisson samplers for $\carb$ (Lemma~\ref{lem:k-ary-lb-pois}) then implies that $n \geq \frac{c k}{\alpha \eps}$ for some constant $c$. For some other constant $c'$, this implies that $\frac{1}{2}[n-\frac{2C \log k}{\alpha \eps}] > \frac{c' 2k}{\alpha \eps} = \frac{c'd}{\alpha \eps}$. Hence, for sufficiently large $d$, and $\eps \leq \frac{1}{2}$, every $(\eps, \delta)$-DP sampler that is $\frac{\alpha}{25}$-accurate on the class $\cB^{\otimes d}$ needs a dataset of size $\Omega(\frac{d}{\alpha \eps})$. 
%
\end{proof}

\ifnum\neurips=1
\section{Products of Bernoulli distributions with bounded bias}\label{sec:bounded-bias}
\else
\section{Products of Bernoulli Distributions with Bounded Bias}\label{sec:bounded-bias}
\fi

\subsection{Sampling Algorithms for Products of Bernoullis with Bounded Bias}

In this section, we consider Bernoulli distributions and, more generally, products of Bernoulli distributions with bounded bias. We show that, when the bias is bounded, differentially private sampling can be performed with datasets of significantly smaller size than in the general case. For Bernoulli distributions with bounded bias, we achieve this (in Theorem~\ref{thm:bernoulli-bb}) with pure differential privacy, that is, with $\delta=0$. For products of Bernoulli distributions, we give a zCDP algorithm (see Theorem~\ref{thm:bernoulli-product-bb}).
Theorems~\ref{thm:bernoulli-bb} and~\ref{thm:bernoulli-product-bb}, in conjunction with Lemma~\ref{prelim:relate_dp_cdp} relating $\rho$-zCDP and differential privacy, directly yield Theorem~\ref{thm:intro-bernoulli-product-bb}. In Section~\ref{sec:bb-lb}, we prove our lower bound for products of Bernoulli distributions with bounded bias, encapsulated in Theorem~\ref{thm:intro-product-bb-lb}. 
\ifnum\neurips=1
\subsection{Private sampler for Bernoulli distributions with bounded bias}\label{sec:bernoulli-bb}
\else 
\subsubsection{Private Sampler for Bernoulli Distributions with Bounded Bias}\label{sec:bernoulli-bb}
\fi

First, we consider the class $\cBB$ of Bernoulli distributions (see Definition~\ref{prelim:bern_def}) with an unknown bias $p\in\big[\frac 13,\frac 23\big].$ 
Even though class $\cBB$ is the hardest to learn privately among the classes of Bernoulli distributions, we show in the next theorem that private sampling from this class is easy.
\begin{theorem}\label{thm:bernoulli-bb}
For all $\eps>0$ and $\alpha\in (0,1)$, there exists an $(\eps,0)$-differentially private sampler for the class $\cBB$ of Bernoulli distributions with bias in $\big[\frac 13,\frac 23 \big]$ that is  $\alpha$-accurate for datasets of size  
$n=O(\frac 1{\eps}+\ln \frac 1 {\alpha})$.
\end{theorem}
\begin{proof}
We use $\clipab{\ \cdot \ }$ to denote rounding an arbitrary real number to the nearest value in $[a,b]$.
Consider the following sampler $\sampler_{clip}$: on input $\Datafixed\in \{0,1\}^n,$
compute the  sample proportion $\hat p=\frac {1}n \sum_{i\in[n]} \datafixed_i$, obtain a clipped proportion $\tilde p = \clipquarter{\hat p}$, and output $b\sim\Ber(\tilde p)$. 

\begin{claim}\label{claim:bernoulli-bb-accuracy}
Sampler $\sampler_{clip}$ is $\alpha$-accurate on the class $\cBB$ with dataset of size $n\geq 72\ln\frac {6}{\alpha}$. 
\end{claim}
\begin{proof}
Let the ``good'' event $E$ be that no rounding occurs when sample proportion is clipped, that is, $\tilde p=\hat p$.
Since $p\in[1/3,2/3]$,
\begin{align}\label{eq:prob-of-E-bar}
\Pr[\overline{E}]=\Pr\Big[\hat p\notin [1/4,3/4]\Big]
\leq\Pr\Big[|\hat{p}-p|\geq\frac 1{12}\Big]\leq 2e^{-n/72}\leq\frac \alpha{3},
\end{align}
where we applied the Hoeffding bound (specifically, that $\Pr[|\hat p -\E[\hat p]|\geq t]\leq 2 e^{-2nt^2}$) and our lower bound on $n.$ 

Let $Q$ be the distribution of the output bit $b$ for a dataset selected i.i.d.\ from $\Ber(p).$ Then, by Claim~\ref{clm:bern_acc_eq}  and the description of $\sampler_{clip}$,
$$
d_{TV}(Q, \Ber(p)) =|\E(b)-p|
=|\E(\tilde p)-\E(\hat p)|.
$$
Next, we observe that $\E[\tilde p | E] = \E[\hat p | E]$ and use (\ref{eq:prob-of-E-bar}) to bound $d_{TV}(Q, \Ber(p)).$ Specifically,
\begin{align*}
\E[\hat p] &=\E[\hat p | E]\cdot\Pr[E]+\E[\hat p | \overline{E}]\cdot\Pr[\overline{E}]
\leq \E[\hat p | E]\cdot 1+1\cdot\Pr[\overline{E}]\\
&\leq \E[\hat p | E] +\frac \alpha 3 
= \E[\tilde p | E] +\frac \alpha 3
\leq \frac{\E[\tilde p]}{\Pr[E]}+\frac \alpha 3
\leq \frac{\E[\tilde p]}{1-\alpha/3}+\frac \alpha 3\\
&\leq \E[\tilde p] + \Big(\frac 1 {1-\alpha/3} -1\Big) +\frac \alpha 3
\leq \E[\tilde p] + \frac \alpha 2 +\frac \alpha 3
\leq \E[\tilde p] +\alpha.
\end{align*}
Similarly, $\E[\tilde p]\leq \E[\hat p]+\alpha.$
We get that $d_{TV}(Q, \Ber(p))=|\E(\tilde p)-\E(\hat p)|\leq \alpha,$ completing the accuracy analysis.
\end{proof}

\begin{claim}\label{claim:bernoulli-bb-privacy}
Sampler $\sampler_{clip}$ is $(4/n,0)$-differentially private. 
\end{claim}
\begin{proof}
By definition of the sampler, $\Pr[b=1]=\tilde p.$ Consider two datasets $\Datafixed$ and $\Datafixed'$ that differ in one record. The sample proportions
 $\hat p=\frac {1}n \sum_{i\in[n]} \datafixed_{i}$ and $\hat p'=\frac {1}n \sum_{i\in[n]} \datafixed'_{i}$ differ by at most $1/n$. Let $\tilde p$ and $\tilde p'$ be the corresponding clipped proportions, which also differ by at most $1/n$. Then, since $\tilde p\geq 1/4,$
 $$\tilde p'\leq \tilde p+\frac 1 n\leq \tilde p + \tilde p \frac 4 n
 =\tilde p\Big(1+\frac 4 n\Big)
 \leq \tilde p\cdot  e^{4/n},
 $$
where we used the fact that $1+t\leq e^t$ for all $t.$
Similarly, since $\tilde p\leq 3/4,$ the probabilities of returning $b=0$ for inputs $\Datafixed$ and $\Datafixed'$ also differ by at most a factor of $e^{4/n}$. Thus, $\sampler_{clip}$ is $4/n$-differentially private.
\end{proof}
Now we set $n\geq\max\big\{72\ln\frac 6\alpha,\frac 4 \eps\big\}$ and use Claims~\ref{claim:bernoulli-bb-accuracy} and~\ref{claim:bernoulli-bb-privacy} to get both accuracy and privacy guarantees. Observe that when $n\geq 4/\eps$, we get that $4/n\leq\eps$, that is, $\sampler_{clip}$ is $(\eps,0)$-differentially private.
This completes the proof of Theorem~\ref{thm:bernoulli-bb}.
\end{proof}


\ifnum\neurips=1
\subsubsection{Private sampler for products of Bernoulli distributions with bounded bias}\label{sec:bernoulli-prod-bb}
\else 
\subsubsection{Private Sampler for Products of Bernoulli Distributions with Bounded Bias}\label{sec:bernoulli-prod-bb}
\fi

In this section, we consider product distributions, where each marginal is a Bernouli distribution with a bias between 1/3 and 2/3. For this class, significantly fewer samples are needed for private sampling than for the general class of products of Bernoulli distributions.

\begin{theorem}\label{thm:bernoulli-product-bb}
For all all $\rho >0$ and $\alpha\in (0,1)$, there exists a $\rho$-zCDP sampler for the class $\cBB^{\otimes d}$ of products of Bernoulli distributions with bias in $\big[\frac 13,\frac 23 \big]$ that is  $\alpha$-accurate for datasets of size $n=O\Big(\frac {\sqrt{d}}{\sqrt{\rho}} + \log \frac d {\alpha}\Big)$.
\end{theorem}


\begin{proof}
The input to a sampler for $\cBB^{\otimes d}$ is an $n\times d$ matrix $\Datafixed\in \{0,1\}^{n\times d},$ 
where row $i$ contains the $i$-th record and column $j$ contains all input bits for $j$-th attribute. Recall sampler $\sampler_{clip}$ from the proof of Theorem~\ref{thm:bernoulli-bb}.
For each $j\in[d]$, our sampler runs sampler $\sampler_{clip}$ on column $j$ of $\Datafixed$ and records its output bit $b_j$; it returns the vector $b=(b_1,\dots,b_d)$.



We show that this sampler is $\alpha$-accurate on the class $\cBB^{\otimes d}$ with $n\geq 72 \ln \frac{6d}{\alpha}$ samples. 
Let $P=P_1\otimes\dots\otimes P_d$ be the input product distribution, where $P_j= \Ber(p_j)$ for all coordinates $j\in[d]$.
Let $Q=Q_1\otimes\dots\otimes Q_d$ be the distribution of the output vector $b$ for a dataset selected i.i.d.\ from $P.$  (Since the coordinates of $b$ are mutually independent, $Q$ is indeed a product distribution.) By Claim~\ref{claim:bernoulli-bb-accuracy} applied with $\alpha/d$ as accuracy parameter, $d_{TV} (Q_j,P_j)\leq \frac\alpha d$. By subadditivity of the statistical distance between two product distributions, 
$$d_{TV}(Q,P)\leq\sum_{j\in[d]} d_{TV} (Q_j,P_j)
\leq d\cdot\frac\alpha d =\alpha,$$
%
%
%
completing the proof that our sampler is $\alpha$-accurate.

%


Finally, we show that our sampler 
is $\rho$-zCDP when $n\geq \sqrt{8d/\rho}$. The sampler is a composition of $d$ algorithms, each returning one bit. 
By Claim~\ref{claim:bernoulli-bb-privacy},
these algorithms are $(4/n,0)$-differentially private and, consequently, also $\frac 8{n^2}$-zCDP. A composition of $d$ such algorithms is then $\frac {8d}{n^2}$-zCDP.
%
That is, when $n\geq \sqrt{8d/\rho}$, it is $\rho$-zCDP, as required.
%
\end{proof}

\subsection{Lower Bound for Products of Bernoullis with Bounded Bias}
\label{sec:bb-lb}

We now prove a lower bound that matches the guarantees of the algorithm of the previous section. 

\begin{theorem}[Theorem~\ref{thm:intro-product-bb-lb}, restated]\label{thm:bb-lb}
For all  sufficiently small $\alpha>0$, and
for all  $d,n \in \N$, $\eps \in (0,1]$,  and $\delta \leq \frac 1{100n}$, if there exists an $(\eps,\delta)$-differentially private sampler that is $\alpha$-accurate on
the class of products of $d$ Bernoulli distributions with biases in $\big[\frac 13,\frac 23 \big]$ on datasets of size $n$, then $n={\Omega}(\sqrt{d}/\eps)$. \asnote{Need to resolve dependency on $\delta$.}
\end{theorem}
To prove the theorem, we reduce the problem of accurately estimating the marginal biases of a product distribution over $\{0,1\}^d$ to the problem of sampling from the product distribution. This involves dividing the dataset into a constant number of disjoint parts and passing each part separately to a sampler for product distributions to obtain a constant number of independent samples. Then, by averaging the samples obtained, we get an estimate of the marginal biases. We also observe that a marginal estimator for the class $\cBB^{\otimes d}$ can be converted into a marginal estimator for the class $\cB^{\otimes d}$ with only a constant factor loss in accuracy. To do this, we flip every bit of every sample with probability $1/3$, which gives us a dataset that looks like it is drawn from a product distribution in $\cBB^{\otimes d}$. We then use the marginal estimator for $\cBB^{\otimes d}$ and transform the estimated biases back to the original range by multiplying by $3$ and subtracting $1$. Finally, applying the lower bound of Bun et al.~\cite{BunUV14j} for the sample complexity of marginal estimation for the class $\cB^{\otimes d}$, we obtain a lower bound on the sample complexity for the problem of accurately sampling from product distributions with bounded biases.

\begin{definition}[Marginal Estimator] For $\alpha', \beta', \gamma \in [0,1]$, and a class \class of distributions on $\bit{d}$, an algorithm \me is an $(\alpha', \beta', \gamma, \class)$-marginal estimator with sample size $n$ if, given $\Datarv \sim \distr^{\otimes n}$ where $\distr \in \class$,  with probability at least $1-\gamma$, algorithm \me returns $\tilde{\biasesfixed}_1, \ldots, \tilde{\biasesfixed}_d$ such that
\begin{equation*}
    |\{j \in [d] : |\biasesfixed_j - \tilde{\biasesfixed}_j| > \alpha' \}| < \beta' d \, .
\end{equation*}
\end{definition}

\begin{algorithm}
    \caption{Marginal Estimator $\me_c$ for $\cBB^{\otimes d}$}
    \label{alg:marginal_est}
    \hspace*{\algorithmicindent} \textbf{Input:} dataset $\Datafixed \in \bit{c n\times d}$, constant $c$, query access to sampler \sampler\\
    \hspace*{\algorithmicindent} \textbf{Output:} marginal estimates $\tilde{\Biasesfixed} = (\tilde{\biasesfixed}_1, \ldots, \tilde{\biasesfixed}_d)$
    \begin{algorithmic}[1] 
            \State Partition dataset \Datafixed into $c$ equal parts: $\Datafixed^{(1)}, \ldots, \Datafixed^{(c)}$
            \For{$i = 1$ to $c$}:
            \State $Y_i \gets \sampler(\Datafixed^{(i)})$ \Comment{Get $c$ independent samples from \sampler}
            \label{step:sample} 
            \EndFor
            \State $\tilde{\Biasesfixed} \gets \frac{1}{c} \sum_{i=1}^c Y_i$ \Comment{Compute marginal estimates} \label{step:computeemp}
            \State \Return $\tilde{\Biasesfixed}$
    \end{algorithmic}
\end{algorithm}

\begin{lemma}[Reduction from Marginal Estimation to Sampling]\label{lem:sam_to_me}
For all $\alpha, \beta_0, \gamma_0 \in (0,1)$, there exists $c\in \N$ such that for all $\eps,\delta>0$: 
if \sampler is an $(\eps, \delta)$-DP sampler that is $\alpha$-accurate on class $\cBB^{\otimes d}$ with sample size $n$, then $\me_c$ (Algorithm~\ref{alg:marginal_est}) is an $(\eps, \delta)$-DP, $(2 \alpha , \beta_0,\gamma_0, \cBB^{\otimes d})$-marginal estimator with sample size $cn$. 
%
\end{lemma}
\begin{proof}
Fix a distribution $\distr \in \cBB^{\otimes d}$. Let $\biasesfixed$ represent the vector of biases corresponding to $\distr$, and $Y_i$ for $i \in [c]$  and $\tilde{\biasesfixed}$ be as defined in Steps~\ref{step:sample} and~\ref{step:computeemp} of Algorithm~\ref{alg:marginal_est}. Consider any index $j \in [d]$. 
The expectations $\E[Y_i[j]]$ are the same for  all $i \in [c]$. Define $q[j]=\E[Y_1[j]]$.
Since $\sampler$ is $\alpha$-accurate with sample size $n$, we have $|q[j] - \biasesfixed_j| \leq \alpha$. 
Let $D$ be a positive constant to be set later. By Hoeffding's inequality (Claim~\ref{claim:hoeff}), 
$\Pr(|q[j] - \tilde{\biasesfixed}_j| \geq \frac{D}{\sqrt{c}}) \leq 2e^{-2D^2}$. By the triangle inequality, with probability at least $1-2e^{-2D^2}$, 
\begin{align}\label{eq:prob-of-bad-indices}
|\tilde{\biasesfixed_j} - \biasesfixed_j| \leq |\tilde{\biasesfixed_j} - q[j]| + |q[j] - \biasesfixed_j| \leq \alpha + \tfrac{D}{\sqrt{c}}.
\end{align}
Since (\ref{eq:prob-of-bad-indices}) holds for all $j \in [d]$,  the expected number of $j \in [d]$ such that $|q[j] - \biasesfixed_j| > \alpha + \frac{D}{\sqrt{c}}$ is at most $2d e^{-2D^2}$. By Markov's inequality, 
$$\Pr(|\{j \in [d] : |\biasesfixed_j - \tilde{\biasesfixed}_j| > \alpha + \tfrac{D}{\sqrt{c}} \}| \geq \tfrac{2d}{\gamma_0}  e^{-2D^2})) \leq \gamma_0.$$ 
Setting $D^2 = \frac 1 2 \ln(\frac{2}{\beta_0\gamma_0})$ ensures $\tfrac{2d}{\gamma_0}e^{-2D^2} \leq \beta_0 d$. Setting $c =\ceil{\frac{D^2}{\alpha^2}}$ further ensures that $ \alpha + \tfrac{D}{\sqrt{c}} \leq 2\alpha$. We thus get the desired accuracy guarantee on $\me_c$ when $c = \ceil{\frac{1}{2\alpha^2} \ln(\frac{2}{\beta_0\gamma_0})}$.

Finally, changing one entry in the dataset $\Datafixed$ changes a single entry in only one of the parts $\Datafixed^{(i)}$, and only this part is fed to the $i^{th}$ call to $\sampler$. Since $\sampler$ is $(\eps, \delta)$-DP, so is $\me_c$. This proves the lemma. 
\end{proof}

Next we show how to transform a marginal estimator 
for a product of bounded Bernoulli distributions into a marginal estimator for a product of arbitrary Bernoulli distributions.
Let \bsc denote a \emph{binary symmetric channel} with bias $1/3$. That is, on input $\Datafixed \in \bit{n\times d}$, each bit gets flipped independently with probability $1/3$. In particular, $\bsc(\Datafixed) = \Datafixed \oplus \Datarvz$, where $\Datarvz \sim \Ber(1/3)^{\otimes n \times d}$.

\begin{algorithm}
    \caption{Marginal Estimator $\bernest$ for $\cB^{\otimes d}$}
    \label{alg:marginal_est_unbounded}
    \hspace*{\algorithmicindent} \textbf{Input:} dataset $\Datafixed \in \bit{n\times d}$, query access to marginal estimator $\me$ for $\cBB^{\otimes d}$\\
    \hspace*{\algorithmicindent} \textbf{Output:} marginal estimates $\tilde{\Biasesfixed} = (\tilde{\biasesfixed}_1, \ldots, \tilde{\biasesfixed}_d)$
    \begin{algorithmic}[1] 
            \State $\Datafixed^* \gets \bsc(\Datafixed)$  \Comment{Change initial distribution}
            \State $\Biasesfixed^* \gets \me(\Datafixed^*)$  \Comment{Get empirical estimates}
            \State $\tilde{\Biasesfixed} \gets (3\cdot\biasesfixed^*_1 - 1, \ldots, 3\cdot\biasesfixed^*_d - 1)$ \Comment{Rescale empirical estimates}
            \State \Return $\tilde{\Biasesfixed}$
    \end{algorithmic}
\end{algorithm}

\begin{lemma}[Reduction from General  to  Bounded Biases]\label{lem:me_bounded_to_unbounded}
If \me is an $(\alpha', \beta', \gamma, \cBB^{\otimes d})$-marginal estimator with sample size $n$, then $\bernest$ (in Algorithm~\ref{alg:marginal_est_unbounded}) is a $(3\alpha', \beta', \gamma, \cB^{\otimes d})$-marginal estimator with sample size $n$. If $\me$ is $(\eps, \delta)$-differentially private, then so is $\bernest$.
\end{lemma}

\begin{proof}
We begin with the accuracy proof. Fix an $(\alpha', \beta', \gamma, \cBB^{\otimes d})$-marginal estimator \me with sample size~$n$. Fix a distribution $\distr \in \cB^{\otimes d}$ with  biases $\Biasesfixed = (\biasesfixed_1, \ldots, \biasesfixed_d)$. Let $\Datarv \sim \distr^{\otimes n}$. Denote the output distribution of $\bsc(\Datarv)$ by $\corr$ and its biases by $\Biasesfixed' = (\biasesfixed'_1, \ldots, \biasesfixed'_d)$. Let $\bsc(\Datarv)^j_i = \Datarv^j_i \oplus \Datarvz$ be the output for the $j$th attribute on the $i$th data record, where $\datarvz^j_i \sim \Ber(1/3)$. Then, for all $j\in[d]$ and all $i \in [n]$, 
\begin{align*}
    \biasesfixed'_j 
    & = \Pr[\bsc(\Datarv)^j_i = 1] 
     = \Pr[\datarv^j_i \oplus \datarvz^j_i = 1 ] 
     = \Pr[\datarv^j_i = 1 \land \datarvz^j_i = 0] + \Pr[\datarv^j_i = 0 \land \datarvz^j_i = 1] \\
    & = \biasesfixed_j \cdot \frac{2}{3} + (1-\biasesfixed_j) \cdot \frac{1}{3} = \frac{\biasesfixed_j}{3} + \frac{1}{3}.
\end{align*}
Thus, $\bsc(\Datarv)\in \cBB^{\otimes d}$, as desired. Since $\me$ is a $(\alpha', \beta', \gamma, \cBB^{\otimes d})$-marginal estimator with sample size~$n$, estimator \me returns $(\biasesfixed^*_1, \ldots, \biasesfixed^*_d)$ such that with probability at least $1-\gamma$,
\begin{equation}\label{eq:marg_est}
    |\{j \in [d] : |\biasesfixed'_j - \biasesfixed^*_j| > \alpha' \}| < \beta' d.
\end{equation}
Substituting $\biasesfixed'_j=\biasesfixed_j/3+1/3$ in the left-hand side of in~\eqref{eq:marg_est} and  then using $\tilde{\biasesfixed}_j=3\biasesfixed^*_j+1$, we get
\begin{align*}
    |\{j \in [d] : |\frac{\biasesfixed_j}{3} + \frac{1}{3} -  \biasesfixed^*_j| > \alpha' \}|
    & = |\{j \in [d] : |\biasesfixed_j - (3\biasesfixed^*_j - 1)| > 3\alpha' \}| \\
    & = |\{j \in [d] : |\biasesfixed_j - \tilde{\biasesfixed}_j| \leq 3\alpha' \}|.
\end{align*}
Thus, \bernest is a $(3\alpha', \beta', \gamma, \cB^{\otimes d})$-marginal estimator with sample size $n$. 

Finally, we show that  \bernest is differentially private.
Suppose \me is $(\eps, \delta)$-differentially private. Fix neighboring datasets \Datafixed and $\Datafixed'$ that differ on record $i$. Then $\bsc(\Datafixed)$ and $\bsc(\Datafixed')$ still only differ on record $i$. The output of \bernest is a post-processing of \me, and thus \bernest is $(\eps, \delta)$-differentially private.
\end{proof}

We prove our main result by combining Lemmas~\ref{lem:sam_to_me}--\ref{lem:me_bounded_to_unbounded} with a lower bound on marginal estimation that is obtained using the fingerprinting codes technique of Bun, Ullman and Vadhan~\cite{BunUV14j}. We use a corollary of the version of the result from~\cite{BunUV14j} presented by Kamath and Ullman~\cite{KamathUprimerpaper20}.

\begin{theorem}[Consequence of \cite{KamathUprimerpaper20}, Theorem 3.3
]\label{thm:buv_LB}
Suppose there exists a function $n = n(d)$, such that for every $d \in \N$, there is a $(\alpha_0, \beta_0, \gamma_0, \cB^{\otimes d})$-marginal estimator $\me: \{0,1\}^{n \times d} \to \R^d$ that is  $(\eps,\frac{1}{100n})$-DP, where $\alpha_0, \beta_0, \gamma_0 \in (0,1)$ are sufficiently small absolute constants. Then $n  = {\Omega}(\sqrt{d}/\eps)$.
\end{theorem}
\begin{proof}[Proof of Theorem~\ref{thm:bb-lb}]
Let $\eps,\delta \in (0,1]$ with $\delta<\frac{1}{100n}$. 
Let $\alpha_0,\beta_0,\gamma_0$ be the constants from Theorem~\ref{thm:buv_LB}, and set $\alpha = \frac{\alpha_0}{6}$. 
Let $\sampler$ be an $\alpha$-accurate, $(\eps,\delta)$-DP sampler for the class $\cBB^{\otimes d}$ for datasets of some size $n$. By Lemma~\ref{lem:sam_to_me}, there exists an $(\eps, \delta)$-differentially private,  $(\frac{\alpha_0}{3},\beta_0, \gamma_0,\cBB^{\otimes d})$-marginal estimator $\me_c$ for datasets of size $cn$ for an absolute constant $c= c(\alpha,\beta_0,\gamma_0)$. 
By Lemma~\ref{lem:me_bounded_to_unbounded}, \bernest defined in Algorithm~\ref{alg:marginal_est_unbounded} is a $(\eps, \delta)$-differentially private, $(\alpha_0, \beta_0,\gamma_0,\cB^{\otimes d})$-marginal estimator for datasets of size $cn$. 
By Theorem~\ref{thm:buv_LB}, $n = {\Omega}(\sqrt{d}/\eps)$, as desired.
%
%
\end{proof}

%

%

\addcontentsline{toc}{section}{References}
\bibliographystyle{plain}
\bibliography{refs}{}

%

\ifnum\supplemental=0
\appendix
\newpage
\section*{Appendix}
\fi 

\section{Inequalities Used in Technical Sections} \label{sec:inequalities}
We argue that the moments of the average of several identically distributed random variables are no larger than the corresponding moments of the individual random variables.
\begin{claim}\label{claim:momavg}
If random variables $A_1, \dots, A_k$ are identically distributed, then, for all $\lambda > 0$,
$$\E\left[ \left(\frac{1}{k}\sum_{i=1}^k A_i \right)^{\lambda} \right] \leq \E\left[ A_1^{\lambda} \right].$$
\end{claim}
\begin{proof}
By Jensen's inequality, $\left(\frac{1}{k}\sum_{i=1}^k A_i \right)^{\lambda} \leq \frac{1}{k}\sum_{i=1}^k A_i^{\lambda}$ for any fixed values of $A_1, \dots, A_k$. We take expectation on both sides, then use the linearity of expectation and that  $A_1, \dots, A_k$ are identically distributed:
\begin{align*}
    \E\left[\left(\frac{1}{k}\sum_{i=1}^k A_i \right)^{\lambda}\right] 
    \leq \E\left[\frac{1}{k}\sum_{i=1}^k A_i^{\lambda}\right] 
     = \frac{1}{k}\sum_{i=1}^k \E\left[ A_i^{\lambda}\right] = \E\left[ A_1^{\lambda} \right].
\end{align*}
\end{proof}
We also use Bernoulli's inequality in our lower bound for product distributions.
\begin{lemma}[Bernoulli's inequality]\label{lem:bernoulli}
For all real $a \geq -1$ and nonnegative integers $r$, $(1+a)^r \geq 1+ar$. 
\end{lemma}

\subsection{Concentration Inequalities}
\begin{claim}[Chernoff Bounds]\label{claim:cher_bounds}
Let $A$ be the average of $m$ independent 0-1 random variables with $\mu=\E[A]$. For $\gamma \in (0,1)$,
\begin{align*}
  \Pr[A \geq \mu(1+\gamma)] \leq e^{-\frac{\gamma^2 \mu m}{3}};\\ 
  \Pr[A \leq \mu(1-\gamma)] \leq e^{-\frac{\gamma^2 \mu m}{2}}.
\end{align*}
For $\gamma \geq 0$, 
\begin{align*}
\Pr[A \geq \mu(1+\gamma)] \leq e^{-\frac{\gamma^2 \mu m}{2+ \gamma}};\\
\Pr[A \leq \mu(1-\gamma)] \leq e^{-\frac{\gamma^2 \mu m}{2 + \gamma}}.
\end{align*}
\end{claim}

\begin{claim}[\cite{KLSU19}, Lemma 2.8, Gaussian Concentration]\label{lem:gaussconc}
If $A$ is drawn from $\mathcal{N}(0,\sigma^2)$, then, for all $t > 0$,
$$ \Pr\left( |A| > t \sigma \right) \leq 2e^{-t^2 / 2}.$$
\end{claim}

\begin{claim}[Hoeffding's Inequality]\label{claim:hoeff}
Let $A$ be the average of $m$ independent random variables in the interval $[0,1]$ with $\mu=\E[A]$. For $h \geq 0$,
\begin{align*}
  \Pr[A - \mu \geq h] \leq e^{-2mh^2}. \\
  \Pr[\mu - A \geq h] \leq e^{-2mh^2}.
\end{align*}
\end{claim}

\section{Lemmas on Privacy Amplification by Subsampling}\label{app:cited-results} \label{sec:resow}
\begin{definition}[\cite{li2012sampling}, Definition 3] An algorithm \sampler is $(\beta, \eps, \delta)$-DPS if and only if $\beta > \delta$ and the algorithm $\sampler^\beta$ is $(\eps, \delta)$-DP where $\sampler^\beta$ denotes the algorithm to first sample with probability $\beta$ (include each tuple in the input dataset with probability $\beta$), and then apply \sampler to the sampled dataset.
\end{definition}

\begin{theorem}[\cite{li2012sampling}, Theorem 1]\label{thm:LQS12}
Any $(\beta_1, \eps_1, \delta_1)$-DPS algorithm is also $(\beta_2, \eps_2, \delta_2)$-DPS for any $\beta_2 < \beta_1$ where $\eps_2 = \ln\left(1 + \left(\frac{\beta_2}{\beta_1} (e^{\eps_1} - 1) \right)\right)$, and $\delta_2 = \frac{\beta_2}{\beta_1}\delta_1$.
\end{theorem}

\fi

\end{document}

%% file: macros.tex
\newcommand{\notes}{0} 
\usepackage[utf8]{inputenc}
\usepackage{amsmath, fullpage}
\usepackage{bbm}
\usepackage{amssymb}
\usepackage{amsthm}
\usepackage{algorithm}
\usepackage[noend]{algpseudocode}
\usepackage{algorithmicx}
\usepackage{xcolor}
\usepackage{comment}
\usepackage{xspace}
\usepackage{bbm}
\usepackage{ulem} 
\usepackage{paralist}
\usepackage{multicol}
\usepackage{hyperref}
\hypersetup{
    colorlinks=true,
    linkcolor=blue,
    filecolor=magenta,      
    urlcolor=cyan,
    linktocpage=true,
}

\ifnum\tpdp=0 
    \newtheorem{theorem}{Theorem}[section]
\else
    \newtheorem{theorem}{Theorem}
\fi
\newtheorem*{theorem*}{Theorem}

\newtheorem{lemma}[theorem]{Lemma}
\newtheorem{claim}[theorem]{Claim}
\newtheorem{corollary}[theorem]{Corollary}
\newtheorem{definition}[theorem]{Definition}

\newcommand{\R}{\mathbb{R}}
\newcommand{\N}{\mathbb{N}}
\newcommand{\eps}{\varepsilon}
\newcommand{\universe}{\ensuremath{\mathcal{U}}\xspace}
\newcommand{\class}{\ensuremath{\mathcal{C}}\xspace} 
\newcommand{\cB}{\ensuremath{\mathcal{B}}\xspace} 
\newcommand{\cBB}{\ensuremath{\cB_{[\frac 13,\frac 23]}}\xspace}

\newcommand{\carb}{\ensuremath{\class_{\alpha, 2k+1}}\xspace} 
\newcommand{\prodoutput}{b}
\newcommand{\hist}{\ensuremath{\mathbf{h}}\xspace}
\newcommand{\distr}{\ensuremath{P}\xspace}
\newcommand{\sampler}{\ensuremath{\mathcal{A}}\xspace}
\newcommand{\Datarv}{\ensuremath{\mathbf{X}}\xspace}
\newcommand{\Datarvy}{\ensuremath{\mathbf{Y}}\xspace}
\newcommand{\Datarvz}{\ensuremath{\mathbf{Z}}\xspace}

\newcommand{\datarvz}{\ensuremath{Z}\xspace}
\newcommand{\Datafixed}{\ensuremath{\mathbf{x}}\xspace}
\newcommand{\Datafixedy}{\ensuremath{\mathbf{y}}\xspace}

\newcommand{\datarv}{\ensuremath{X}\xspace}
\newcommand{\datafixed}{\ensuremath{x}\xspace}
\newcommand{\datafixedy}{\ensuremath{y}\xspace}

\newcommand{\el}{\ensuremath{s}\xspace}
\newcommand{\me}{\ensuremath{\mathcal{M}}\xspace}
\newcommand{\bsc}{\ensuremath{BSC_{1/3}}\xspace} 
\newcommand{\indicator}{\ensuremath{\mathbbm{1}}\xspace}
\newcommand{\distroutput}[1]{\ensuremath{Q_{#1}\xspace}}

\newcommand{\Biasesfixed}{\ensuremath{\mathbf{p}}\xspace} 
\newcommand{\biasesfixed}{\ensuremath{p}\xspace} 
\newcommand{\bucket}{\ensuremath{B}\xspace} 
\newcommand{\corr}{\ensuremath{\distr_{corr}}\xspace} 

\DeclareMathOperator*{\argmin}{arg\,min}

\newcommand{\Po}{\ensuremath{\mathsf{Po}}\xspace}
\newcommand{\Lap}{\mathsf{Lap}}
\newcommand{\Ber}{\ensuremath{\mathsf{Ber}}\xspace} 
\newcommand{\Gauss}{\ensuremath{\mathcal{N}}\xspace} 
\newcommand{\Bin}{\ensuremath{\mathsf{Bin}}\xspace}
\newcommand{\Mult}{\ensuremath{\mathsf{Mult}}\xspace}

\DeclareMathOperator*{\E}{\mathbb{E}}

\newcommand{\bit}[1]{{\{0,1\}^{#1}}}

\newcommand{\paren}[1]{{\left ( {#1} \right)}}
\newcommand{\bparen}[1]{{\big ( {#1} \big)}}
\newcommand{\Bparen}[1]{{\Big ( {#1} \Big)}}
\newcommand{\ceil}[1]{{\left \lceil {#1} \right\rceil}}

\newcommand{\clipab}[1]{{\left [ {#1} \right]_a^b}}
\newcommand{\clipquarter}[1]{{\left [ {#1} \right]_{1/4}^{3/4}}}
\newcommand{\reduction}{\ensuremath{\mathcal{R}}}
\newcommand{\fpsampler}{\ensuremath{\mathcal{A}_{\text{FCB}}}\xspace} 
\newcommand{\posampler}{\ensuremath{\mathcal{A}_{\Po}}\xspace} 
\newcommand{\karysampler}{\ensuremath{\mathcal{A}_{\class_k}}\xspace} 
\newcommand{\bernest}{\ensuremath{\me_{\Ber}}\xspace} 

\ifnum\notes=1 
\newcommand{\srnote}[1]{{\color{blue}\footnote{{\color{blue} {\bf SR:} #1}}}}
\newcommand{\msnote}[1]{{\color{magenta}\footnote{{\color{magenta} {\bf M:} #1}}}}
\newcommand{\ssnote}[1]{{\color{red}\footnote{{\color{red} {\bf SS:} #1}}}}
\newcommand{\asnote}[1]{{\color{brown}\footnote{{\color{brown} {\bf A:} #1}}}}
\newcommand{\old}[1]{\sout{#1}}
\else
\newcommand{\srnote}[1]{}
\newcommand{\msnote}[1]{}
\newcommand{\ssnote}[1]{}
\newcommand{\asnote}[1]{}
\newcommand{\old}[1]{}
\fi 

\ifnum\notes=1

\else 

\fi


%% file: lb-proof-short-neurips.tex
%
%

\section{The lower bound for $k$-ary distributions}\label{sec:k-ary-lb-frequency-count-based}\label{sec:kary-short}

In this section, we prove the lower bound for sampling from discrete distributions with the universe of size at least 3.  As discussed in Section~\ref{intro:overview}, we prove that $n = \Omega(k/\alpha \eps)$. The crux of the proof is the case where the sampler is frequency-count-based, Poisson, and is $(\eps, \delta)$-DP for small $\eps$. The transformation from general samplers to this restricted type of sampler is presented in the supplementary materials. This transformation together with the following Lemma~\ref{lem:main-k-ary-lb} completes the proof of Theorem~\ref{thm:intro-k-ary-lb} for discrete distributions with the universe of size at least 3.

\begin{lemma}\label{lem:main-k-ary-lb} Fix $k, 
n \in{\mathbb N}, \alpha \in (0,0.02], \eps \in(0,1/\ln (1/\alpha)],$ and $\delta \in [0, 0.1\alpha\eps/k]$. Let $\class_{2k+1}$ denote the class of discrete distributions over the universe $[2k+1]$. If sampler \sampler is $(\eps, \delta)$-DP, frequency-count-based, and $\alpha$-accurate on class $\class_{2k+1}$ with dataset size distributed as $\Po(n)$, then $n > \frac 1 {60}\cdot \frac{k}{\alpha \eps}$.
\end{lemma}

\begin{proof}
We consider the following distribution $\distr$ over the universe $\universe=[2k+1].$ Fix $\alpha^*=60 \alpha$ and a set $S^*\subset[2k]$ of size $k.$ Distribution $\distr$ has mass $\alpha^*/k$ on each element in $S^*$ and mass $1-\alpha^*$ on the {\em special} element $2k+1.$ 

Consider a sampler \sampler satisfying the conditions of Lemma~\ref{lem:main-k-ary-lb}. Let $\distroutput{\sampler, \distr}$ denote the output distribution of $\sampler$ when the dataset size $N \sim \Po(n)$ and the dataset $\Datarv \sim \distr^{\otimes N}$.
Observe that
\begin{eqnarray}\label{eq:main-dist-lb-delta}
d_{TV}(\distroutput{\sampler,\distr}, \distr)
  \geq \Pr_{\substack{N\sim\Po(n) \\ \Datarv\sim  \distr^{\otimes N}}}[\sampler(\Datarv) \notin Supp(\distr)].  
\end{eqnarray}
We show that if $n\leq \frac k{60\alpha \eps}$  and $\eps$ and $\delta$ are in the specified range, the right-hand side of (\ref{eq:main-dist-lb-delta}) is large. We start by deriving a lower bound on $\Pr[\sampler(\Datafixed)\notin Supp(\distr)]$ for a fixed dataset $\Datafixed$ of a fixed size $N$. Let $p_{j,F(\Datafixed)}$ be the probability that \sampler outputs a specific element in $[2k]$ that occurs $j$ times in $\Datafixed$, where $F(\Datafixed)$ is the frequency-count of \Datafixed; these probabilities are well-defined because \sampler is frequency-count-based. Let $F^*_0(\Datafixed)$ denote the number of elements in $[2k]$ that occur $0$ times in $\Datafixed$ (excluding the special element $2k+1$ from this count).
 By definition, $F^*_0(\Datafixed)\leq 2k.$ Consequently,
\begin{equation}\label{eq:notinsupport-delta}
\Pr[\sampler(\Datafixed) \notin Supp(\distr)] = k\cdot p_{0,F(\Datafixed)} \geq \frac{1}{2}\cdot  F^*_0(\Datafixed)\cdot p_{0,F(\Datafixed)}.
\end{equation}

The next claim uses the fact that 
\sampler is $(\eps,\delta)$-DP to show that the probability $p_{j,F(\Datafixed)}$ cannot be much larger than the probability that \sampler outputs a specific element in $\universe$ that does not appear in $\Datafixed$.

\begin{claim}\label{clm:epsdelfing}
For all $(\eps, \delta)$-DP samplers, frequency counts $f \in \mathbb{Z}^*$, and elements $j \in \universe$,
\begin{equation}\label{eq:grouppriv}
 p_{j,f} 
 \leq e^{\eps j} \Big( p_{0,f} + \frac{\delta}{\eps} \Big).
\end{equation}
\end{claim}

\begin{proof}
Consider a frequency count $f$ and a dataset $\Datafixed$ with $F(\Datafixed) = f$. Note that (\ref{eq:grouppriv}) is true trivially for all $j$ such that $F_j(\Datafixed)=0$ because, in that case, $p_{j,F(\Datafixed)}$ is set to $0$.

Fix any $j \in \universe$ with $F_j(\Datafixed) > 0$. Let $a$ be any element in $\universe$ that occurs $j$ times in the dataset $\Datafixed$. Let $b$ be any element in $\universe$ that is not in the support of the distribution $\distr$. Let $\Datafixed|_{a\rightarrow b}$ denote the dataset obtained by replacing every instance of $a$ in the dataset $\Datafixed$ with element $b$. By group privacy~\cite{DworkMNS06j},
\begin{equation}\label{eq:group_privacy-delta}
\Pr[\sampler(\Datafixed) = a] \leq e^{j \eps} \Pr[\sampler( \Datafixed|_{a \rightarrow b}) = a ] + \delta \cdot \frac{e^{\eps j} -1}{e^{\eps} - 1}.
\end{equation}
Note that the dataset $\Datafixed|_{a\rightarrow b} $ does not contain element $a$, since we've replaced every instance of it with $b$. Importantly, $F(\Datafixed|_{a \rightarrow b}) = F(\Datafixed)$ because $b$ is outside of the support of the distribution $\distr$ and hence does not occur in $\Datafixed$. Since $\sampler$ is frequency-count-based and $F(\Datafixed) = F(\Datafixed|_{a \rightarrow b})$, we get that $p_{0,F(\Datafixed)} = p_{0,F(\Datafixed|_{a \rightarrow b})}$. 
Substituting this into (\ref{eq:group_privacy-delta}) and using the fact that $e^{\eps}-1\geq \eps$ for all $\eps$, we get 
\begin{equation*}
p_{j,F(\Datafixed)} \leq e^{j \eps}\cdot p_{0,F(\Datafixed)} + \delta \cdot \frac{e^{\eps j} - 1}{e^\eps -1}
\leq e^{\eps j} \Big(p_{0,f} + \frac{\delta}{\eps} \Big).
\end{equation*}
This completes the proof of Claim~\ref{clm:epsdelfing}.
\end{proof}

For a dataset $\Datafixed$ and $i\in[2k+1]$, let $N_i(\Datafixed)$ denote the number of occurrences of element $i$ in $\Datafixed$.
Next, we give a lower bound on $\Pr[\sampler(\Datafixed) \notin Supp(\distr)]$ in terms of the counts $N_i(\Datafixed)$.

\begin{claim}\label{clm:nonsupport-lb-fixed-delta}
Let $N\in\mathbb{N}$ and $\Datafixed \in [2k+1]^N$ be a fixed dataset. Set 
$Y=\sum_{i \in S^*}  \left[ e^{N_i(\Datafixed) \eps} \right]$.
Then
$$\Pr[\sampler(\Datafixed) \notin Supp(\distr)] \geq\frac{1}{2}
\cdot\frac{\Pr[\sampler(\Datafixed) \in [2k]]}{1+Y/k} -\frac{k\delta}\eps.$$
\end{claim}

\begin{proof}
In the following derivation, we use the fact that that an element $j\in[2k]$ that appears $j$ times in $\Datafixed$ is returned by \sampler with probability $p_{j,F(\Datafixed)}$, then split the elements into those that do not appear in $\Datafixed$ and those that do, next use the fact that all elements from $[2k]$ that appear in $\Datafixed$ must be in $S^*$, then apply Claim~\ref{clm:epsdelfing}, and finally substitute $Y$ for $\sum_{i \in S^*}  \left[ e^{N_i(\Datafixed) \eps} \right]$:
\begin{align*}
\Pr[\sampler(\Datafixed) \in[2k]] 
&=\sum_{i\in[2k]}  p_{N_i(\Datafixed),F(\Datafixed)} 
=F^*_0(\Datafixed)\cdot p_{0,F(\Datafixed)}+\sum_{i\in[2k]\cap\Datafixed}  p_{N_i(\Datafixed),F(\Datafixed)} \\
& \leq F^*_0(\Datafixed)\cdot p_{0,F(\Datafixed)}+\sum_{i\in S^*}  p_{N_i(\Datafixed),F(\Datafixed)} \\
& \leq F^*_0(\Datafixed)\cdot p_{0,F(\Datafixed)}+\sum_{i\in S^*}  p_{0,F(\Datafixed)}\cdot \Big(e^{\eps N_i(\Datafixed)} + \frac{\delta}{\eps} \Big) \leq \Big(F^*_0(\Datafixed)+Y\Big)\Big(p_{0,F(\Datafixed)}+  \frac{\delta}{\eps}\Big).
\end{align*} 
We rearrange the terms to get \ 
$\displaystyle
p_{0,F(\Datafixed)}
\geq \frac{\Pr[\sampler(\Datafixed) \in[2k]]}{F^*_0(\Datafixed)+Y}-\frac \delta \eps.
$

Substituting this bound on $p_{0,F(\Datafixed)}$ into (\ref{eq:notinsupport-delta}), we obtain that $\Pr[\sampler(\Datafixed)\notin Support(P)]$ is at least
\begin{align*}
    \frac{1}{2} \cdot \frac {F^*_0(\Datafixed)\Pr[\sampler(\Datafixed) \in [2k]]}{F^*_0(\Datafixed) + Y} -\frac 12 \cdot\frac{F^*_0(\Datafixed)\cdot\delta}{\eps} 
   \geq\frac{1}{2}
\cdot\frac{\Pr[\sampler(\Datafixed) \in [2k]]}{1+Y/k} -\frac{k\delta}\eps,
\end{align*}
where in the inequality, we used that $k\leq F^*_0(\Datafixed) \leq 2k$.
This holds since the support of $\distr$ excludes $k$ elements from $[2k]$ and since $F^*_0(\Datafixed)$ counts only elements from $[2k]$ that do not appear in~$\Datafixed.$
\end{proof}

Finally, we give a lower bound on the right-hand side of (\ref{eq:main-dist-lb-delta}). Assume for the sake of contradiction that $n\leq \frac{k}{\alpha^* \eps}$. 
By Claim~\ref{clm:nonsupport-lb-fixed-delta},
\begin{align}
 \Pr_{N, \Datarv}[\sampler(\Datarv) \notin Supp(\distr)]
 &\geq \E_{N, \Datarv}\Big[\frac{1}{2}
\cdot \frac{\Pr[\sampler(\Datarv) \in [2k]]}{1+Y/k} -\frac{k\delta}\eps\Big]\nonumber\\
& = \frac{1}{2} \cdot \E_{N, \Datarv}\Big[\frac{\Pr[\sampler(\Datarv) \in [2k]]}{1+Y/k}\Big] -  \frac{k\delta}{\eps}. \label{eq:mainlbcond}
\end{align}
Next, we analyze the expectation in (\ref{eq:mainlbcond}). Let $E$ be the event that $\frac{Y}{k} \leq e^3$. By the law of total expectation, 
\begin{align}\label{eq:mainlbcondfirst}
     \E_{N, \Datarv}\Big[\frac{\Pr[\sampler(\Datarv) \in [2k]]}{1+Y/k}\Big]
     & \geq \E_{N, \Datarv}\Big[\frac{\Pr[\sampler(\Datarv) \in [2k]]}{1+Y/k} \big | E\Big] \Pr(E).
\end{align}
In Claims~\ref{claim:eventE} and~\ref{claim:exp-of-regular-output}, we argue that both $\Pr(E)$ and $\E_{N, \Datarv}\left[\frac{\Pr[\sampler(\Datarv) \in [2k]]}{1+Y/k} \big | E\right]$ are large.
\begin{claim}\label{claim:eventE}
Suppose $n\leq \frac k{60\alpha \eps}$. Let $E$ be the event that $\frac{Y}{k} \leq e^3$. Then
    $\Pr(E) \geq 1 - \alpha$.
    \end{claim}
\begin{proof}
Recall that $Y$ was defined as $\sum_{i \in S^*}  \left[ e^{N_i(\Datafixed) \eps} \right]$ for a fixed dataset $\Datafixed.$ Now consider the case when dataset $\Datarv$ is a random variable. 
If $N\sim\Po(n)$ and $\Datarv \sim \distr^{\otimes N}$ then $N_i(\Datarv) \sim \Po(\frac{\alpha^* n}{k})$ for all $i \in S^*$ and, additionally, the random variables $N_i(\Datarv)$ are mutually independent. When $\Datarv$ is clear from the context, we write $N_i$ instead of $N_i(\Datarv)$. Now we calculate the moments of $\frac{Y}{k}$.  For all $\lambda > 0$, 
\begin{align}
\E_{\substack{N\sim\Po(n) \\ \Datarv \sim  \distr^{\otimes N}}} \Big[\Big(\frac{Y}{k}\Big)^{\lambda} \Big] 
 = \E_{\substack{N\sim\Po(n) \\ \Datarv \sim  \distr^{\otimes N}}} \Big[\Big(\frac{1}{k} \sum_{i \in S^*} e^{N_i(\Datarv) \eps} \Big)^{\lambda} \Big] 
 = \E_{N_1, \cdots, N_k \sim\Po(\frac{\alpha^* n}{k})} \Big[ \Big(\frac{1}{k} \sum_{i \in S^*} e^{N_i \eps} \Big)^{\lambda} \Big]. \label{eq:mompoiss}
\end{align}


Finally, we bound the probability of event $E$. Set $c=e^3$ and $\lambda = \ln \frac{1}{\alpha}$. By definition of $E$, 
\begin{align}
    \Pr(\overline{E}) 
    & = \Pr\Big(\frac{Y}{k} \geq c\Big)\nonumber 
     = \Pr\Big(\Big(\frac{Y}{k}\Big)^\lambda \geq c^{\lambda}\Big) 
    \leq \frac 1{c^{\lambda}}\cdot {\E_{\substack{N\sim\Po(n) \\ \Datarv \sim  \distr^{\otimes N}}}
    \Big[ \Big(\frac{Y}{k}\Big)^\lambda \Big]} \nonumber \\
    & \leq  \frac 1{c^{\lambda}}\cdot {\E_{N_1, \cdots, N_k \sim\Po(\frac{\alpha^* n}{k})} \Big[ \Big(\frac{1}{k} \sum_{i \in S^*} e^{N_i \eps} \Big)^{\lambda} \Big]} 
     \leq  \frac 1{c^{\lambda}}\cdot{\E_{N_1 \sim\Po(\frac{\alpha^* n}{k})} \Big[ \Big( e^{N_1 \eps} \Big) ^{\lambda} \Big]} \label{eq:moments2} \\
    & = c^{-\lambda}\cdot {e^{\frac{\alpha^* n}{k}(e^{\lambda \eps} - 1)}}
    \leq e^{-3\lambda}\cdot {e^{\frac{(e^{\lambda \eps} - 1 )}{\eps}}}
    \leq e^{-3\lambda}\cdot e^{2\lambda}= e^{-\lambda}
    =e^{-\ln (1/\alpha)}=\alpha, \label{eq:moments3}
\end{align}
where we use $\lambda > 0$ in the second equality, then apply Markov's inequality. To get the inequalities in (\ref{eq:moments2}), we apply (\ref{eq:mompoiss}) and then use the fact that the moments of the average of random variables is less than the moment of a single random variable (the proof of this fact is in the supplementary material). To get (\ref{eq:moments3}), we use the moment generating function of a Poisson random variable, and then we substitute $c=e^3$ and use the assumption that $n\leq \frac k{60\alpha \eps} =\frac k{\alpha^*\eps}$. The second inequality in (\ref{eq:moments3}) holds because $\lambda = \ln \frac{1}{\alpha}$ and $\eps \in (0,1/\ln \frac{1}{\alpha}]$, so $\lambda \eps \leq 1$ and hence $e^{\lambda \eps} \leq 1 + 2\lambda \eps$.
The final expression is obtained by substituting the value of $\lambda.$
We get that $\Pr(E) \geq 1-\alpha$, completing the proof of Claim~\ref{claim:eventE}.
\end{proof}
\begin{claim}\label{claim:exp-of-regular-output}
$\displaystyle\E_{\substack{N\sim\Po(n) \\ \Datarv\sim  \distr^{\otimes N}}}\Big[\frac{\Pr[\sampler(\Datarv) \in [2k]]}{1+Y/k} \big| E \Big] \geq 2.3\alpha.$
\end{claim}
\begin{proof}
When event $E$ occurs, $1+\frac{Y}{k} \leq 1+e^3<22$. Then  
\begin{align}\label{eq:conditioning}
    \E_{N, \Datarv}\Big[\frac{\Pr[\sampler(\Datarv) \in [2k]]}{1+Y/k} \big | E \Big] > \E_{N, \Datarv}\Big[\frac{\Pr[\sampler(\Datarv) \in [2k]]}{22} \big | E\Big]  = \frac 1{22}\cdot\E_{N, \Datarv}\Big[\Pr[\sampler(\Datarv) \in [2k]] \mid E \Big].
\end{align}
By the product rule,
$$\Pr[\sampler(\Datarv) \in [2k]] \mid E ]
=\frac{\Pr[\sampler(\Datarv) \in [2k]] \wedge E]}{\Pr[E]}
\geq \Pr[\sampler(\Datarv) \in [2k]] \wedge E]
\geq  {\Pr[\sampler(\Datarv) \in [2k]] - \Pr[\overline{E}]}.$$
Substituting this into (\ref{eq:conditioning}) and recalling that $\alpha^*=60\alpha$, we get
\begin{align*}
    \E_{N, \Datarv}\Big[\frac{\Pr[\sampler(\Datarv) \in [2k]]}{1+Y/k} \big | E \Big]
    &\geq  \frac 1{22}\cdot\E_{N, \Datarv}\Big[\Pr[\sampler(\Datarv) \in [2k]] - \Pr[\overline{E}] \Big]
     \geq \frac 1{22}\cdot \Big( \alpha^* - \alpha - \alpha \Big) 
     \geq 2.3\alpha,
\end{align*}
since sampler $\sampler$ is $\alpha$-accurate on $\distr$, and  $\distr$ has mass $\alpha^*$ on $[2k]$, and by Claim~\ref{claim:eventE}. 
\end{proof}

Combining (\ref{eq:main-dist-lb-delta}), (\ref{eq:mainlbcond}), and (\ref{eq:mainlbcondfirst}), applying Claims~\ref{claim:eventE} and~\ref{claim:exp-of-regular-output}, and recalling that $\delta\leq 0.1\cdot\alpha\eps/k$, we get
\begin{align*}
    d_{TV}&(\distr,\distroutput{\sampler, \distr})
  \geq \Pr_{N, \Datarv}[\sampler(\Datarv) \notin Supp(\distr)]
  \geq \frac{1}{2} \cdot \E_{N, \Datarv}\Big[\frac{\Pr[\sampler(\Datarv) \in [2k]]}{1+Y/k}\Big]  - \frac{k\delta}{\eps} \\
  &\geq \frac 12\cdot   \E_{N, \Datarv}\Big[\frac{\Pr[\sampler(\Datarv) \in [2k]]}{1+Y/k} \big | E\Big] \Pr(E) - 0.1\alpha 
  \geq  \frac{1}{2} \cdot 2.3\alpha\cdot \Big(1 - \alpha\Big) - 0.1\alpha 
  > \alpha,
\end{align*}
where the last inequality holds since $\alpha\leq 0.02$. This contradicts $\alpha$-accuracy of $\sampler$ on datasets of size $\Po(n)$, where $n\leq \frac{k}{\alpha^* \eps}$, and completes the proof of Lemma~\ref{lem:main-k-ary-lb}.
\end{proof}
\color{black}

%% file: main.bbl
\begin{thebibliography}{10}

\bibitem{AcharyaCT20}
Jayadev Acharya, Cl{\'{e}}ment~L. Canonne, and Himanshu Tyagi.
\newblock Inference under information constraints {I:} lower bounds from
  chi-square contraction.
\newblock {\em {IEEE} Trans. Inf. Theory}, 66(12):7835--7855, 2020.

\bibitem{AcharyaCT20a}
Jayadev Acharya, Cl{\'{e}}ment~L. Canonne, and Himanshu Tyagi.
\newblock Inference under information constraints {II:} communication
  constraints and shared randomness.
\newblock {\em {IEEE} Trans. Inf. Theory}, 66(12):7856--7877, 2020.

\bibitem{AcharyaSZ21}
Jayadev Acharya, Ziteng Sun, and Huanyu Zhang.
\newblock Differentially private {A}ssouad, {F}ano, and {L}e {C}am.
\newblock In Vitaly Feldman, Katrina Ligett, and Sivan Sabato, editors, {\em
  Algorithmic Learning Theory, 16-19 March 2021, Virtual Conference,
  Worldwide}, volume 132 of {\em Proceedings of Machine Learning Research},
  pages 48--78. {PMLR}, 2021.

\bibitem{AxelrodGHS22}
Brian Axelrod, Shivam Garg, Yanjun Han, Vatsal Sharan, and Gregory Valiant.
\newblock On the statistical complexity of sample amplification.
\newblock {\em CoRR}, abs/2201.04315, 2022.

\bibitem{Axelrod0SV20}
Brian Axelrod, Shivam Garg, Vatsal Sharan, and Gregory Valiant.
\newblock Sample amplification: Increasing dataset size even when learning is
  impossible.
\newblock In {\em Proceedings of the 37th International Conference on Machine
  Learning, {ICML} 2020, 13-18 July 2020, Virtual Event}, volume 119 of {\em
  Proceedings of Machine Learning Research}, pages 442--451. {PMLR}, 2020.

\bibitem{batu2001testing}
Tugkan Batu, Eldar Fischer, Lance Fortnow, Ravi Kumar, Ronitt Rubinfeld, and
  Patrick White.
\newblock Testing random variables for independence and identity.
\newblock In {\em Proceedings 42nd IEEE Symposium on Foundations of Computer
  Science}, pages 442--451. IEEE, 2001.

\bibitem{batu2000testing}
Tugkan Batu, Lance Fortnow, Ronitt Rubinfeld, Warren~D Smith, and Patrick
  White.
\newblock Testing that distributions are close.
\newblock In {\em Proceedings 41st Annual Symposium on Foundations of Computer
  Science}, pages 259--269. IEEE, 2000.

\bibitem{blum2013learning}
Avrim Blum, Katrina Ligett, and Aaron Roth.
\newblock A learning theory approach to noninteractive database privacy.
\newblock {\em Journal of the ACM (JACM)}, 60(2):1--25, 2013.

\bibitem{BunKSW21J}
Mark Bun, Gautam Kamath, Thomas Steinke, and Zhiwei~Steven Wu.
\newblock Private hypothesis selection.
\newblock {\em {IEEE} Trans. Inf. Theory}, 67(3):1981--2000, 2021.

\bibitem{BunNSV15}
Mark Bun, Kobbi Nissim, Uri Stemmer, and Salil~P. Vadhan.
\newblock Differentially private release and learning of threshold functions.
\newblock In Venkatesan Guruswami, editor, {\em {IEEE} 56th Annual Symposium on
  Foundations of Computer Science, {FOCS} 2015, Berkeley, CA, USA, 17-20
  October, 2015}, pages 634--649. {IEEE} Computer Society, 2015.

\bibitem{bun2016concentrated}
Mark Bun and Thomas Steinke.
\newblock Concentrated differential privacy: Simplifications, extensions, and
  lower bounds.
\newblock In {\em Theory of Cryptography Conference}, pages 635--658. Springer,
  2016.

\bibitem{BunUV14j}
Mark Bun, Jonathan Ullman, and Salil Vadhan.
\newblock Fingerprinting codes and the price of approximate differential
  privacy.
\newblock {\em SIAM Journal on Computing}, 47(5):1888--1938, 2018.

\bibitem{clementpoiss}
Clement Canonne.
\newblock A short note on poisson tail bounds. \\
  http://www.cs.columbia.edu/~ccanonne/files/misc/2017-poissonconcentration.pdf.
\newblock 2019.

\bibitem{diakonikolas2015differentially}
Ilias Diakonikolas, Moritz Hardt, and Ludwig Schmidt.
\newblock Differentially private learning of structured discrete distributions.
\newblock In {\em Advances in Neural Information Processing Systems}, pages
  2566--2574, 2015.

\bibitem{DimitrakakisNMR14}
Christos Dimitrakakis, Blaine Nelson, Aikaterini Mitrokotsa, and Benjamin I.~P.
  Rubinstein.
\newblock Robust and private {Bayesian} inference.
\newblock In Peter Auer, Alexander Clark, Thomas Zeugmann, and Sandra Zilles,
  editors, {\em Algorithmic Learning Theory - 25th International Conference,
  {ALT} 2014, Bled, Slovenia, October 8-10, 2014. Proceedings}, volume 8776 of
  {\em Lecture Notes in Computer Science}, pages 291--305. Springer, 2014.

\bibitem{DimitrakakisNZM17}
Christos Dimitrakakis, Blaine Nelson, Zuhe Zhang, Aikaterini Mitrokotsa, and
  Benjamin I.~P. Rubinstein.
\newblock Differential privacy for {Bayesian} inference through posterior
  sampling.
\newblock {\em J. Mach. Learn. Res.}, 18:11:1--11:39, 2017.

\bibitem{DworkKMMN06}
Cynthia Dwork, Krishnaram Kenthapadi, Frank McSherry, Ilya Mironov, and Moni
  Naor.
\newblock Our data, ourselves: Privacy via distributed noise generation.
\newblock In {\em International Conference on the Theory and Applications of
  Cryptographic Techniques}, EUROCRYPT '06, pages 486--503, St.~Petersburg,
  Russia, 2006.

\bibitem{DworkMNS06j}
Cynthia Dwork, Frank McSherry, Kobbi Nissim, and Adam Smith.
\newblock Calibrating noise to sensitivity in private data analysis.
\newblock {\em Journal of Privacy and Confidentiality}, 7(3):17–51, 2017.

\bibitem{DworkS09}
Cynthia Dwork and Adam Smith.
\newblock Differential privacy for statistics: What we know and what we want to
  learn.
\newblock {\em Journal of Privacy and Confidentiality}, 1(2):135--154, 2009.

\bibitem{KLSU19}
Gautam Kamath, Jerry Li, Vikrant Singhal, and Jonathan~R. Ullman.
\newblock Privately learning high-dimensional distributions.
\newblock In Alina Beygelzimer and Daniel Hsu, editors, {\em Conference on
  Learning Theory, {COLT} 2019, 25-28 June 2019, Phoenix, AZ, {USA}}, volume~99
  of {\em Proceedings of Machine Learning Research}, pages 1853--1902. {PMLR},
  2019.

\bibitem{KamathUprimerpaper20}
Gautam Kamath and Jonathan Ullman.
\newblock A primer on private statistics, 2020.
\newblock ArXiv preprint 2005.00010 [stat.ML].

\bibitem{KasiviswanathanLNRS11}
Shiva~Prasad Kasiviswanathan, Homin~K. Lee, Kobbi Nissim, Sofya Raskhodnikova,
  and Adam~D. Smith.
\newblock What can we learn privately?
\newblock {\em {SIAM} J. Comput.}, 40(3):793--826, 2011.

\bibitem{KasiviswanathanS14}
Shiva~Prasad Kasiviswanathan and Adam~D. Smith.
\newblock On the `semantics' of differential privacy: A {Bayesian} formulation.
\newblock {\em Journal of Privacy and Confidentiality}, 2014.

\bibitem{klenke2010stochastic}
Achim Klenke and Lutz Mattner.
\newblock Stochastic ordering of classical discrete distributions.
\newblock {\em Advances in Applied probability}, 42(2):392--410, 2010.

\bibitem{li2012sampling}
Ninghui Li, Wahbeh Qardaji, and Dong Su.
\newblock On sampling, anonymization, and differential privacy or,
  $k$-anonymization meets differential privacy.
\newblock In {\em Proceedings of the 7th ACM Symposium on Information, Computer
  and Communications Security}, pages 32--33, 2012.

\bibitem{McTalwar}
Frank McSherry and Kunal Talwar.
\newblock Mechanism design via differential privacy.
\newblock In {\em Proceedings of the 48th Annual IEEE Symposium on Foundations
  of Computer Science}, FOCS '07, page 94–103, USA, 2007. IEEE Computer
  Society.

\bibitem{NissimRS07}
Kobbi Nissim, Sofya Raskhodnikova, and Adam~D. Smith.
\newblock Smooth sensitivity and sampling in private data analysis.
\newblock In David~S. Johnson and Uriel Feige, editors, {\em Proceedings of the
  39th Annual {ACM} Symposium on Theory of Computing, San Diego, California,
  USA, June 11-13, 2007}, pages 75--84. {ACM}, 2007.

\bibitem{raskhodnikova2009strong}
Sofya Raskhodnikova, Dana Ron, Amir Shpilka, and Adam Smith.
\newblock Strong lower bounds for approximating distribution support size and
  the distinct elements problem.
\newblock {\em SIAM Journal on Computing}, 39(3):813--842, 2009.

\bibitem{RaskhodnikovaSSS21}
Sofya Raskhodnikova, Satchit Sivakumar, Adam~D. Smith, and Marika Swanberg.
\newblock Differentially private sampling from distributions.
\newblock In Marc'Aurelio Ranzato, Alina Beygelzimer, Yann~N. Dauphin, Percy
  Liang, and Jennifer~Wortman Vaughan, editors, {\em Advances in Neural
  Information Processing Systems 34: Annual Conference on Neural Information
  Processing Systems 2021, NeurIPS 2021, December 6-14, 2021, virtual}, pages
  28983--28994, 2021.

\bibitem{RaskhodnikovaS06}
Sofya Raskhodnikova and Adam~D. Smith.
\newblock A note on adaptivity in testing properties of bounded degree graphs.
\newblock {\em Electron. Colloquium Comput. Complex.}, (089), 2006.

\bibitem{Vadhan17}
Salil Vadhan.
\newblock {\em The Complexity of Differential Privacy}, pages 347--450.
\newblock Springer, 2017.

\bibitem{WangFS15}
Yu{-}Xiang Wang, Stephen~E. Fienberg, and Alexander~J. Smola.
\newblock Privacy for free: Posterior sampling and stochastic gradient monte
  carlo.
\newblock In Francis~R. Bach and David~M. Blei, editors, {\em Proceedings of
  the 32nd International Conference on Machine Learning, {ICML} 2015, Lille,
  France, 6-11 July 2015}, volume~37 of {\em {JMLR} Workshop and Conference
  Proceedings}, pages 2493--2502. JMLR.org, 2015.

\bibitem{Warner65}
Stanley~L. Warner.
\newblock Randomized response: A survey technique for eliminating evasive
  answer bias.
\newblock {\em Journal of the American Statistical Association},
  60(309):63--69, 1965.

\bibitem{ZhangRD16}
Zuhe Zhang, Benjamin I.~P. Rubinstein, and Christos Dimitrakakis.
\newblock On the differential privacy of {Bayesian} inference.
\newblock In Dale Schuurmans and Michael~P. Wellman, editors, {\em Proceedings
  of the Thirtieth {AAAI} Conference on Artificial Intelligence, February
  12-17, 2016, Phoenix, Arizona, {USA}}, pages 2365--2371. {AAAI} Press, 2016.

\end{thebibliography}
